\documentclass[11pt]{article}
\usepackage{style}

\usepackage{enumerate}
\usepackage[T1]{fontenc}
\usepackage[utf8]{inputenc}
\usepackage{microtype}
\usepackage[margin=1in]{geometry}
\usepackage{graphicx}
\usepackage{subcaption}
\usepackage{booktabs} %
\usepackage{mathpazo}
\usepackage{pifont}
\usepackage{makecell}
\usepackage{tabularx}
\usepackage{pbox}
\usepackage{wrapfig}

\usepackage{xcolor}
\usepackage{pgfplots}
\usepackage{pgfplotstable}
\pgfplotsset{compat=1.3}
\usepackage{tikz}
\usetikzlibrary{arrows.meta}
\usetikzlibrary{pgfplots.groupplots}
\usepackage{ragged2e}

\usepackage[toc,page,header]{appendix}
\usepackage{minitoc}

\usepackage[colorlinks=true,linkcolor=blue,urlcolor=blue,citecolor=blue]{hyperref}
\usepackage{xspace}
\usepackage{bm}
\usepackage{comment}
\usepackage{caption}
\usepackage{makecell}

\newlength\myindent
\usepackage{amsmath}
\usepackage{amssymb}
\usepackage{mathtools}
\usepackage{amsthm}
\usepackage{thmtools}
\usepackage{algorithm}
\usepackage{algpseudocode}
\usepackage{color}

\usepackage[capitalize,noabbrev]{cleveref}

\definecolor{mydarkblue}{rgb}{0,0.08,0.85}
\definecolor{mylightblue}{rgb}{0.06,0.56,1.0}
\definecolor{mylightorange}{rgb}{1.0,0.62,0.12}
\definecolor{mylightred}{rgb}{0.99,0.00,0.04}
\hypersetup{colorlinks=true,
linkcolor=mydarkblue,
citecolor=mydarkblue,
filecolor=mydarkblue,
urlcolor=mydarkblue,
bookmarksopen=true,
linktoc=page}

\theoremstyle{plain}
\newtheorem{theorem}{Theorem}[section]

\newtheorem{lemma}[theorem]{Lemma}

\theoremstyle{definition}

\theoremstyle{remark}

\title{
Is This the Subspace You Are Looking for? An Interpretability Illusion for Subspace Activation Patching
}

\author{
\begin{tabular}[t]{c c}
Aleksandar Makelov\thanks{Equal Contribution.} &
Georg Lange\footnotemark[1] \\
\texttt{aleksandar.makelov@gmail.com}  &
\texttt{mail@georglange.com} \\
SERI MATS & SERI MATS \\
\end{tabular}
\\[0.6in] %
\begin{tabular}[t]{c}
Neel Nanda \\
\texttt{neelnanda27@gmail.com} \\
\end{tabular}
}
\date{}

\definecolor{mygreen}{HTML}{2F9E44}
\definecolor{myred}{HTML}{E03131}
\definecolor{myblue}{HTML}{1971C2}
\def\l{\left}
\def\r{\right}

\begin{document}
\maketitle

\begin{abstract}
Mechanistic interpretability aims to understand model behaviors in terms of
specific, interpretable features, often hypothesized to manifest as
low-dimensional subspaces of activations.  Specifically, recent studies have
explored subspace interventions (such as activation patching) as a way to
simultaneously manipulate model behavior and attribute the features behind it to
given subspaces.

In this work, we demonstrate that these two aims diverge, potentially leading to
an illusory sense of interpretability.  Counterintuitively, even if a subspace
intervention makes the model's output behave \emph{as if} the value of a feature
was changed, this effect may be achieved by activating a \emph{dormant parallel
pathway} leveraging another subspace that is \emph{causally disconnected} from model
outputs.  We demonstrate this phenomenon in a distilled mathematical example, in
two real-world domains (the indirect object identification task and factual
recall), and present evidence for its prevalence in practice.  In the context of
factual recall, we further show a link to rank-1 fact editing, providing a
mechanistic explanation for previous work observing an inconsistency between
fact editing performance and fact localization.

However, this does not imply that activation patching of subspaces is
intrinsically unfit for interpretability.  To contextualize our findings, we
also show what a success case looks like in a task (indirect object identification) where prior manual
circuit analysis informs an understanding of the location of a feature. We
explore the additional evidence needed to argue that a patched subspace is
faithful.
\end{abstract}

\section{Introduction}
Recently, large language models (LLMs) have demonstrated impressive
\citep{vaswani2017attention, bert, gpt4, radford2019language,
brown2020language}, and often surprising \citep{wei2022emergent}, capability
gains. However, they are still widely considered `black boxes': their successes
-- and failures -- remain largely a mystery. It is thus an increasingly
pressing scientific and practical question to understand \emph{what} LLMs learn
and \emph{how} they make predictions. 

This is the goal of machine learning interpretability, a broad field that
presents us with both technical and conceptual challenges \citep{lipton2016}.
Within it, mechanistic interpretability (MI) is a subfield that seeks to develop
a rigorous low-level understanding of the mechanisms and learned algorithms
behind a model's computations. MI frames these computations as collections of
narrow, task-specific algorithms
-- \emph{circuits} \citep{olah2020zoom, geiger2021causal,
wang2022interpretability}
-- whose operations are grounded in concrete, atomic building blocks akin to
variables in a computer program \citep{olah2022mechanistic} or causal model
\citep{vig2020causal, Geiger-etal:2023:CA}. 
MI has found applications in several downstream tasks: removing toxic behaviors
from a model while otherwise preserving performance by minimally editing model
weights \citep{li2023circuit}, changing factual knowledge encoded by models in
specific components to e.g. enable more efficient fine-tuning in a changing
world \citep{meng2022locating}, improving the truthfulness of LLMs at inference
time via efficient, localized inference-time interventions in specific subspaces
\citep{li2023inference} and studying the mechanics of gender bias in language
models \citep{vig2020causal}.

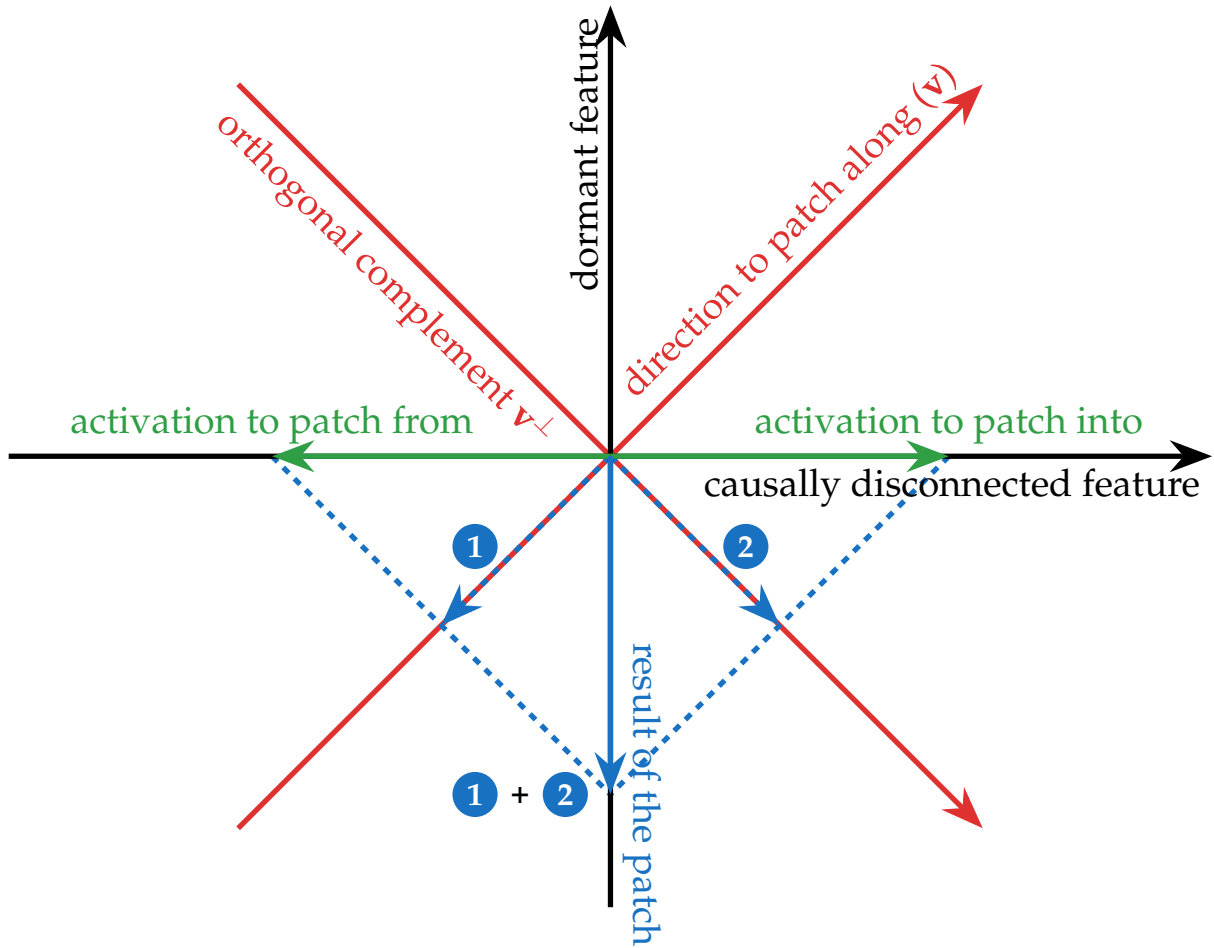
\begin{figure}
\begin{tikzpicture}
    \draw[line width=2pt, myred, -{Stealth[length=8pt, scale=2]}] (-135:7)--(45:7)
    node[above left, font=\Large, rotate=45]{direction to patch along ($\mathbf{v}$)};
    \draw[line width=2pt, myred, -{Stealth[length=8pt, scale=2]}] (135:7)--(-45:7)
    node[pos=0, below right,font=\Large, rotate=-45]{orthogonal complement
    $\mathbf{v}^\perp$};
    
    \draw[dashed, line width=2pt, myblue] (-4.5,0)--(0,-4.5);
    \fill[myblue] (-1.8,-1.2) circle (0.3);
    \node at (-1.8,-1.2) [font={\large\bfseries}, text=white] {1};

    \fill[myblue] (1.8,-1.2) circle (0.3);
    \node at (1.8,-1.2) [font={\large\bfseries}, text=white] {2};

    \draw[dashed, -{Stealth[length=8pt, scale=2]}, line width=2pt, myblue]
    (0,0)--(-2.25,-2.25);
    \draw[dashed, line width=2pt, myblue] (4.5,0)--(0,-4.5);
    \draw[dashed, -{Stealth[length=8pt, scale=2]}, line width=2pt, myblue]
    (0,0)--(2.25,-2.25);

    \draw[line width=2pt,-{Stealth[length=8pt, scale=2]}] (-8,0)--(8,0)
    node[below left,font=\Large]{causally disconnected feature};
    \draw[line width=2pt,-{Stealth[length=8pt, scale=2]}] (0,-6)--(0,6) node[above left,rotate=90,font=\Large]{dormant feature};
    \draw[-{Stealth[length=8pt, scale=2]}, mygreen, line width=2pt] (0,0)--(4.5,0)
    node[above,font=\Large]{activation to patch into};
    \draw[-{Stealth[length=8pt, scale=2]}, mygreen, line width=2pt] (0,0)--(-4.5,0)
    node[above,font=\Large]{activation to patch from};
    \draw[-{Stealth[length=8pt, scale=2]}, myblue, line width=2pt] (0,0)--(0,-4.5)
    node[above,font=\Large, rotate=-90]{result of the patch};

    \fill[myblue] (-1.8,-4.5) circle (0.3);
    \node at (-1.8,-4.5) [font={\large\bfseries}, text=white] {1};
    \node at (-1.2,-4.5) [font={\large\bfseries}, text=black] {+};
    \fill[myblue] (-0.6,-4.5) circle (0.3);
    \node at (-0.6,-4.5) [font={\large\bfseries}, text=white] {2};

\end{tikzpicture}
  \caption{The key mathematical phenomenon behind the activation patching
  illusion illustrated for a 2-dimensional activation space. We intervene on an
  example's \textcolor{mygreen}{activation} (green, right) by setting its
  orthogonal projection on a \textcolor{myred}{1-dimensional subspace
  $\mathbf{v}$ of activation space} (red, top-right) to equal the orthogonal
  projection of another example's \textcolor{mygreen}{activation} (green, left)
  on $\mathbf{v}$. The result is a \textcolor{myblue}{patched activation vector}
  orthogonal to both activations.
  Specifically, to form the patched activation we take the $\mathbf{v}$ component
  (\raisebox{.5pt}{\textcircled{\raisebox{-.9pt} {1}}}) of the activation we are
  patching from, and combine it with the $\mathbf{v}^\perp$ component
  (\raisebox{.5pt}{\textcircled{\raisebox{-.9pt} {2}}}) of the original
  activation. This results in the patched activation 
  \raisebox{.5pt}{\textcircled{\raisebox{-.9pt} {1}}}
  +\raisebox{.5pt}{\textcircled{\raisebox{-.9pt} {2}}}.
  This can lead to counterintuitive results when the original and new directions
  have fundamentally different roles in a model's computation; see Section
  \ref{sec:methods} for details, and Figure \ref{fig:illusion-step-by-step} for
  a step-by-step guide through this figure.
  }
  \label{fig:illusion}
\end{figure}

A central question in MI is: what \emph{is} the proper definition of these
building blocks?  Any satisfying mechanistic analysis of high-level LLM
capabilities must rest on a rigorous and comprehensive answer to this question
\citep{olah2022mechanistic}.  Many initial mechanistic analyses have focused on
mapping circuits to collections of \emph{model components}
\citep{wang2022interpretability, docstring}, such as attention heads and MLP
layers.  A workhorse of these analyses is \emph{activation
patching}\footnote{also known as `interchange intervention' \citep{geigerPatch}
and sometimes referred to as `resample ablation' \citep{causal_scrubbing} or
`causal tracing' \citep{meng2022locating}.} \citep{vig2020causal,geigerPatch,
meng2022locating,wang2022interpretability}, which intervenes on model
computation on an input by replacing the activation of a given component
with its value when the model is run on another input. By seeing which
components lead to a significant task-relevant change in outputs compared to
running the model normally, activation patching aims to pinpoint tasks to
specific components.

However, localizing features to entire components is not sufficient for a
detailed understanding. A plethora of empirical evidence suggests that the
features LLMs represent and use are more accurately captured by \emph{linear
subspaces} of component activations \citep{nanda2023actually, li2021implicit,
abdou2021can, grand2018semantic}. Complicating matters, phenomena like
superposition and polysemanticity \citep{elhage2022superposition} suggest that
these subspaces are not easily enumerable, like individual neurons -- so
searching for them can be non-trivial.  This raises the question:
\begin{center}
    \emph{
Does the success of activation patching carry over from component-level
analysis to finding the precise subspaces corresponding to features?   
    }
\end{center}

In this paper, we demonstrate that naive generalizations of subspace activation
patching can lead to misleading interpretability results. Specifically, we argue
empirically and theoretically that a subspace seemingly encoding some feature
may be found in the MLP layers on the path between two model components in a
transformer model that communicate this feature as part of some circuit. 

As a concrete example of how this illusion can happen in the practice of
interpretability, recent works such as \citet{geiger2023finding,
wu2023interpretability} have sought to identify interpretable subspaces using
gradient descent, with training objectives that optimize for a subspace patch
with a causal effect on model predictions. While this kind of end-to-end
optimization has promise, we show that, instead of localizing a variable used by
the model, subspace interventions such as subspace activation patching can
create such a variable by \emph{activating a dormant pathway}.

Counterintuitively, the mathematics of subspace interventions makes it possible to activate another, `dormant', direction, which is
ordinarily inactive, but can change model outputs when activated (see Figure
\ref{fig:illusion}), by exploiting the variation of model activations in a direction correlated with a feature even if this second direction does not causally affect the output. 
An equivalent view of this phenomenon that we explore in
Appendix \ref{app:rotated-basis} is that the component contains two subspaces
that mediate the variable, but whose effects normally cancel each other out
(hence, there's no total effect, making the component as a whole appear `dormant'). The
activation patching intervention decouples these two subspaces by altering an activation only
along one of them. Under this perspective, our contribution is to show that model components are likely to contain such pairs of subspaces that perfectly cancel out.
When this phenomenon is realized in the hidden activations of an MLP
layer, it leads to causally meaningful subspaces which have a substantial and crucial
component that is causally disconnected from model outputs, owing to the
high-dimensional kernel of an MLP layer's down-projection in a transformer (see
Figure \ref{fig:ioi-diagram}).

While it is, in principle, possible that subspaces that represent some variable
but cancel each other out exist in many model components, we find this unlikely.
Specifically, our results suggest that every MLP layer between two components
communicating some feature through residual connections is likely to contain a
subspace which appears to mediate the feature when activation patched. We find
this implausible on various grounds that we revisit in Section
\ref{sec:discussion}. Thus, we consider at least some of these subspaces to
exhibit a kind of \emph{interpretability illusion}
\citep{bolukbasi2021interpretability,sanity_checks_saliency}.

Our contributions can be summarized as follows:
\begin{itemize}
    \item In Section \ref{sec:methods}, we provide the key intuition for the
    illusion, and construct a distilled mathematical example.
    \item In Section \ref{sec:ioi}, we find a realization of this phenomenon `in
    the wild', in the context of the indirect object identification task
    \citep{wang2022interpretability}, where a 1-dimensional subspace of MLP
    activations found using DAS \citep{geiger2023finding} can seem to encode
    position information about names in the sentence, despite this MLP layer
    having negligible contribution to the circuit as argued by \citet{wang2022interpretability}.
    \item To contextualize our results, in Section \ref{sec:ground-truth} we
    also show how DAS can be used to find subspaces that faithfully represent a
    feature in a model's computation. Specifically, we find a 1-dimensional
    subspace encoding the same position information in the IOI task, and
    validate its role in model computations via mechanistic experiments beyond
    end-to-end causal effect. We argue that activation patching on subspaces of
    the residual stream is less prone to illusions.
    \item Going beyond the IOI task, in Section \ref{sec:facts} we also exhibit
    this phenomenon in the setting of \emph{fact editing}
    \citep{meng2022locating}.  We show that 1-dimensional activation patches
    imply approximately equivalent rank-1 model edits \citep{meng2022locating}.
    In particular, this shows that rank-1 model edits can also be achieved by
    activating a dormant pathway in the model, without necessarily relying on
    the presence of a fact in the weight being edited. This suggests a
    mechanistic explanation for the observation of \citep{hase2023does} that
    rank-1 model editing `works' regardless of whether the fact is present in
    the weights being edited.
    \item In Section \ref{sec:prevalent}, we collect arguments and evidence for
    why this interpretability illusion ought to be prevalent in real-world
    language models.
    \item Finally, in Section \ref{sec:discussion}, we provide conceptual
    discussion of these findings.
\end{itemize}
We have also released code to reproduce our findings\footnote{\url{https://github.com/amakelov/activation-patching-illusion}}.

\section{Related Work}
\label{sec:related-work}

\subsection{Discovering and Causally Intervening on Representations with
Activation Patching}
\label{subsection:}
Researchers have been exploring increasingly fine-grained ways of
reverse-engineering and steering model behavior. In this context,
\emph{activation patching} \citep{vig2020causal, geigerPatch} is a widely used
causal intervention, whereby the model is run on an input A, but chosen
activations are `patched in' from input B.  Motivated by causal mediation
analysis \citep{pearl2001direct} and causal abstraction
\cite{Geiger-etal:2023:CA}, activation patching has been used to localize model
components causally involved in various behaviors, such as gender bias
\citep{causal_meditation_analysis}, factual recall \citep{meng2022locating},
multiple choice questions \citep{lieberum2023does}, arithmetic
\citep{stolfo2023understanding} and natural language reasoning
\citep{geiger2021causal,wang2022interpretability,geiger2023finding,
wu2023interpretability}, code \citep{docstring}, and (in certain regimes)
topic/sentiment/style of free-form natural language
\citep{turner2023activation}.

Activation patching is an area of active research, and many recent works have
extended the method, with patching paths between components
\citep{goldowsky2023localizing}, automating the finding of sparse subgraphs
\citep{conmy2023towards}, fast approximations \citep{neelattribution}, and
automating the verification of hypotheses \citep{causal_scrubbing}.  

In particular, \emph{full-component activation patching} -- where the entire
activation of a model component such as attention head or MLP layer is replaced
-- is not the end of the story. A wide range of interpretability work
\citep{DBLP:conf/nips/MikolovSCCD13, conneau-etal-2018-cram, hewitt2019structural,
tenney-etal-2019-bert, burns2022discovering, nanda2023emergent} suggests the
\emph{linear representation hypothesis}: models encode features as linear
subspaces of component activations that can be arbitrarily rotated with respect
to the standard basis (due to phenomena like superposition, polysemanticity
\citep{arora2018linear, elhage2022superposition} and lack of privileged bases
\citep{Smolensky1986,elhage2021mathematical}).   

Motivated by this, recent works such as \citet{geiger2023finding,
wu2023interpretability, lieberum2023does} have explored \emph{subspace
activation patching}: a generalization of activation patching that operates only
on linear subspaces of features (as low as 1-dimensional) rather than patching entire components. 

Our work contributes to this research direction by demonstrating both (i) a
common illusion to avoid when looking for such subspaces and (ii) a detailed
case study of successfully localizing a binary feature to a 1-dimensional
subspace.

\subsection{Interpretability Illusions}
\label{subsection:}
Despite the promise of interpretability, it is difficult to be rigorous and easy
to mislead yourself. A common theme in the field is identifying ways that
techniques may lead to misleading conclusions about model behavior
\citep{lipton2016}. 
In computer vision, \citet{sanity_checks_saliency} show that a popular at the
time class of pixel attribution methods is not sensitive to whether or not the
model used to produce is has actually been trained or not. In
\citet{geirhos2023don}, the authors show how a circuit can be hardcoded into a
learned model so that it fools interpretability methods; this bears some
similarity to our illusion, especially its fact editing counterpart.  

In natural language processing, \citet{bolukbasi2021interpretability} show that
interpreting single neurons with maximum activating dataset examples may lead to
conflicting results across datasets due to subtle polysemanticity
\citep{elhage2022superposition}. Recently, \citet{mcgrath2023hydra} demonstrated
that full-component activation patching in large language models is vulnerable
to false negatives due to (ordinariliy dormant) backup behavior of downstream
components that activates when a component is ablated.

We contribute to the study of interpretability illusions by demonstrating a new
kind of illusion which can arise when intervening on model activations along
arbitrary subspaces, by demonstrating it in two real-world scenarios, and
providing recommendations on how to avoid it.

\subsection{Factual Recall}
\label{subsection:}

A well-studied domain for discovering and intervening on learned representations
is the localization and editing of factual knowledge in language models
\citep{singh-etal-2020-bertnesia, meng2022memit, 
dai-etal-2022-knowledge, geva2023dissecting, hernandez2023inspecting}.  A work of particular note is
\citet{meng2022locating}, which localizes factual information using a variation
of full-component activation patching, and then edits factual information with a
rank-1 intervention on model weights. However, recent work has shown that rank-1
editing can work even on weights where the fact supposedly is not encoded
\citep{hase2023does}, and that editing a single fact often fails to have its
expected common-sense effect on logically related downstream facts
\citep{cohen2023evaluating, zhong2023mquake}.

We contribute to this line of work by showing a formal and empirical connection
between activation patching along 1-dimensional subspaces and rank-1 model
editing. In particular, rank-1 model edits can work by creating a dormant
pathway of an MLP layer, regardless of whether the fact is stored there. This
provides a mechanistic explanation for the discrepancy observed in
\citet{hase2023does}.

\section{A Conceptual View of the Illusion}
\label{sec:methods}

\subsection{Preliminaries: (Subspace) Activation Patching}
\label{subsection:}

\emph{Activation patching} \citep{vig2020causal, geigerPatch,
wang2022interpretability, causal_scrubbing} is an interpretability technique
that intervenes upon model components, forcing them to take on values they would
have taken if a different input were provided.  For instance, consider a model
that has knowledge of the locations of famous landmarks, and completes e.g. the
sentence $A=\text{`The Eiffel Tower is in'}$ with `Paris'.  

How can we find which components of the model are responsible for knowing that
`Paris' is the right completion? Activation patching approaches this question by
\begin{enumerate}[(i)]
\item Running the model on $A$;
\item Storing the activation of a chosen model component $\mathcal{C}$, such as
the output of an attention head, the hidden activations of an MLP layer, or an
entire residual stream (a.k.a. bottleneck) layer;
\item Running the model on e.g. $B=\text{`The Colosseum is in'}$, \emph{but}
with the activation of $\mathcal{C}$ taken from $A$.
\end{enumerate}

If we find that the model outputs
`Paris' instead of `Rome' in step (iii), this suggests that  component
$\mathcal{C}$ is important for the task of recalling the location of a landmark.

The linear representation hypothesis proposes that \emph{linear subspaces} of
vectors will be the most interpretable model components.  To search for such
subspaces, we can adopt a natural generalization of full component activation
patching, which only replaces the values of a subspace $U$ (while leaving the
projection on its orthogonal complement $U^\perp$ unchanged). This was proposed
in \citet{geiger2023finding}, and closely related variants appear in
\citet{turner2023activation, nanda2023emergent, lieberum2023does}.

For the purposes of exposition, we now restrict our discussion to activation
patching of a 1-dimensional subspace (i.e., a \emph{direction}) spanned by a
unit vector $\mathbf{v}$ (i.e., $\l\|\mathbf{v}\r\|_2=1$). We remark that the
illusion also applies to higher-dimensional subspaces (see Appendix
\ref{app:higher-dim} for theoretical details; later on, in Appendix
\ref{app:generalization-high-dim}, we also show this empirically for the IOI
task).  If $\mathbf{act}_A, \mathbf{act}_B\in\mathbb{R}^d$ are the
activations of a model component $\mathcal{C}$ on examples $A, B$ and
$p_A=\mathbf{v}^\top \mathbf{act}_A, p_B=\mathbf{v}^\top \mathbf{act}_B$ are
their projections on $\mathbf{v}$, patching from $A$ into $B$ along
$\mathbf{v}$ results in the patched activation
\begin{align}
\label{eq:onedim-patch}
    \mathbf{act}_B^{\text{patched}} = \mathbf{act}_B + (p_A-p_B)\mathbf{v}.
\end{align}
For a concrete scenario motivating such a patch, consider a discrete binary
feature used by the model to perform a task, and prompts $A, B$ which only
differ in the value of this feature. A 1-dimensional subspace can easily encode
such a feature (and indeed we explore an example of this in great detail in
Sections \ref{sec:ioi} and \ref{sec:ground-truth}).

\subsection{Intuition for the Illusion}
\label{subsection:}
What would make activation patching a good attribution method? Intuitively, an
equivalence is needed: an activation patch should work \emph{if and only if} the
component/subspace being patched is indeed a \emph{faithful to the model's computation} representation of the concept we seek to localize.
Revisiting Equation \ref{eq:onedim-patch} with this in mind, it is quite
plausible that, if $\mathbf{v}$ indeed encodes a binary feature
relevant to the task, the patch will essentially overwrite the feature with its
value on $A$, and this would lead to the expected downstream effect on model
predictions\footnote{It is in principle possible that, even if the value of the
feature is overwritten, this has no effect on model predictions. For example, it
is possible that $\mathbf{v}$ is not the only location in the model's
computation where this feature is represented; or, it may be that there are
backup components that are normally inactive on the task, but activate when the
value of the subspace $\mathbf{v}$ is changed, as explored in
\citet{mcgrath2023hydra}. Such scenarios are beyond the scope of this simplified
discussion.}.

Going in the other direction of the equivalence, when will the update in
Equation \ref{eq:onedim-patch} change the model's output in the intended way?
Intuitively, two properties are necessary:
\begin{itemize}
    \item \textbf{correlation with the concept}: $\mathbf{v}$ must be activated
    differently by the two prompts.  Otherwise, $p_A\approx p_B$, and the patch
    has no effect; \item \textbf{potential for changing model outputs}:
    $\mathbf{v}$ must be `causally connected' to the model's outputs; in other
    words, it should be the case that changing the activation along $\mathbf{v}$
    can at least in some cases lead to a change in the next-token probabilities
    output by the model.
    Otherwise, if, for instance, $\mathbf{v}$ is in the nullspace of all
    downstream model components, changing the activation's projection along
    $\mathbf{v}$ alone won't have any effect on the model's predictions. 
    
    For
    example, if the component $\mathcal{C}$ we are patching is the
    post-nonlinearity activation of an MLP layer, the only way this activation
    affects the model's output is through matrix multiplication with a
    down-projection $W_{out}$. So, if $\mathbf{v}\in \ker W_{out}$, we will have
\begin{align*}
    W_{out}\mathbf{act}_B^{\text{patched}} = W_{out}\mathbf{act}_B +
    (p_A-p_B)W_{out}\mathbf{v} = W_{out}\mathbf{act}_B.
\end{align*}
In other words, the activation patch leads to the exact same output of the MLP
layer as when running the model on $B$ without an intervention. So, the patch
will leave model predictions unchanged.
\end{itemize}
The crux of the illusion is that $\mathbf{v}$ may obtain each of the two
properties from two `unrelated' directions in activation space (as shown in Figure
\ref{fig:illusion}) which `happen to be there' as a side effect of linear algebra. Specifically, we can form
\begin{align}
\label{eq:illusion-decomposition}
  \mathbf{v}_{\text{illusory}}=\frac{1}{\sqrt{2}}\l(\mathbf{v}_{\text{disconnected}}
  + \mathbf{v}_{\text{dormant}}\r),
\end{align}
for orthogonal unit vectors $\mathbf{v}_{\text{disconnected}}^\top
\mathbf{v}_{\text{dormant}}=0$ such that
\begin{itemize}
\item $\mathbf{v}_{\text{disconnected}}$ is a \textbf{causally disconnected
direction} in activation space: it distinguishes between the two prompts, but is in the nullspace of
all downstream model components (e.g., a vector in $\ker W_{out}$ for an MLP
layer with down-projection $W_{out}$);
\item $\mathbf{v}_{\text{dormant}}$ is a \textbf{dormant direction} in activation
space: it can
\emph{in principle} steer the model's output in the intended way, but is not
activated differently by the two prompts (in other words,
$\mathbf{v}_{\text{dormant}}^\top \mathbf{act}_A \approx
\mathbf{v}_{\text{dormant}}^\top\mathbf{act}_B$).
\end{itemize}
 
To illustrate this algebraically, consider what happens when we patch
along $\mathbf{v}_{\text{illusory}}$. We have 
\begin{align*}
    p_A = \mathbf{v}_{\text{illusory}}^\top \mathbf{act}_A = 
    \frac{1}{\sqrt{2}} \l(\mathbf{v}_{\text{disconnected}}^\top \mathbf{act}_A + \mathbf{v}_{\text{dormant}}^\top \mathbf{act}_A\r)
    \\
    p_B = \mathbf{v}_{\text{illusory}}^\top \mathbf{act}_B = 
    \frac{1}{\sqrt{2}}\l(\mathbf{v}_{\text{disconnected}}^\top \mathbf{act}_B + \mathbf{v}_{\text{dormant}}^\top \mathbf{act}_B\r)
\end{align*}
By assumption, $\mathbf{v}_{\text{dormant}}^\top \mathbf{act}_B =
\mathbf{v}_{\text{dormant}}^\top \mathbf{act}_A$. Thus,
\begin{align*}
    p_A - p_B = \frac{1}{\sqrt{2}}\l(\mathbf{v}_{\text{disconnected}}^\top \mathbf{act}_A - \mathbf{v}_{\text{disconnected}}^\top \mathbf{act}_B\r)
\end{align*}
so the patched activation is 
\begin{align*}
    \mathbf{act}_B^{\text{patched}} = \mathbf{act}_B +
    \frac{1}{\sqrt{2}}\l(\mathbf{v}_{\text{disconnected}}^\top \mathbf{act}_A -
    \mathbf{v}_{\text{disconnected}}^\top \mathbf{act}_B\r)\mathbf{v}_{\text{illusory}}.
\end{align*}
If for example $\mathbf{v}_{\text{illusory}}$ is in the space of
post-nonlinearity activations of an MLP layer with down-projection matrix $W_{out}$, and
$\mathbf{v}_{\text{disconnected}}\in\ker W_{out}$, the new output of the MLP
layer after the patch will be 
\begin{align*}
    W_{out} \mathbf{act}_B^{\text{patched}} &= W_{out}\mathbf{act}_B +
    \frac{1}{\sqrt{2}}\l(\mathbf{v}_{\text{disconnected}}^\top \mathbf{act}_A -
    \mathbf{v}_{\text{disconnected}}^\top
    \mathbf{act}_B\r)W_{out}\mathbf{v}_{\text{illusory}}
\end{align*}
\begin{align}
    \label{eq:patched-mlp-output-formula}
    &= W_{out}\mathbf{act}_B + \frac{1}{2}\l(\mathbf{v}_{\text{disconnected}}^\top \mathbf{act}_A -
    \mathbf{v}_{\text{disconnected}}^\top
    \mathbf{act}_B\r)W_{out}\mathbf{v}_{\text{dormant}}
\end{align}
where we used that $W_{out}\mathbf{v}_{\text{disconnected}} = 0$.
From this equation, we see that, by patching along the sum of a disconnected and
dormant direction, the variation in activation projections on the disconnected
part (which we assume is significant) `activates' the dormant
part: we get a new contribution to the MLP's output (along
$W_{out}\mathbf{v}_{\text{dormant}}$) which can
then possibly influence model outputs. This contribution would not exist if we patched only along $\mathbf{v}_{\text{disconnected}}$
(because it would be nullified by $W_{out}$) or
$\mathbf{v}_{\text{dormant}}$ (because then we would have $p_A\approx p_B$).

We make the concepts
of causally disconnected and dormant subspaces formal in Subsection
\ref{sub:formalize-intuitions}.  We also remark that, under the assumptions of
the above discussion, the optimal illusory patch will provably combine the
disconnected and dormant directions with equal weight $\frac{1}{\sqrt{2}}$ as in
Equation \ref{eq:illusion-decomposition}; the proof is given in Appendix
\ref{app:equal-norms}.

\subsection{The Illusion in a Toy Model}
\label{sub:toy-illusion}
With these concepts in mind, we can construct a distilled example of the
illusion in a toy (linear) neural network. Specifically, consider a network
$\mathcal{M}$ that takes in an input $x\in\mathbb{R}$, computes a three-dimensional
hidden representation $\mathbf{h} = x\mathbf{w}_1$, and then a real-valued
output $y = \mathbf{w}_2^T \mathbf{h}$. Define the weights to be 
\begin{align*}
    \mathbf{w}_1 = \l(1, 0, 1\r), \quad \text{and}\quad \mathbf{w}_2 = \l(0, 2, 1\r)
\end{align*}
and observe that $\mathcal{M}\l(x\r)=x$, i.e. the network computes the identity
function:
\begin{align*}
    x \quad\mapsto\quad \mathbf{h} = (x, 0, x) \quad\mapsto\quad y = 0\times x + 2\times 0 +1\times x = x.
\end{align*}
This network is illustrated in Figure
\ref{fig:mathematical-example-diagram-main}. We can analyze the 1-dimensional
subspaces (directions) spanned by each of the three hidden activations:
\begin{itemize}
\item the $h_1$ direction is causally disconnected: setting it to any value has
no effect on the output;
\item the $h_2$ direction is dormant: it is constant (always $0$) on the data,
but setting it to some other value will affect the model's output;
\item the $h_3$ direction mediates the signal through the network: the input $x$
is copied to it, and is in turn copied to the output\footnote{
An important note on this particular example is that the distinction between
causally disconnected, dormant and faithful to the computation directions is
artificial, and here it is only used for exposition. In particular, we show in
Appendix \ref{app:rotated-basis} that re-parametrizing the hidden layer of the
network via a rotation makes $\mathbf{v}_{\text{illusory}}$ take the role of the
faithful direction $\mathbf{e}_3$, and the two other (rotated) basis vectors
become a disconnected/dormant pair. By contrast, when we exhibit the illusion in
real-world scenarios, a reparametrization of this kind would need to combine
activations between different model components, such as MLP layers and residual
stream activations. We return to this point in Section \ref{sec:discussion}.
}.
\end{itemize}
As expected, patching along the direction $h_3$ overwrites the value of the $x$
feature (which in this example is identical to the input). That is, patching
along $h_3$ from $x'$ into $x$ makes the network output $x'$ instead of $x$.

However, patching along the sum of the causally disconnected direction $h_1$ and
the dormant direction $h_2$ represented by the unit vector
$\mathbf{v}_{\text{illusory}}=\l(\frac{1}{\sqrt{2}}, \frac{1}{\sqrt{2}}, 0\r)$
has the same effect: using Equation \ref{eq:onedim-patch}, patching from $x'$
into $x$ along $\mathbf{v}_{\text{illusory}}$ results in the hidden activation
\begin{align*}
    \mathbf{h}^{\text{patched}} = \l(\frac{x+x'}{2}, \frac{x'-x}{2}, x\r)^\top
\end{align*}
which when multiplied with $\mathbf{w}_2$ gives the final output $2\times
\frac{x'-x}{2} + 1\times x = x'$.

\begin{figure}
  \begin{center}
\resizebox{0.5\textwidth}{!}{
\begin{tikzpicture}[state/.style={circle, draw, 
  minimum size=2.5cm, font=\LARGE}]
\node[state] (X) {$x$};
\node[state] (H1) at (5,4) {$h_1 \gets x$};
\node[state] (H2) at (5,0) {$h_2\gets 0$};
\node[state] (H3) at (5,-4) {$h_3 \gets x$};
\node[state, right of= H2] at (10,0) (Y) {$y \gets x$};

\draw[-{Stealth[scale=2]}] (X) -> (H1) node[midway, above left, font=\LARGE] {$\times 1$};
\draw[-{Stealth[scale=2]}] (X) -> (H3) node[midway, below left, font=\LARGE] {$\times 1$};
\draw[-{Stealth[scale=2]}] (H2) -> (Y) node[midway, below, font=\LARGE] {$\times
2$};
\draw[-{Stealth[scale=2]}] (H3) -> (Y) node[midway, below, font=\LARGE] {$\times
1$};
\end{tikzpicture}
}
  \end{center}
\caption{A network $\mathcal{M}$ illustrating the illusion. The network computes
the identity function: $\mathcal{M}(x)=x$. The activation of the input, output
and each hidden neuron for a generic input $x$ are shown in the circles, with
arrows indicating the weight of the connections (no arrow means a weight of
$0$). The hidden unit $h_3$ stores the value of the input and passes this to the
output, while the unit $h_2$ is dormant and $h_1$ is disconnected from the output. However,
activation patching the 1-dimensional linear subspace spanned by the sum of the
$h_1$ and $h_2$ basis vectors (defined by the unit vector
$\mathbf{v}=(\frac{1}{\sqrt{2}},\frac{1}{\sqrt{2}}, 0)$) has the same effect on
model behavior as patching just the unit $h_3$.}
\label{fig:mathematical-example-diagram-main}
\end{figure}
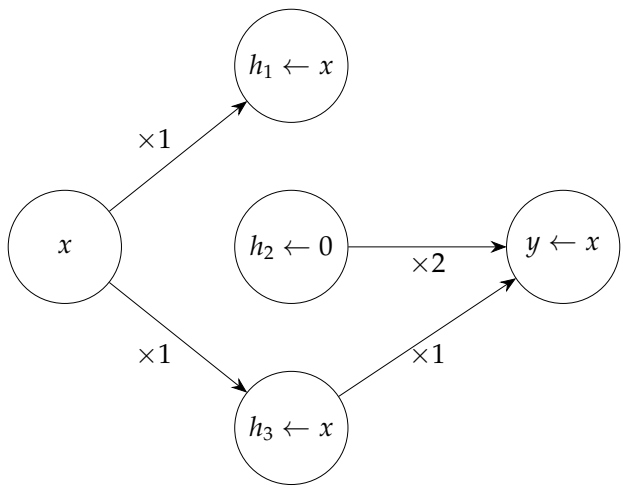

\subsection{Detecting the illusion in practice}
\label{sub:detecting-illusion}

How can we tell if this kind of phenomenon occurs for a given subspace activation
patch? Given a subspace spanned by a unit vector $\mathbf{v}$, suppose that
activation patching along this subspace has an effect on model outputs
consistent with changing the property that varies between the examples being
patched. We can attempt to
decompose it orthogonally into a causally disconnected part and a dormant part,
and argue that each of these parts has the properties described in the above
sections. 

Specifically, when $\mathbf{v}$ is in the post-GELU activations of an MLP layer
in a transformer with down-projection $W_{out}$ (almost all examples in this
paper are of this form), it is clear that the orthogonal
projection of $\mathbf{v}$ on $\ker W_{out}$ is causally disconnected. This
suggests writing $\mathbf{v} = \mathbf{v}_{\text{null}} +
\mathbf{v}_{\text{row}}$ where $\mathbf{v}_{\text{null}}\in\ker W_{out}$ is the
orthogonal projection on $\ker W_{out}$, and $\mathbf{v}_{\text{row}}$ is the
remainder, which is in $\l(\ker W_{out}\r)^\perp$, the rowspace of $W_{out}$.
Using this decomposition, we can perform several experiments:
\begin{itemize}
\item compare the strength of the patch to patching along the subspace spanned
by $\mathbf{v}_{\text{row}}$ alone, obtained by removing the causally
disconnected part of $\mathbf{v}$. If $\mathbf{v}_{\text{row}}$ is
indeed dormant as we hope to show, this patch should have no effect on model outputs; in
reality, $\mathbf{v}_{\text{row}}$ may only be approximately dormant, so the
patch may have a small effect. Conversely, if this patch has an effect similar
to the original patch along $\mathbf{v}$, this is evidence against the illusion;
\item check how dormant $\mathbf{v}_{\text{row}}$ is compared to
$\mathbf{v}_{\text{null}}$ by comparing the spread of projections of the
examples on both directions.
\end{itemize}

We use these experiments, as well as others, throughout the paper in order to
rule out or confirm the illusion.

\subsection{Formalization of Causally Disconnected and Dormant Subspaces}
\label{sub:formalize-intuitions}
For completeness, in this subsection we give a (somewhat) formal treatment of
the intuitive ideas introduced in the previous subsection. Readers may also want
to consult Appendix \ref{app:higher-dim} for background on patching
higher-dimensional subspaces, which is used to define these concepts. 

Let $\mathcal{M}:\mathcal{X}\to\mathcal{O}$ be a machine learning model that on input $x\in\mathcal{X}$ outputs a vector $y\in\mathcal{O}$ of probabilities over a set of output classes. Let $\mathcal{D}$ be a distribution over $\mathcal{X}$, and $\mathcal{C}$ be a component of $\mathcal{M}$, such that for $x\sim\mathcal{D}$ the hidden activation of $\mathcal{C}$ is a vector $c_x \in\mathbb{R}^d$. For a subspace $U_{\mathcal{C}}\subset\mathbb{R}^d$, we let $u_x$ be the orthogonal projection of $c_x$ onto $U_{\mathcal{C}}$. Finally, let $\mathcal{M}_{U_{\mathcal{C}} \leftarrow u_y}(x)$ be the result of running $\mathcal{M}$ with the input $x$ and setting the subspace $U_{\mathcal{C}}$ patched to $u_y$.

We say $U_{\mathcal{C}}$ is \emph{causally disconnected} if $\mathcal{M}_{U_{\mathcal{C}}
\leftarrow u'}(x) = \mathcal{M}(x)$ for all $u' \in U_{\mathcal{C}}$. In other
words, intervening on the model by setting
the orthogonal projection of $\mathcal{C}$'s activation on $U_{\mathcal{C}}$ to any other
value does not change the model's outputs. For a concrete example of a causally
disconnected subspace, consider an MLP layer in a transformer model with an
output projection matrix $W_{out}$; then, $\ker W_{out}$ is a causally
disconnected subspace of the hidden (post-nonlinearity) activations of the MLP layer.

We say $U_{\mathcal{C}}$ is \emph{dormant} if $\mathcal{M}_{U_{\mathcal{C}} \leftarrow u_y}(x)
\approx \mathcal{M}(x)$ with high probability over $x, y\sim\mathcal{D}$, but
there exists $u'\in \mathbb{R}^d$ such that $\mathcal{M}_{U_{\mathcal{C}}
\leftarrow u'}(x)$ is substantially different from $\mathcal{M}(x)$ (e.g.,
significantly changes the model's confidence on the task's answer).
In other words, a dormant subspace is approximately causally disconnected when
we patch its value using activations realized under the distribution
$\mathcal{D}$, but can have substantial causal effect if set to
other values.

\section{The Illusion in the Indirect Object Identification Task}
\label{sec:ioi}

\subsection{Preliminaries}
In \citet{wang2022interpretability}, the authors analyze how the decoder-only transformer language model GPT-2 Small \citep{radford2019language} performs the \emph{indirect object identification} task. 
In this task, the model is required to complete sentences of the form `When Mary
and John went to the store, John gave a bottle of milk to' (with the intended
completion in this case being ` Mary'). We refer to the repeated name (John) as
\textbf{S} (the subject) and the non-repeated name (Mary) as \textbf{IO} (the
indirect object). For each choice of the \textbf{IO} and \textbf{S} names, there
are two patterns the sentence can have: one where the \textbf{IO} name comes
first (we call these `ABB examples'), and one where it comes second (we call
these `BAB examples').
Additional details on the data distribution, model and task performance are given in Appendix \ref{app:ioi-dataset-details}.

\citet{wang2022interpretability} suggest the model uses the algorithm `Find the
two names in the sentence, detect the repeated name, and predict the
non-repeated name' to do this task.  In particular, they find a set of four
heads in layers 7 and 8 -- the \textbf{S-Inhibition heads} -- that output the
signal responsible for \emph{not} predicting the repeated name.  The dominant
part of this signal is of the form `Don't attend to the name in first/second
position in the first sentence' depending on where the \textbf{S} name appears
(see Appendix A in \citet{wang2022interpretability} for details). In other
words, this signal detects whether the example is an ABB or BAB example. This signal
is added to the residual stream\footnote{We follow the conventions of
\citet{elhage2021mathematical} when describing internals of transformer models.
The residual stream at layer $k$ is the sum of the output of all layers up to
$k-1$, and is the input into layer $k$.} at the last token position, and is then
picked up by another class of heads in layers 9, 10 and 11 -- the \textbf{Name
Mover heads} -- which incorporate it in their queries to shift attention to the
\textbf{IO} name and copy it to the last token position, so that it can be
predicted (Figure \ref{fig:ioi-diagram}).

\begin{wrapfigure}{R}{0.5\textwidth}
    \centering
    \begin{tikzpicture}
        \node[anchor=south west,inner sep=0] (image) at (0,0)
        {\includegraphics[width=0.5\textwidth]{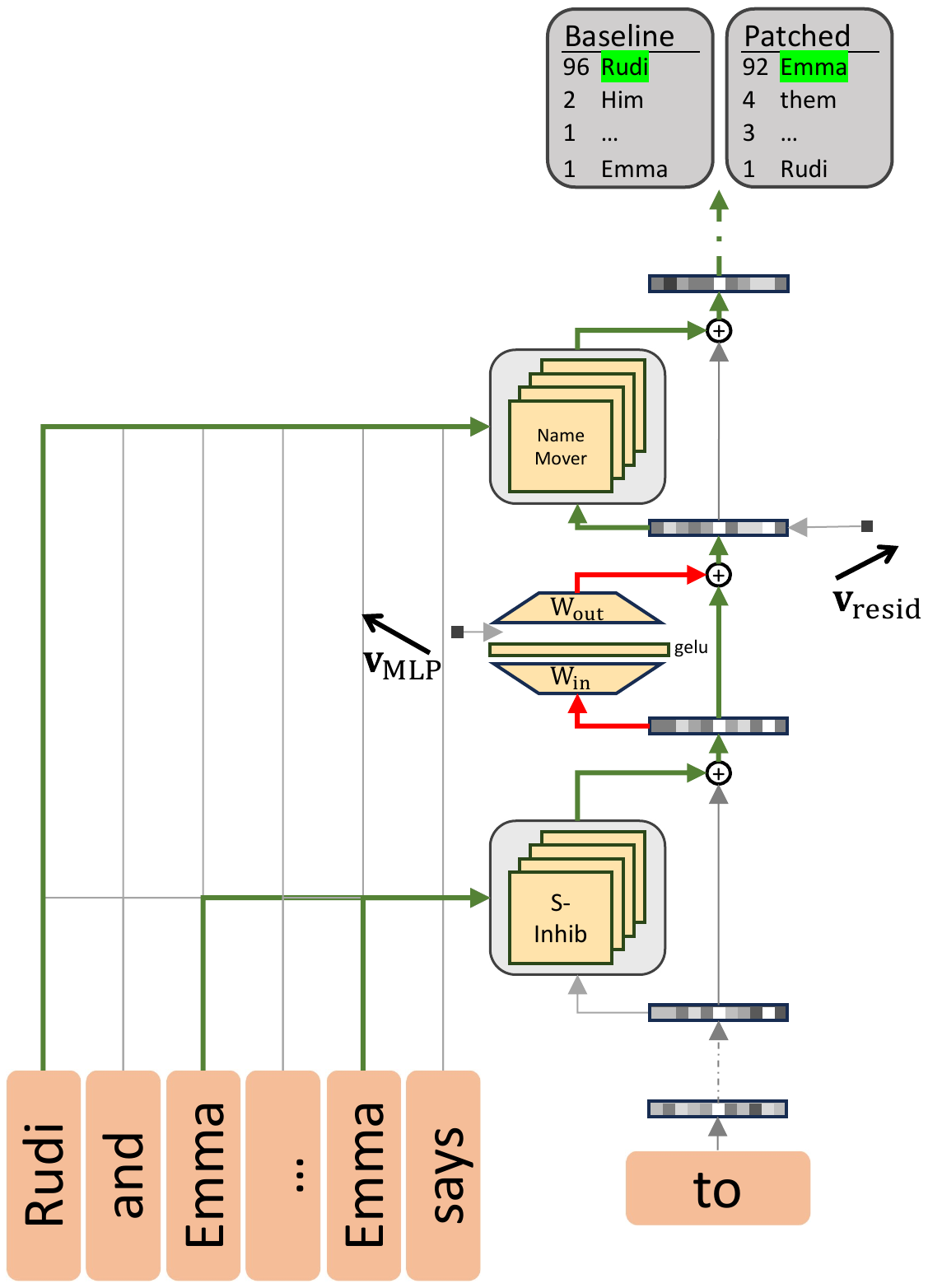}};
    \end{tikzpicture}
    \caption{Schematic of the IOI circuit and locations of key interventions.
    As argued in \citet{wang2022interpretability}, GPT2-Small predicts the
    correct name by S-inhibition heads writing positional information to the
    residual stream, which is used by the name movers to copy the non-duplicated
    name (green arrows). Location of subspace interventions
    $\textbf{v}_{\text{MLP}}$ (analyzed in Subsection \ref{sub:demonstrating-mlp-illusion}) and
    $\textbf{v}_{\text{resid}}$ (analyzed in Section \ref{sec:ground-truth}) are
    marked.  Patching the illusory subspace $\textbf{v}_{\text{MLP}}$ adds a new
    path (red) along the established one that is used to flip positional
    information when patched.}
    \label{fig:ioi-diagram}
\end{wrapfigure}

\subsection{Finding Subspaces Mediating Name Position Information}
\label{sub:ioi-methodology}
How, precisely, is the positional signal communicated?  In particular, `don't
attend to the first/second name' is plausibly a binary feature represented by a
1-dimensional subspace. In this subsection, we present methods to look for such
a subspace.

\textbf{Gradient of name mover attention scores}. As shown in
\cite{wang2022interpretability}, the three name mover heads identified therein
will attend to one of the names, and the model will predict whichever name is
attended to. The position feature matters mechanistically by determining whether
they attend to \textbf{IO} over \textbf{S}. This motivates us to consider the
gradient $\mathbf{v}_{\text{grad}}$ of the difference of attention scores of
these heads on the \textbf{S} and \textbf{IO} names with respect to the residual
stream at the last token, right after layer 8. We choose this layer because it
right after the S-Inhibition heads (in layers 7 and 8) and before the name mover
heads (in layers 9 and 10); see Figure \ref{fig:ioi-diagram}. This gradient is
the direction in the space of residual stream activations at this location that maximally shifts
attention between the two names (per unit $\ell_2$ norm), so we expect it to be
a strong mediator of the position signal. Implementation details are given in Appendix
\ref{app:gradient-details}.

Importantly, the transformation from residual stream activations to attention
scores is an approximately linear map: it consists of layer normalization
followed by matrix multiplication. Layer normalization is a linear operation
modulo the scaling step, and empirically, the scales of different examples in a
trained model at inference time are tightly concentrated (see also `Handling
Layer Normalization' in \citet{elhage2021mathematical}). This justifies the use
of the gradient -- which is in general only locally meaningful -- as a direction
in the residual stream globally meaningful for the attention scores of the name
mover heads.

\textbf{Distributed alignment search}. We can also directly optimize for a
direction that mediates the position signal. This is the approach taken by DAS
\citep{geiger2023finding}. In our context, DAS optimizes for an activation
subspace which, when activation patched from prompt $B$ into prompt $A$, makes
the model behave as if the relative position of the \textbf{IO} and \textbf{S}
names in the sentence is as in prompt $B$. Specifically, if we patch between examples where the positions of the two names are the
same, we optimize for a patch that \emph{maximizes} the difference in predicted logits
for the \textbf{IO} and \textbf{S} names. Conversely, if we patch between examples where the positions of the two names are
switched, we optimize to \emph{minimize} this difference.
This approach is based purely on the
model's predictions, and does not make any assumptions about its internal
computations. 

We let $\mathbf{v}_{\text{MLP}}$ and $\mathbf{v}_{\text{resid}}$
be 1-dimensional subspaces found by DAS in the layer 8 MLP activations and layer
8 residual stream output at the last token, respectively (see Figure
\ref{fig:ioi-diagram}). Both of these locations are between the
S-Inhibition and Name Mover heads; however, \cite{wang2022interpretability} did
not find any significant contribution from the MLP layer, making it a potential
location for our illusion. Implementation details are given in Apendix
\ref{app:das-training-details}.

\subsection{Measuring Patching Success via the Logit Difference Metric}
\label{sub:ioi-metrics}

In our experiments, we perform all patches between examples that only differ in
the variable we want to localize in the model, i.e. the position of the
\textbf{S} and \textbf{IO} names in the first sentence. In other words, we patch
from an example of the form `Then, Mary and John went to the store. John gave a
book to' (an ABB example) into the corresponding example `Then, John and Mary went to the store.
John gave a book to' (a BAB example), and vice-versa. Our activation patches have the goal of
making the model output the \textbf{S} name instead of the \textbf{IO} name. 

Accordingly, we use the \emph{logit difference} between the logits assigned to
the \textbf{IO} and \textbf{S} names as our main measure of how well a patch
performs. We note that the logit difference is a meaningful quantifier of the
model's confidence for one name over the other (it is equal to the log-odds
between the two names assigned by the model), and has been extensively used in
the original IOI circuit work \citet{wang2022interpretability} to measure
success on the IOI task.

Given a prompt $x$ from the IOI distribution, let
$\operatorname{logit}_{\textbf{IO}}\l(x\r),
\operatorname{logit}_{\textbf{S}}\l(x\r)$ denote the last-token logits output by
the model on input $x$, for the \textbf{IO} and \textbf{S} names in the prompt
$x$ respectively (note that in our IOI distribution, all names are single tokens
in the vocabulary of the model).
The logit difference
\begin{align*}
    \operatorname{logitdiff}\l(x\r) := \operatorname{logit}_{\textbf{IO}}\l(x\r) - \operatorname{logit}_{\textbf{S}} \l(x\r)
\end{align*}
when $x$ is sampled from the IOI distribution is $>0$ for almost all examples
($99+\%$), and is on average $\approx 3.5$ (for this average value, the probability
ratio in favor of the \textbf{IO} name is $e^{3.5}\approx 33$).

Similarly, for an activation patching intervention $\iota$, let
$\operatorname{logit}_{\textbf{IO}}^{\iota(x\gets x')}(x),
\operatorname{logit}_{\textbf{S}}^{\iota(x\gets x')}(x)$ denote the last-token
logits output by the model when run on input $x$ but patching from $x'$ using
$\iota$. The logit difference after intervening via $\iota$ is thus
\begin{align*}
    \operatorname{logitdiff}_{\iota(x\gets x')}\l(x\r) := 
\operatorname{logit}_{\textbf{IO}}^{\iota(x\gets x')}(x) -
    \operatorname{logit}_{\textbf{S}}^{\iota(x\gets x')}(x)
\end{align*}
Our main metric is the average \textbf{fractional logit difference decrease
(FLDD)} due to the intervention $\iota$, where
\begin{align}
\label{eq:fractional-logit-diff}
    \operatorname{FLDD}_{\iota \l(x\gets x'\r)}(x) 
    :=\frac{\operatorname{logitdiff}(x) - \operatorname{logitdiff}_{\iota(x\gets
    x')}}
    {\operatorname{logitdiff}(x)}
    = 1 - \frac{\operatorname{logitdiff}_{\iota(x\gets
    x')}}{\operatorname{logitdiff}(x)}
\end{align}
The average FLDD is $0$ when the patch does not, on average, change the model's
log-odds. The more positive FLDD is, the more successful the patch, with values
above $100\%$ indicating that the patch more often than not makes the model
prefer the \textbf{S} name over the \textbf{IO} name. Finally, an average FLDD
below $0\%$ means that the patch on average helps the model do the task (and
thus the patch has failed).

We also measure the \textbf{interchange accuracy} of the intervention: the 
fraction of patches for which the model predicts the \textbf{S} (i.e., wrong)
name for the patched run. This is a `hard' 0-1 counterpart to the FLDD metric.

\paragraph{Why prefer the FLDD metric to interchange accuracy?} We argue that
our main metric, which is based on logit difference (Equation
\ref{eq:fractional-logit-diff}), is a better reflection of the success of a
patch than the accuracy-based interchange accuracy.  Specifically, there are
practical cases (e.g. the results in Subsection
\ref{sub:demonstrating-mlp-illusion}) in which an intervention consistently
achieves FLDD $\approx 50\%$, even though the interchange accuracy is $\approx
0\%$. In practice, circuits often have multiple components contributing to the
same signal (including the IOI circuit found in
\citet{wang2022interpretability}). So a single non-residual-stream component
consistently responsible for shifting $50\%$ of the model's log-odds towards
another prediction is significant (even more so for a low-dimensional subspace
of the component). Indeed, even if this component's contribution alone is
insufficient to cause the predicted token to change, three such components would
robustly change the prediction.

\subsection{Results: Demonstrating the Illusion for the
$\mathbf{v}_{\text{MLP}}$ Direction}
\label{sub:demonstrating-mlp-illusion}

\begin{figure}
  \begin{minipage}[b]{.42\linewidth}
    \centering
    \begin{tikzpicture}
        \node[anchor=south west,inner sep=0] (image) at (0,0) {\includegraphics[width=1.0\textwidth]{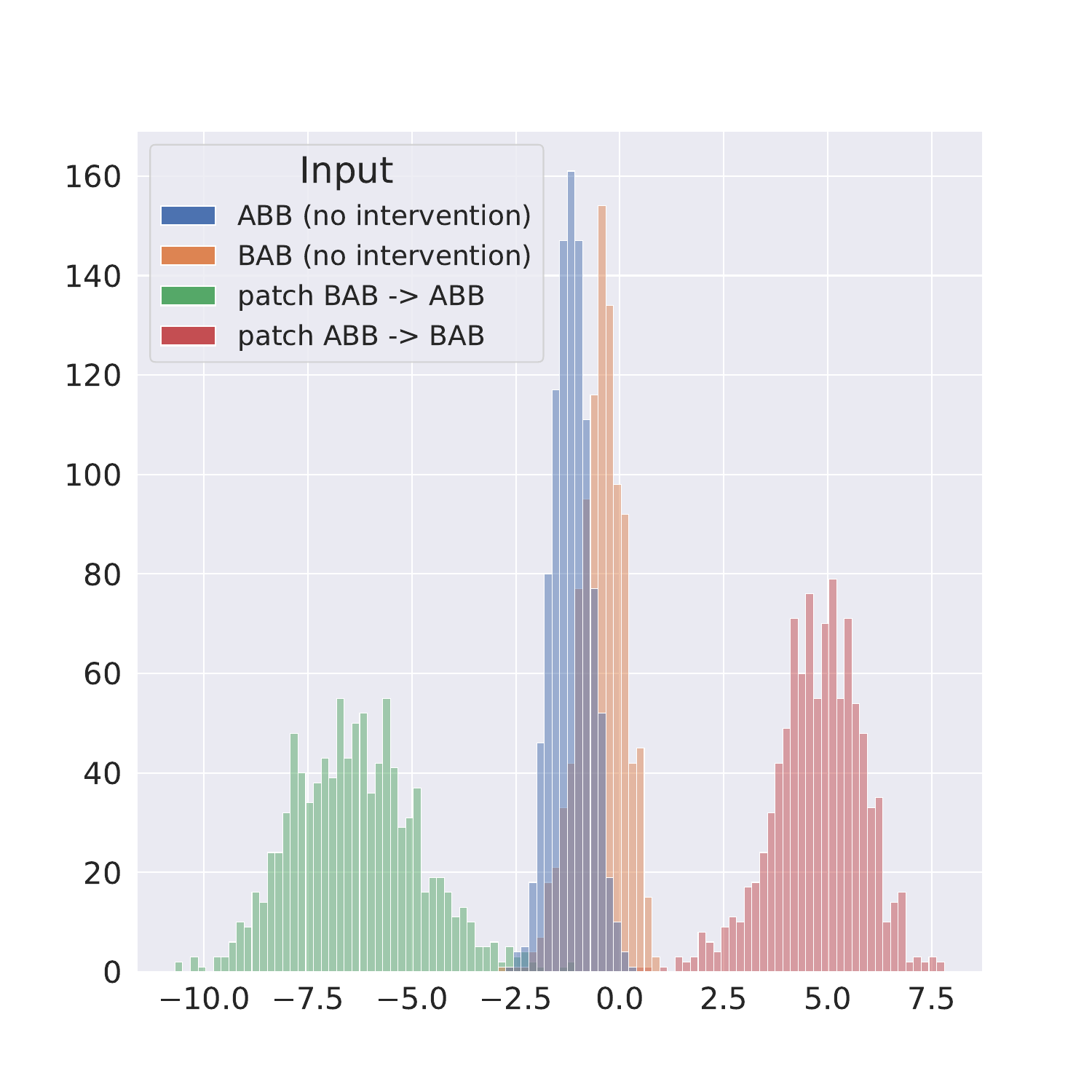}};
        \begin{scope}[x={(image.south east)},y={(image.north west)}]
            \node [anchor=east, rotate=90] at (-0.01,0.75) {Number of examples};
            \node [anchor=north] at (0.5,0.05) {Activation projection};
        \end{scope}
    \end{tikzpicture}
    \captionof{figure}{
    Projections of the output of the MLP layer on the gradient direction
    $\mathbf{v}_{\text{grad}}$ before (blue/orange) and after (green/red) the
    activation patch along $\mathbf{v}_{\text{MLP}}$. In the legend, `ABB'
    denotes examples where the \textbf{IO} name comes before the \textbf{S}
    name, and `BAB' the other kind of examples. 
    \\    
    While before the patch the contribution of the MLP layer to the causally
    relevant direction $\mathbf{v}_{\text{grad}}$ distinguishes between values
    of the \textbf{IO} position in the prompt, after the patch there is a strong
    distinction (in the opposite direction). This shows that the patch activates
    a potential mediator of this feature that is normally dormant, taking model
    activations off-distribution.
    }%
    \label{fig:mlp8-output-projections}
  \end{minipage}
  \hspace{0.5cm}
  \begin{minipage}[b]{.53\linewidth}
    \centering
    \begin{tabular}{lrr}
    \toprule
    \multicolumn{1}{p{1cm}}{\centering Patching\\ subspace} &
    \multicolumn{1}{p{1.5cm}}{\centering FLDD} &
    \multicolumn{1}{p{1.8cm}}{\centering Interchange\\ accuracy} \\
    \midrule
    full MLP & -8\% & 0.0\% \\
    $\mathbf{v}_{\text{MLP}}$ & 46.7\% & 4.2\% \\
    $\mathbf{v}_{\text{MLP}}$ rowspace & 13.5\% & 0.2\% \\
    $\mathbf{v}_{\text{MLP}}$ nullspace & 0\% & 0.0\% \\
    full residual stream & 123.6\% & 54.8\% \\
    $\mathbf{v}_{\text{resid}}$ & 140.7\% & 74.8\% \\
    $\mathbf{v}_{\text{resid}}$ rowspace & 127.5\% & 63.1\% \\
    $\mathbf{v}_{\text{resid}}$ nullspace & 13.9\% & 0.4\% \\
    $\mathbf{v}_{\text{grad}}$ & 111.5\% & 45.1\% \\
    $\mathbf{v}_{\text{grad}}$ rowspace & 106.47\% & 40.6\% \\
    $\mathbf{v}_{\text{grad}}$ nullspace & 2.2\% & 0.0\% \\
    \bottomrule
    \end{tabular}
    \captionof{table}{Effects of activation patching of full components and
    1-dimensional subspaces on the IOI task: fractional logit difference
    decrease (FLDD, higher means more successful patch; $0\%$ means no change)
    and interchange accuracy (fraction of predictions flipped; higher means more
    successful patch). 
    \\
    The first 5 interventions are described in more detail in Section
    \ref{sec:ioi}, and the next 6 in Section \ref{sec:ground-truth}.
    \\
    An FLDD metric of $>100\%$ indicates that the patch is more successful than
    not on average; however, an FLDD of $\approx 50\%$ is still significant,
    even if the associated interchange accuracy may be $\approx 0\%$. See
    Subsection \ref{sub:ioi-metrics} for more on interpreting the FLDD metric.
    }
    \label{tab:ioi-metrics}
  \end{minipage}
\end{figure}

We now show that patching the $\mathbf{v}_{\text{MLP}}$ direction exhibits the
illusion from Section \ref{sec:methods}. By contrast, we revisit
$\mathbf{v}_{\text{grad}}$ and $\mathbf{v}_{\text{resid}}$ in Section
\ref{sec:ground-truth}, where we show that both are representations of the name
position information that are highly faithful to the model's computation.

\paragraph{Methodology and interventions considered} To contextualize the effect
of the $\textbf{v}_{\text{MLP}}$ patch, we compare it to several additional
subspace- and component-level activation patching interventions:
\begin{itemize}
\item \textbf{full MLP}: patching the full value of the hidden activation of the
8-th MLP layer at the last token.
\item \textbf{$\mathbf{v}_{\text{MLP}}$}: patching along the 1-dimensional
subspace spanned by
the direction $\textbf{v}_{\text{MLP}}$ found in Subsection
\ref{sub:ioi-methodology}.
\item \textbf{$\mathbf{v}_{\text{MLP}}$ nullspace}: patching along the
1-dimensional subspace spanned by the causally
disconnected component of $\mathbf{v}_{\text{MLP}}$. This is the orthogonal
projection $\mathbf{v}_{\text{MLP}}^{\text{nullspace}}$ of
$\mathbf{v}_{\text{MLP}}$ on the nullspace $\ker W_{out}$ of the down-projection
$W_{out}$ of the MLP layer. Note that $W_{out}\in \mathbb{R}^{768\times 3072}$,
so its kernel occupies at least $2304$ dimensions, or $3/4$ of the total
dimension of the space of MLP activations.
\item \textbf{$\mathbf{v}_{\text{MLP}}$ rowspace}: patching along the
1-dimensional subspace spanned by causally relevant
component of $\mathbf{v}_{\text{MLP}}$. This is the orthogonal projection
$\mathbf{v}_{\text{MLP}}^{\text{rowspace}}$ of $\mathbf{v}_{\text{MLP}}$ on the
rowspace of $W_{out}$. Note that we have the orthogonal decomposition
\begin{align*}
    \mathbf{v}_{\text{MLP}} = \mathbf{v}_{\text{MLP}}^{\text{nullspace}} +
    \mathbf{v}_{\text{MLP}}^{\text{rowspace}}.
\end{align*}
\item \textbf{full residual stream}: patching the entire activation of the
residual stream at the last token after layer 8 of the model. This is indicated
as the location of $\mathbf{v}_{\text{resid}}$ in Figure \ref{fig:ioi-diagram}.
\end{itemize}

\paragraph{Results.} Metrics are shown in Table \ref{tab:ioi-metrics}. In
particular, we confirm the mechanics of the illusion are at play through the
following observations.

\subparagraph{The causally disconnected component of $\mathbf{v}_{\text{MLP}}$
drives the effect.}
While patching the $\mathbf{v}_{\text{MLP}}$ direction has a significant effect
on the FLDD metric ($46.7\%$), this effect is greatly diminished when we remove
the component of $\mathbf{v}_{\text{MLP}}$ in $\ker W_{out}$ whose activations
are (provably) causally disconnected from model predictions ($13.5\%$), or when
we patch the entire MLP activation ($-8\%$, actually increasing confidence).  By
contrast, performing analogous ablations on $\mathbf{v}_{\text{resid}}$ leads to
roughly the same numbers for the three analogous interventions ($140.7\% / 127.5\% / 123.6\%$; we refer the reader to
Section \ref{sec:ground-truth} for details on the $\mathbf{v}_{\text{resid}}$
experiments). 

\subparagraph{Patching $\mathbf{v}_{\text{MLP}}$ activates a dormant pathway
through the MLP.} To corroborate these findings, in Figure
\ref{fig:mlp8-output-projections}, we plot the projection of the MLP layer's
contribution to the residual stream on the gradient direction
$\mathbf{v}_{\text{grad}}$ before and after patching, in order to see how it
contributes to the attention of name mover heads.  We observe that in the
absence of intervention, the MLP output is weakly sensitive to the name position
information, whereas after the patch this changes significantly.

\subparagraph{Further validations of the illusion.}

We observe that the
disconnected-dormant decomposition from the illusion approximately holds: the
causally disconnected component of $\mathbf{v}_{\text{MLP}}$ (the one in $\ker
W_{out}$) is significantly more discriminating between ABB and BAB examples than
the component in $\l(\ker W_{out}\r)^\top$, which is the one driving the causal
effect (Figure \ref{fig:mlp8-projections-by-pattern}); in this sense, the
causally relevant component is `dormant' relative to the causally diconnected
one\footnote{The projection of $\mathbf{v}_{\text{MLP}}$ onto $\ker W_{out}$ is
substantial: it has norm $\approx 0.65$, and the orthogonal projection onto
$\l(\ker W_{out}\r)^\top$ has norm $\approx 0.75$ (as predicted by our model,
the two components are approximately equal in norm; see Appendix
\ref{app:equal-norms}).}. 

While the contribution of the $\mathbf{v}_{\text{MLP}}$ patch
to the FLDD metric may appear relatively small, and the interchange accuracy of
this intervention is very low, in Appendix
\ref{app:illusion-magnitude} we argue that this is significant for a single
component. 

A potential concern when evaluating these results is overfitting by DAS. In our
experiments, we always evaluate trained subspaces on a held-out test dataset
which uses different names, objects, places and templates for the sentences;
this makes sure that we learn a general (relative to our IOI distribution)
subspace representing position information and not a subspace that only works
for particular names or other details of the sentence. We investigated
overfitting in DAS further in Appendix \ref{app:overfitting}, and found that
when DAS is trained on a dataset with a small number of names, overfitting is a
real concern. However, the extent of overfitting is not such that DAS works in
layers of the model where a generalizing DAS solution can also be found.

Another potential concerns is that the model could be somehow representing the
position information in the MLP layer in a higher-dimensional subspace, and that
our 1-dimensional intervention is perhaps not fit to illuminate the properties
of that larger representation. In Appendix \ref{app:generalization-high-dim}, we
show that the illusion occurs when patching 100-dimensional subspaces as well,
and the quantitative effect of the illusion is just a little stronger than that
for 1-dimensional subspaces (as measured by the FLDD metric).

Finally, in Appendix \ref{app:random-mlp}, we show that we can find a
direction within the post-$\operatorname{gelu}$ activations that has an even
stronger effect on the model's behavior, \emph{even when} we replace the MLP
weights with random matrices.

\begin{figure}[ht]
    \centering
    \begin{minipage}[b]{0.45\textwidth}
    \begin{tikzpicture}
        \node[anchor=south west,inner sep=0] (image) at (0,0) {\includegraphics[width=1.0\textwidth]{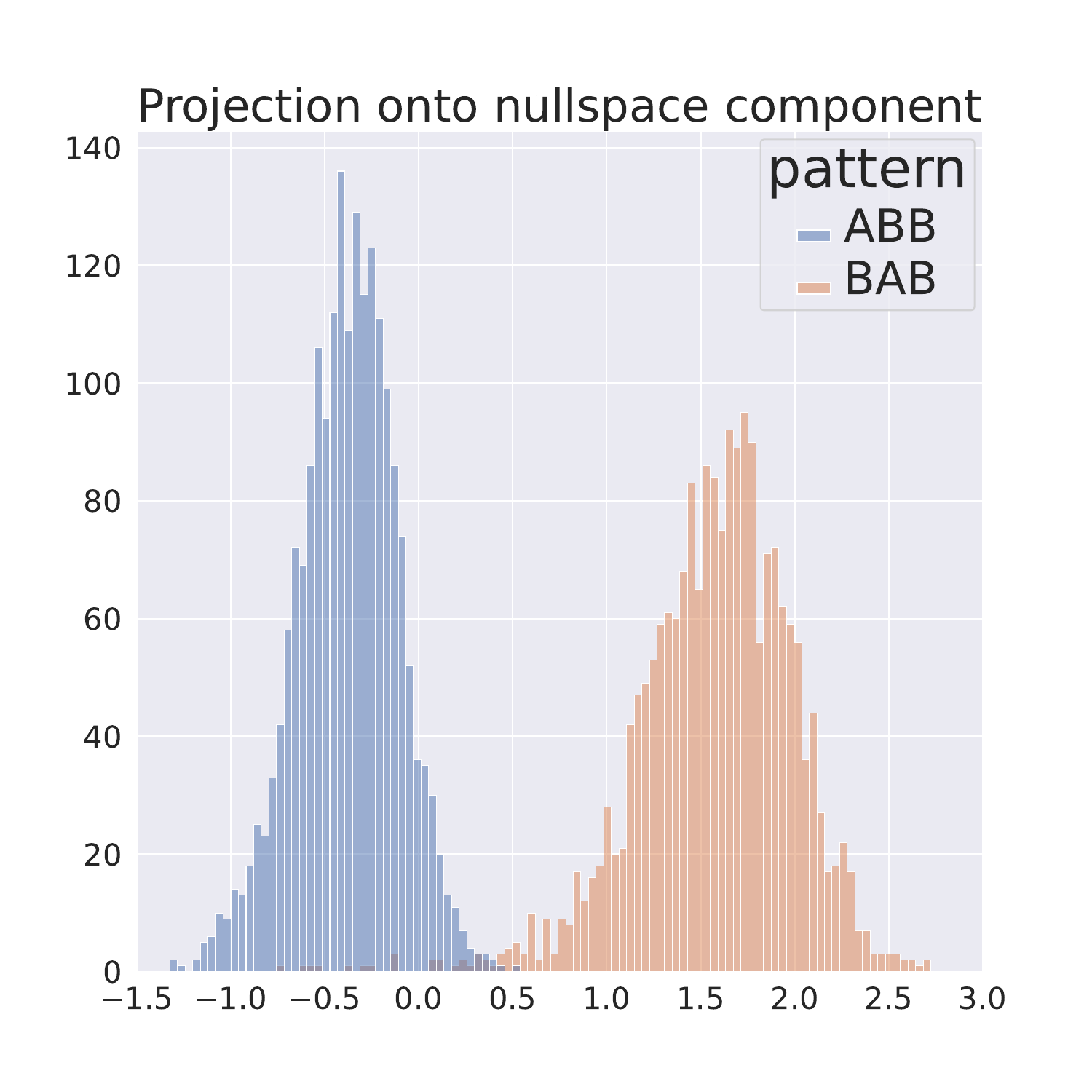}};
        \begin{scope}[x={(image.south east)},y={(image.north west)}]
            \node [anchor=east, rotate=90] at (-0.05,0.8) {Number of examples};
            \node [anchor=north] at (0.5,0.0) {Activation projection};
        \end{scope}
    \end{tikzpicture}
    \end{minipage}
    \begin{minipage}[b]{0.45\textwidth}
    \begin{tikzpicture}
        \node[anchor=south west,inner sep=0] (image) at (0,0) {\includegraphics[width=1.0\textwidth]{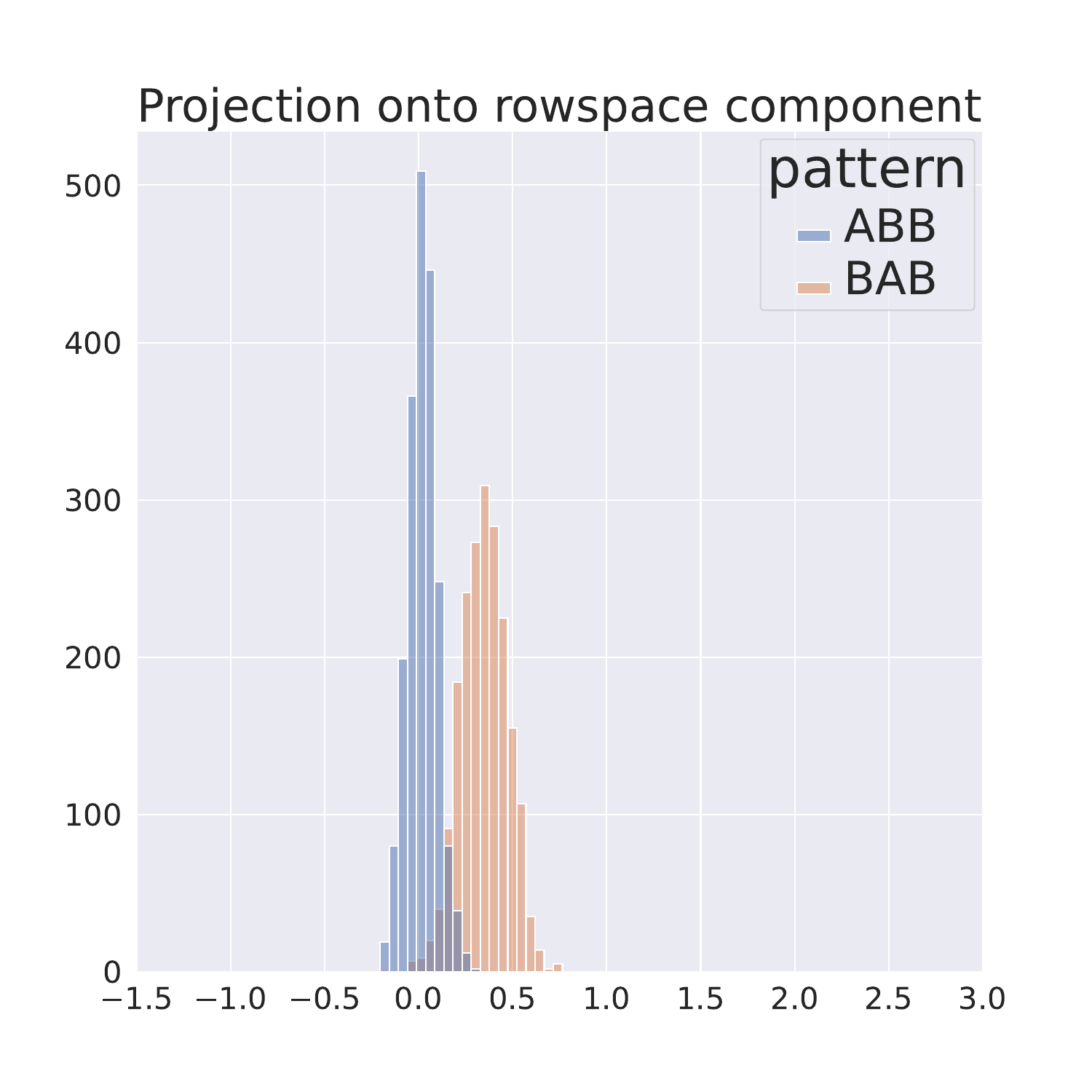}};
        \begin{scope}[x={(image.south east)},y={(image.north west)}]
            \node [anchor=north] at (0.5,0.0) {Activation projection};
        \end{scope}
    \end{tikzpicture}
    \end{minipage}
    
    \caption{Projections of dataset examples onto the two (normalized to unit $\ell_2$ norm) components of the illusory patching direction found in MLP8: the nullspace (irrelevant) component (left), and the rowspace (dormant) component (right).}
    \label{fig:mlp8-projections-by-pattern}
\end{figure}

\section{Finding and Validating a Faithful Direction Mediating Name Position in the IOI Task}
\label{sec:ground-truth}

As a counterpoint to the illusion, in this section we demonstrate a success case
for subspace activation patching, as well as for DAS as a method for finding
meaningful subspaces, by revisiting the directions $\mathbf{v}_{\text{grad}}$
and $\mathbf{v}_{\text{resid}}$ defined in Subsection \ref{sub:ioi-methodology},
and arguing they are faithful to the model's computation to a high degree. 

Specifically, we subject these directions to the same tests we used for the
illusory direction $\mathbf{v}_{\text{MLP}}$, and arrive at significantly
different results. Through these and additional validations, we demonstrate that
these directions possess the necessary and sufficient properties of a successful
activation patch -- being both correlated with input variation and causal for
the targeted behavior -- in a way that does not rely on a large causally
disconnected component for the effect.

\subsection{Ruling Out the Illusion}
\label{sub:ruling-out-illusion}
Intuitively, the main property of $\mathbf{v}_{\text{resid}}$ we want to
establish in order to rule out the illusion is that it is simultaneously (1) strongly
discriminating between ABB and BAB examples (i.e., projections of activations on
$\mathbf{v}_{\text{resid}}$ separate these two classes well), and (2) is highly aligned with the
direction $\mathbf{v}_{\text{grad}}$ that downstream model components read this
information from in order to put attention on the \textbf{IO} name and not the
\textbf{S} name. To this end, we define a notion of
causally disconnected component for $\mathbf{v}_{\text{resid}}$, and we show
that removing it does not diminish the effect of the patch; we further show that
$\mathbf{v}_{\text{resid}}$ and $\mathbf{v}_{\text{grad}}$ are quite similar,
and that $\mathbf{v}_{\text{grad}}$ is also strongly activated by position information.

\paragraph{What is the causally (dis)connected subspace of the residual stream?}
While for an MLP layer it is clear that $\ker W_{out}$ is the subspace of
post-GELU activations which is causally disconnected from model outputs, the
residual stream after layer 8 has no subspace which is simultaneously in the
kernel of all downstream model components, or even of all the query matrices of
downstream attention heads (we checked this empirically).

To overcome this, recall from Section \ref{sec:ioi} that
\citet{wang2022interpretability} argued that the three Name Mover
heads in layers 9 and 10 are mostly responsible for the IOI task
specifically. Let
$W_Q^{\text{NM}}\in\mathbb{R}^{(3\times64)\times768}=\mathbb{R}^{192\times768}$
be the stacked query matrices of the three name mover heads (which are full-rank). We use the
192-dimensional subspace $\l(\ker W_Q^{NM}\r)^\top$ as a proxy for the causally
relevant subspace\footnote{Note that, while technically all attention heads in
layers 9, 10 and 11 read information from the residual stream after layer 8,
using their collective query matrices instead of just the name movers would lead
to a vacuous concept of a causally relevant subspace, because their collective
query matrices' rowspaces span the entire residual stream. As a rough baseline,
a random isotropic unit vector would have on average
$\sqrt{\frac{192}{768}}=\frac{1}{2}$ of its $\ell_2$-norm in $\l(\ker
W_Q^{NM}\r)^\top$. We also note that this is on par with the decomposition of
$\mathbf{v}_{\text{MLP}}$ we considered in Section \ref{sec:ioi}, where $\ker
W_{out}$ occupied $3/4$ of the dimension of the full space of activations.} of
the residual stream at the last token position after layer 8. 

To further narrow down the precise subspace read by the Name Mover heads, we
also compare $\mathbf{v}_{\text{resid}}$ with the gradient
$\mathbf{v}_{\text{grad}}$, which is the direction that the Name Mover's
attention on the \textbf{IO} vs. \textbf{S} name is most sensitive to.

\paragraph{Results.}
In Table \ref{tab:ioi-metrics}, we report the fractional logit difference
decrease (FLDD, recall Subsection \ref{sub:ioi-metrics}) and interchange
accuracy when patching $\mathbf{v}_{\text{resid}}$ and
$\mathbf{v}_{\text{grad}}$, as well as their components along $\ker
W_Q^{\text{NM}}$ (denoted `nullspace') and its orthogonal complement $\l(\ker
W_Q^{NM}\r)^\top$ (denoted `rowspace'). We observe that the non-nullspace
metrics are broadly similar\footnote{We observe that the
$\mathbf{v}_{\text{resid}}$ patch is more successful than the
$\mathbf{v}_{\text{grad}}$ patch; we conjecture that this is due to
$\mathbf{v}_{\text{resid}}$ being able to contribute to all downstream attention
heads, not just the three name-mover heads. In particular, the original IOI
paper \citet{wang2022interpretability} found that there is another class of
heads, Backup Name Movers, which act somewhat like Name Mover heads.}; in
particular, removing the causally disconnected component of
$\mathbf{v}_{\text{resid}}$ does not significantly diminish the effect of the
patch in terms of the logit difference metrics (as it does for
$\mathbf{v}_{\text{MLP}}$). 

Furthermore, we find that the cosine similarity between
$\mathbf{v}_{\text{resid}}$ and $\mathbf{v}_{\text{grad}}$ is $\approx0.78$,
which is significant (the baseline for random vectors in the residual stream is
on the order of $\frac{1}{\sqrt{768}}\approx0.03$). Both
$\mathbf{v}_{\text{resid}}$ and $\mathbf{v}_{\text{grad}}$ have a significant
fraction of their norm in the $\l(\ker W_Q^{NM}\r)^\top$ subspace (91\% and
98\%, respectively). 
These results suggest that this $\mathbf{v}_{\text{resid}}$ and
$\mathbf{v}_{\text{grad}}$ are highly similar directions, and that they're both
strongly causally connected to the model's output. 

In Figure \ref{fig:resid-projections-by-pattern}, we also find that both
directions are strongly discriminating between ABB and BAB examples.

\paragraph{Discussion.}
A key observation about the residual stream at the last token is that it is a
full bottleneck for the model's computation over the last token position: all
updates to that position are added to it. This provides another viewpoint on why
the successful patches we find don't rely on a dormant subspace: there can be no
earlier model component that activates the directions we find in a way that
skips over the patch via a residual connection (unlike for
$\mathbf{v}_{\text{MLP}}$). Indeed, in Figure \ref{fig:which-heads-to-resid} in
Appendix \ref{app:ground-truth} we show that the $\mathbf{v}_{\text{resid}}$
direction gets written to by the S-Inhibition heads.

\subsection{Additional Validations}
\label{subsection:}
In Appendix \ref{app:ground-truth}, we further validate these directions'
faithfulness to the computation of the IOI circuit from
\citet{wang2019structured} by finding the model components that write to them
and studying how they generalize on the pre-training distribution (OpenWebText);
representative samples annotated with attention scores are shown in Figures
\ref{fig:openwebtext10}, \ref{fig:openwebtext2}, \ref{fig:openwebtext8} in
Appendix \ref{app:ground-truth}.

\begin{figure}[ht]
    \centering
    \begin{minipage}[b]{0.45\textwidth}
    \begin{tikzpicture}
        \node[anchor=south west,inner sep=0] (image) at (0,0) {\includegraphics[width=1.0\textwidth]{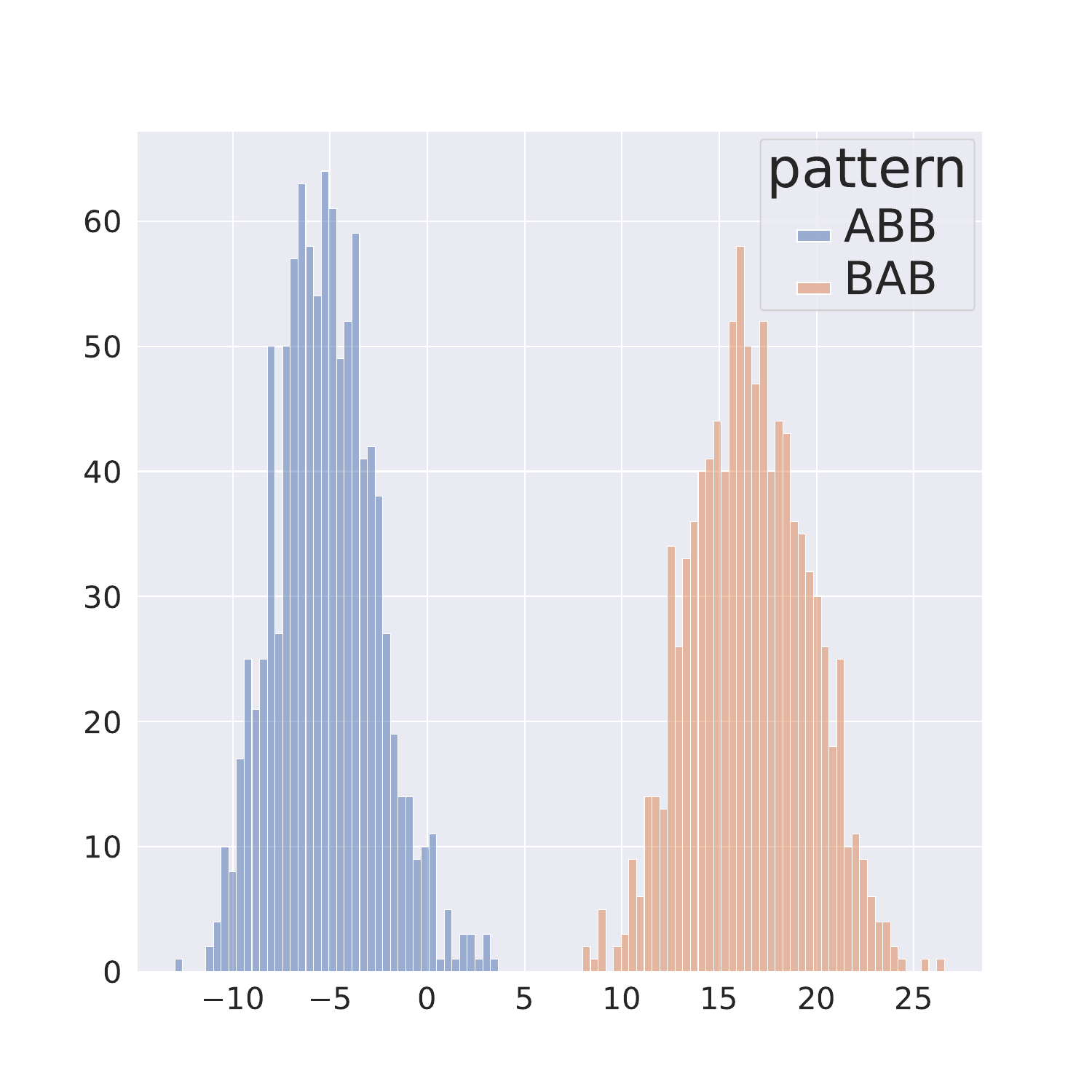}};
        \begin{scope}[x={(image.south east)},y={(image.north west)}]
            \node [anchor=east, rotate=90] at (-0.05,0.8) {Number of examples};
            \node [anchor=north] at (0.5,0.0) {Activation projection};
        \end{scope}
    \end{tikzpicture}
    \end{minipage}
    \begin{minipage}[b]{0.45\textwidth}
    \begin{tikzpicture}
        \node[anchor=south west,inner sep=0] (image) at (0,0) {\includegraphics[width=1.0\textwidth]{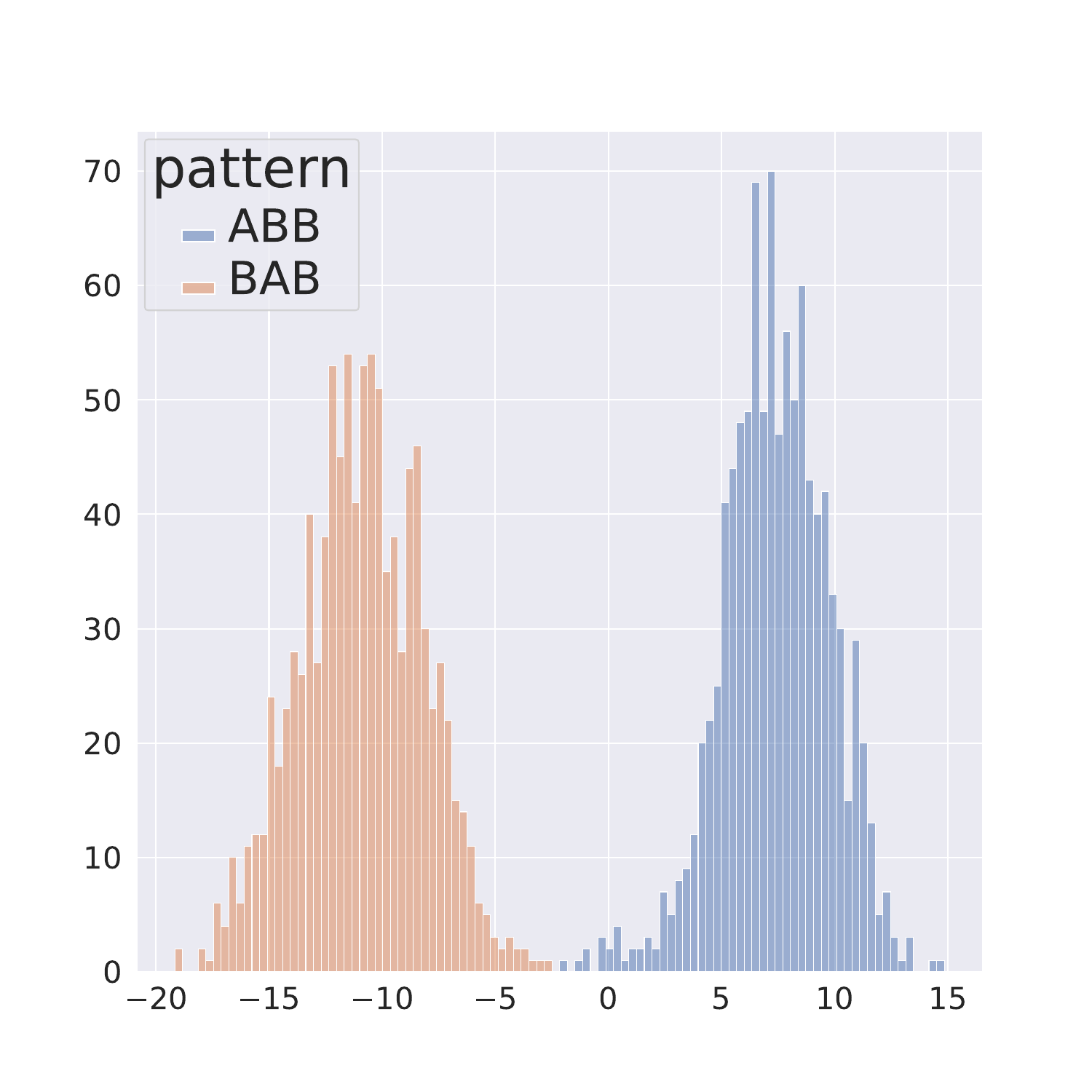}};
        \begin{scope}[x={(image.south east)},y={(image.north west)}]
            \node [anchor=north] at (0.5,0.0) {Activation projection};
        \end{scope}
    \end{tikzpicture}
    \end{minipage}
    
    \caption{Projections of dataset examples' activations in the residual stream
    after layer 8 onto the $\mathbf{v}_{\text{resid}}$ direction found by DAS
    (left) and the $\mathbf{v}_{\text{grad}}$ direction (right) which is the
    gradient for
    difference in attention of the name mover heads to the two names.}
    \label{fig:resid-projections-by-pattern}
\end{figure}

\section{Factual Recall}
\label{sec:facts}

This section has two major goals. One is to show that the interpretability illusion can
also be exhibited for the factual recall capability of language models, a much
broader setting than the IOI task. The other is to exhibit in practice an
approximate equivalence between two seemingly different interventions: rank-1
weight editing and interventions on 1-dimensional subspaces of activations. We do this in several complementary ways:
\begin{enumerate}
\item we show that DAS \citep{geiger2023finding} finds illusory 1-dimensional
subspace patches that change factual recall (e.g., to make a model complete `The
Eiffel Tower is in' with ` Rome' instead of ` Paris'). The patches found
strongly update the model's confidence towards the false completion, but the
effect disappears when the causally disconnected component of the subspace is
removed, or when the entire MLP activation containing the subspace is patched
instead. 
\item we show that for a wide range of layers in the middle of the model (GPT2-XL
\citet{radford2019language}), rank-1 fact editing using the ROME method
\citep{meng2022locating} is approximately equivalent to a 1-dimensional subspace
intervention that generalizes activation patching. The same arguments from
Sections \ref{sec:methods} and \ref{sec:discussion} apply to this intervention,
suggesting that it is likely to work successfully in a wide range of MLP layers,
regardless of the role of these MLP layers for factual recall.
\item Finally, we show that the existence of the illusory patches from 1. implies the existence of rank-1 weight edits which have identical effect at the token being
patched, and comparable overall effect on the model. This provides the other
direction of an approximate equivalence between 1-dimensional subspace
interventions and rank-1 editing, which may be of independent interest.
\end{enumerate}

In particular, these findings provide a mechanistic explanation for the
observation of prior work \citep{hase2023does} that ROME works even in layers
where the fact is supposedly not stored. As we discuss in Section
\ref{sec:discussion}, we expect that in practice all MLP layers between two
model components communicating some feature are likely to contain an illusory
subspace -- and, by virtue of the approximate equivalence we demonstrate,
rank-one fact edits will exist in these MLP layers, regardless of whether they
are responsible for recalling the fact being edited. 

\subsection{Finding Illusory 1-Dimensional Patches for Factual Recall}   
\label{sub:fact-patching}

Given a fact expressed as a subject-relation-object triple $(s, r, o)$ (e.g.,
$s=\text{`Eiffel Tower'}, r=\text{`is in'}, o=\text{`Paris'}$), we say that a
model $M$ \emph{recalls} the fact $(s, r, o)$ if $M$ completes a prompt
expressing just the subject-relation pair $(s, r)$ (e.g., `The Eiffel Tower is
in') with the object $o$ (`Paris'). 

Let us be given two facts $(s, r, o)$ and $(s', r, o')$ for the same relation that a model recalls correctly, with corresponding factual prompts $A$ expressing $(s, r)$ and $B$ expressing $(s', r)$ (e.g., $r=\text{`is in'}, A=\text{`The Eiffel Tower is in'}, B=\text{`The Colosseum is in'}$). In this subsection, we patch from $B$ into $A$, with the goal of changing the model's output from $o$ to $o'$. Implementation details are given in Appendix \ref{app:fact-patching-details}.

Results are shown in figure \ref{fig:fact-patching}.  We find a stronger version
of the same qualitative phenomena as for the IOI illusory direction from Section
\ref{sec:ioi}: (i) the directions we find have a strong causal effect
(successfully changing $o$ to $o'$), but (ii) this effect disappears when we
instead patch along the subspace spanned by the component orthogonal to $\ker W_{out}$, and (iii) patching the
entire MLP activation instead has a negligible effect on the difference in
logits between the correct and incorrect objects. Further experiments
confirming the illusion are in Appendix \ref{app:additional-fact-patching}.

We conclude that it is possible to make a model output a different object for a
given fact by exploiting a 1-dimensional subspace patch that activates a dormant
circuit in the model; in particular, using such a patch to localize the fact in
the model is prone to interpretability illusions. Next, we turn to a more
sophisticated intervetion that has been used to edit a fact in a more holistic
way, so that related facts update accordingly while the model otherwise stays
mostly the same.

\begin{figure}
    \centering
    \begin{tikzpicture}
        \node [anchor=east, rotate=90] at (-0.15,4.2) {Fraction of facts changed};
        \node [anchor=north] at (3.5,0.0) {Intervention layer};
        \node[anchor=south west,inner sep=0] (image) at (0,0) {\includegraphics[width=0.5\textwidth]{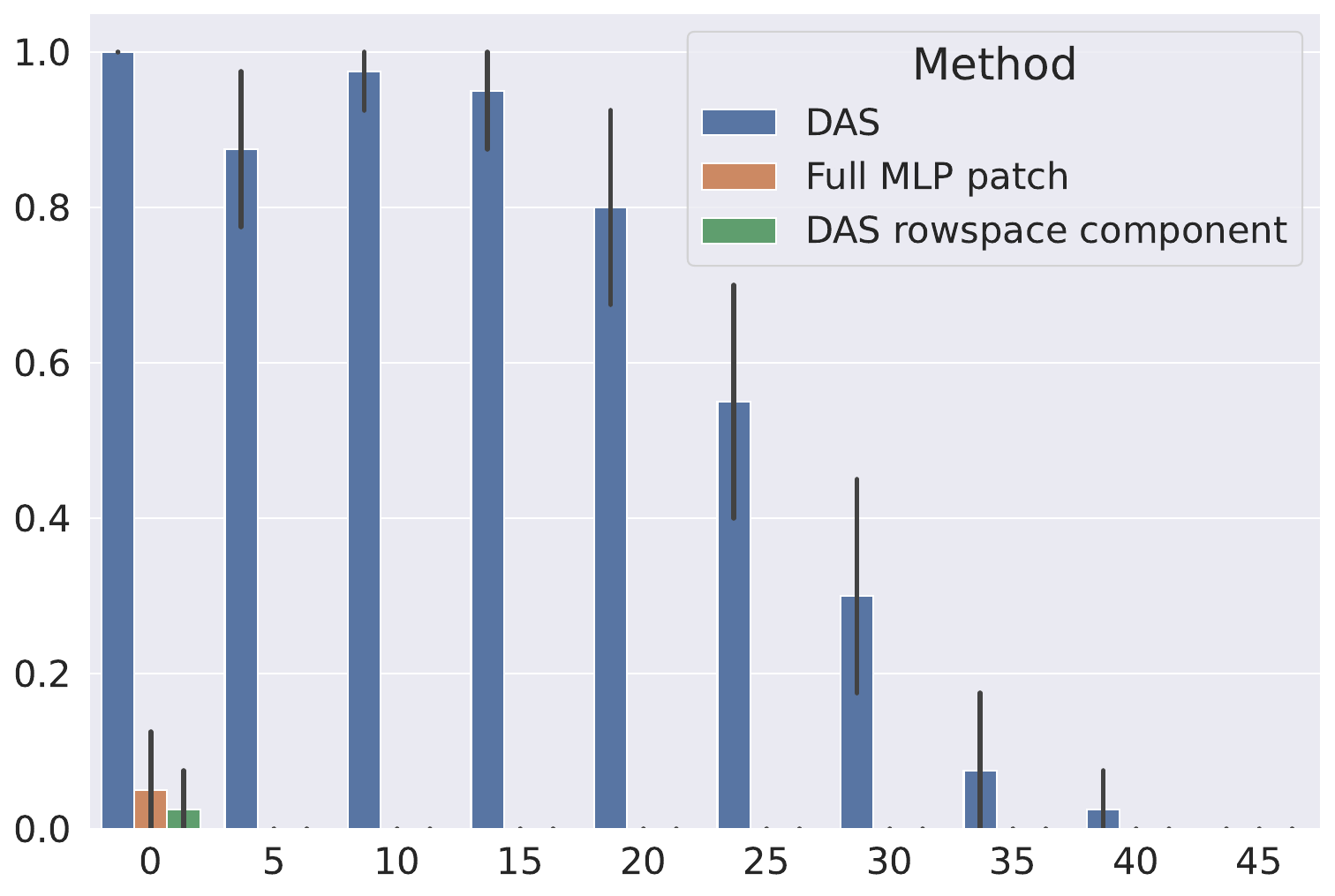}};
    \end{tikzpicture}
    \caption{Fraction of successful fact patches under three interventions: patching along the direction found by DAS (blue), patching the component of the DAS direction in the rowspace of $W_{out}$ (green), and patching the entire hidden MLP activation (orange).}
    \label{fig:fact-patching}
\end{figure}

\subsection{Background on ROME}
\label{subsection:}
\citet{meng2022locating} propose an intervention to overwrite a fact $(s, r, o)$
with another fact $(s, r, o')$ while minimally
changing the model otherwise. The intervention is a \emph{rank-1 weight edit},
which updates the down-projection $W_{out}$ of a chosen MLP layer to become
$W_{out}'=W_{out} + \mathbf{a}\mathbf{b}^\top$ for some
$\mathbf{a}\in\mathbb{R}^{d_{\text{resid}}},
\mathbf{b}\in\mathbb{R}^{d_{\text{MLP}}}$. The edit takes a `key' vector $\mathbf{k}\in\mathbb{R}^{d_{\text{MLP}}}$ representing
the subject (e.g., an average of its last-token MLP post-$\operatorname{gelu}$
activations over many prompts containing it) and a `value' vector
$\mathbf{v}\in\mathbb{R}^{d_{\text{resid}}}$ which, when output by the MLP
layer, will cause the model to predict the new object $o'$ for the factual
prompt (together with some other conditions incentivizing the model to not
change much otherwise).

Importantly, we demonstrate that ROME can be formulated as an optimization
problem with a natural objective, and this objective allows us to compare it to
related interventions. Namely, the vectors $\mathbf{a},\mathbf{b}$ are the solution to
\begin{align}
\label{eq:rome-opt-objective}
    \min_{\mathbf{a}, \mathbf{b}}
    \operatorname{trace}\l(\operatorname{Cov}_{\mathbf{x}\sim \mathcal{N}\l(0, \Sigma\r)}\left[\mathbf{a}\mathbf{b}^\top \mathbf{x}\right]\r) \quad\text{subject to}\quad W_{out}'\mathbf{k} = \mathbf{v}.
\end{align}
where $\operatorname{Cov}\left[\mathbf{r}\right] = \mathbb{E}\left[\l(\mathbf{r}
- \boldsymbol{\mu}\r)\l(\mathbf{r} - \boldsymbol{\mu}\r)^\top\right]$ denotes the covariance
matrix of a random vector $\mathbf{r}$ with mean $\boldsymbol{\mu}$, and
$\Sigma\succeq0$ is an empirical (uncentered) covariance matrix for the MLP
activations (proof in Appendix \ref{app:rome-as-opt}). In words, the ROME update
is the update that alters $W_{out}$ so it outputs $\mathbf{v}$ on input
$\mathbf{k}$, and minimizes the total variance of the extra contribution of the
update in the output of the MLP layer under the assumption that the
pre-$W_{out}$ activations are normally distributed with mean zero and covariance
$\Sigma\succeq0$.

\subsection{Rank-1 Fact Edits Imply Approximately Equivalent 1-Dimensional Subspace Interventions}
\label{sub:from-rome-to-patch}
Comparing the effect of a rank-1 edit to the MLP layer's output with equation
\ref{eq:patched-mlp-output-formula} expressing the effect of patching on the
MLP's outputs, we see that the two are quite similar.
This leads to a natural question: given a rank-1 weight edit $W_{out}' = W_{out}
+ \mathbf{ab}^\top$ such as ROME, can we find a 1-dimensional activation patch that has
the same contribution to the MLP's output for any MLP activation $\mathbf{x}$? 

\paragraph{Motivation and details.} 
As it turns out, finding a patch that has the same effect as a rank-1 edit is
not feasible in practice. For an activation $\mathbf{x}$, the extra contribution to the
MLP's output due to a rank-1 edit is $\mathbf{\left(b^\top x\right)a}$, whereas the extra
contribution of a 1-dimensional patch from activation $\mathbf{x}'$ is $\mathbf{\left(v^\top x' -
v^\top x\right)}W \mathbf{v}$, where crucially $\|\mathbf{v}\|_2=1$. In particular, the vectors $\mathbf{a},
\mathbf{b}$ are not norm-constrained, unlike $\mathbf{v}$. This restricts the magnitude of the
contribution of a patch, and we find this matters in practice. 

To overcome this, we consider a closely related subspace intervention,
\begin{align*}
    \mathbf{x}_{\text{intervention}} = \mathbf{x} + \left(\mathbf{v}^\top \mathbf{0} - \mathbf{v}^\top \mathbf{x}\right)\mathbf{v} = \mathbf{x} - \left(\mathbf{v}^\top \mathbf{x}\right)\mathbf{v}
\end{align*}
where $\mathbf{v}$ is no longer restricted to be unit norm, and $\mathbf{x}'$ is chosen to be $\mathbf{0}$
to match the expectation of the rank-1 edit's contribution (see Appendix
\ref{app:from-rank1-to-subspace} for details). This intervention bears many
similarities to subspace patching; in particular, this intervention leaves the
projections on all directions orthogonal to $\mathbf{v}$ the same, and the intuitions
about the illusion from Sections \ref{sec:methods} and \ref{sec:prevalent} also
apply to this intervention. We also remark that, at the same time, this
intervention is exactly equivalent to the rank-1 edit $W_{out}'' = W_{out} +
W_{out}\mathbf{v} \left(-\mathbf{v}\right)^\top$ in terms of contribution to the
MLP output.

It turns out that this more general intervention can often approximate the ROME
intervention well. Specifically, given a rank-1 edit $W_{out}+\mathbf{a} \mathbf{b}^\top$, we can
treat the problem probabilistically over activations $\mathbf{x}\sim \mathcal{N}\left(0,
\Sigma\right)$ like done in \citet{meng2022locating}, and ask for a direction $\mathbf{v}$
with the following properties:
\begin{itemize}
\item the expected value of both interventions is the same;
\item the resulting extra contribution $\left(-\mathbf{v}^\top\mathbf{x}\right)W_{out}\mathbf{v}$ to the MLP's output
points in the same direction as the extra contribution $\left(\mathbf{b}^\top\mathbf{x}\right)\mathbf{a}$ of the
rank-1 edit;
\item the total variance
$\operatorname{trace}\left(\operatorname{Cov}\left[\left(-\mathbf{v}^\top\mathbf{x}\right)W_{out}\mathbf{v} - \left(\mathbf{b}^\top\mathbf{x}\right)\mathbf{a}\right]\right)$ of the
difference of these two contributions is minimized over $\mathbf{x}\sim \mathcal{N}\left(0,
\Sigma\right)$.
\end{itemize}
Details are given in Appendix \ref{app:from-rank1-to-subspace}. The important
takeaway is that the solution $\mathbf{v}$ has the form
\begin{align*}
    \mathbf{v} = \alpha W_{out}^+\mathbf{a} + \mathbf{u} \quad \text{where} \quad \mathbf{u}\in \ker W_{out}
\end{align*}
for a constant $\alpha\geq0$ that is optimized. In particular, $W_{out}\mathbf{v}$ points in
the direction $\mathbf{a}$, and the component $\mathbf{u}$ (which is causally disconnected) is a
'free variable' that is essentially optimized to bring $\mathbf{v}^\top\mathbf{x}$ close to
$-\mathbf{b}^\top\mathbf{x} / \alpha$ (subject to accounting for $\Sigma$).

\paragraph{Metrics and evaluation.}
We apply this method to find subspace interventions corresponding to edits
extracted from the \textsc{CounterFact} dataset \citep{meng2022locating}; see
Appendix \ref{app:fact-patching-details} for details. Specifically, we run ROME
for all the edits, and we find the approximate subspace intervention (defined by
a vector $\mathbf{v}$) corresponding to each ROME edit. To compare the
interventions, we consider the following metrics:

\subparagraph{Rewrite score.}
Defined in \citet{hase2023does} (and a closely related metric is optimized by
\citet{meng2022locating}), the rewrite score is a relative measure of how well a
change to the model (ROME or our subspace intervention) increases the
probability of the false target $o'$ we are trying to substitute for $o$.
Specifically, if $p_{\text{clean}}\l(o\r)$ is the probability assigned by the
model to output $o$ under normal operation, and $p_{\text{intervention}}\l(o\r)$
is the probability assigned when the intervention is applied, the rewrite score
is
\begin{align*}
    \frac{p_{\text{intervention}}\l(o'\r) - p_{\text{clean}}\l(o'\r)}{1-p_{\text{clean}}\l(o'\r)}\in (-\infty, 1].
\end{align*}
with a value of $1$ indicating the model assigns probability $1$ to $o'$ after
the intervention. We measure the rewrite score for the ROME intervention, our
approximation of it, and also the corresponding subspace intervention with the
$\ker W_{out}$ component of $\mathbf{v}$ removed, in analogy with how we examine
the subspace patches in Sections \ref{sec:ioi} and \ref{sub:fact-patching}. That
is, if $\mathbf{v}_{\text{null}}$ is the orthogonal projection of $\mathbf{v}$
on $\ker W_{out}$ and $\mathbf{v}_{\text{rowsp}}=\mathbf{v}-\mathbf{v}_{\text{null}}$, we apply the intervention
\begin{align*}
    \mathbf{x}_{\text{rowspace intervention}} = \mathbf{x} - \left(\mathbf{v}_{\text{rowsp}}^\top \mathbf{x}\right)\mathbf{v}_{\text{rowsp}}.
\end{align*}
Results comparing ROME and the subspace intervention we use to approximate it
are shown in Figure \ref{fig:rank1-to-subspace-rewrite-scores}. When using the
rowspace intervention, all rewrite scores are less than $10^{-3}$, indicating a
strong reliance on the nullspace component.

\subparagraph{Cosine similarity of $\mathbf{v}$ and $\mathbf{b}$.}

Our intervention contributes $-\left(\mathbf{v}^\top \mathbf{x}\right)W\mathbf{v}$, and the ROME edit contributes
$\left(\mathbf{b}^\top \mathbf{x}\right)\mathbf{a}$. Note that, by construction, the cosine similarity of $W\mathbf{v}$
with $\mathbf{a}$ is $1$. So, the cosine similarity of $\mathbf{v}$ and $\mathbf{b}$ measures how well the
direction we are projecting the activation $\mathbf{x}$ on matches that from the ROME
edit. Results are shown in Figure \ref{fig:rank1-to-subspace-metrics} (left); in
a range of layers we observe cosine similarity significantly close to $1$.

\subparagraph{Overall change to the model relative to ROME.}
This is the total variance introduced by this intervention as a fraction of the
total variance introduced by the corresponding ROME edit. It measures the extent
to which the subspace intervention damages the model overall, following our
formulation of ROME as an optimization problem (see Appendix
\ref{app:rome-as-opt}). Results are shown in Figure
\ref{fig:rank1-to-subspace-metrics} (right). Note that this metric is a ratio of
variances; a ratio of standard deviations can be obtained by taking the square
root.

In conclusion, we observe that in layers 20-35 inclusive, the two interventions
are very similar according to all metrics considered. 

\paragraph{What is the interpretability illusion implied by this?}
An important difference between the IOI case study from Section \ref{sec:ioi}
and the factual recall results from the current section is that, while
activating a dormant circuit is contrary to activation patching's
interpretability  fact editing is,
by definition, allowed to alter the model. In this sense, activating a dormant
circuit via a rank-1 edit should no longer be considered a sign of spuriosity.

Instead, we argue that the interpretability illusion is to assume that the
success of ROME means that the fact is stored in the layer being edited. This
was already observed in \citet{hase2023does}. Our work provides a 
mechanistic explanation for this observation.

We also note that we have evaluated the success of ROME and our
approximately-equivalent subspace intervention only using the rewrite score
metric and the measure of total variance implicit in the ROME algorithm.
Ideally, there would be other validations of a fact edit that test the behavior
of the intervened-upon model on other facts that should be changed by the edit.

\begin{figure}[ht]
    \centering
    \begin{minipage}[b]{0.45\textwidth}
    \begin{tikzpicture}
        \node[anchor=south west,inner sep=0] (image) at (0,0)
        {\includegraphics[width=1.0\textwidth]{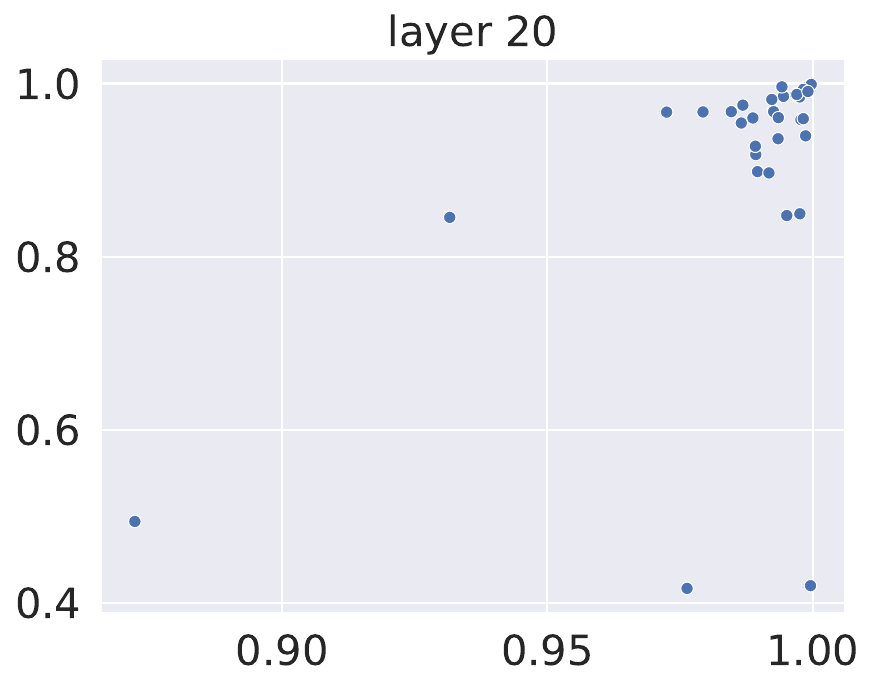}};
        \begin{scope}[x={(image.south east)},y={(image.north west)}]
            \node [anchor=east, rotate=90] at (-0.05,0.8) {Rewrite score (ours)};
        \end{scope}
    \end{tikzpicture}
    \end{minipage}
    \hfill
    \begin{minipage}[b]{0.45\textwidth}
    \begin{tikzpicture}
        \node[anchor=south west,inner sep=0] (image) at (0,0)
        {\includegraphics[width=1.0\textwidth]{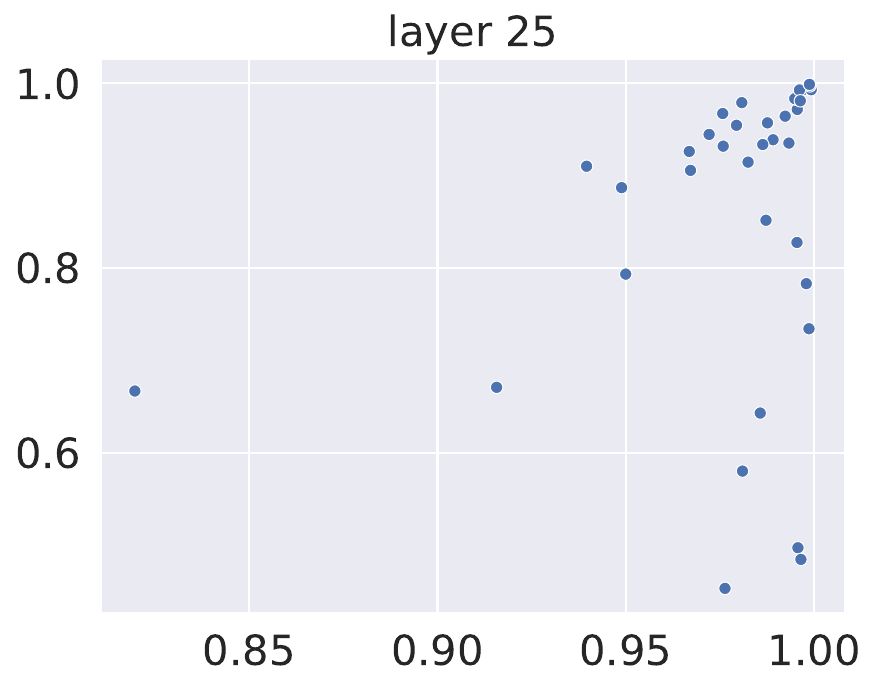}};
        \begin{scope}[x={(image.south east)},y={(image.north west)}]
        \end{scope}
    \end{tikzpicture}
    \end{minipage}

    \begin{minipage}[b]{0.45\textwidth}
    \begin{tikzpicture}
        \node[anchor=south west,inner sep=0] (image) at (0,0)
        {\includegraphics[width=1.0\textwidth]{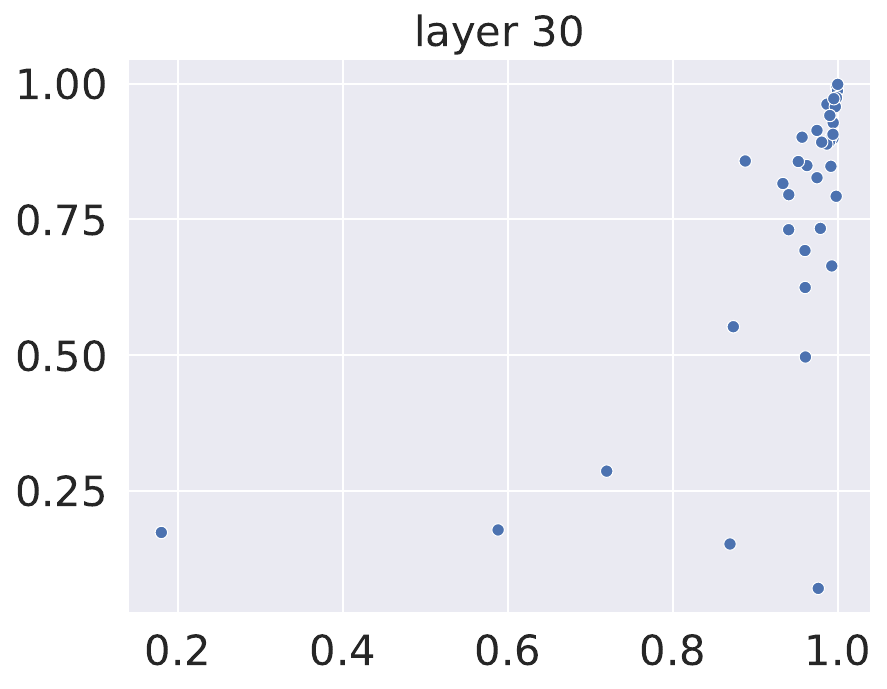}};
        \begin{scope}[x={(image.south east)},y={(image.north west)}]
            \node [anchor=east, rotate=90] at (-0.05,0.8) {Rewrite score (ours)};
            \node [anchor=north] at (0.5,0.0) {Rewrite score (ROME)};
        \end{scope}
    \end{tikzpicture}
    \end{minipage}
    \hfill
    \begin{minipage}[b]{0.45\textwidth}
    \begin{tikzpicture}
        \node[anchor=south west,inner sep=0] (image) at (0,0)
        {\includegraphics[width=1.0\textwidth]{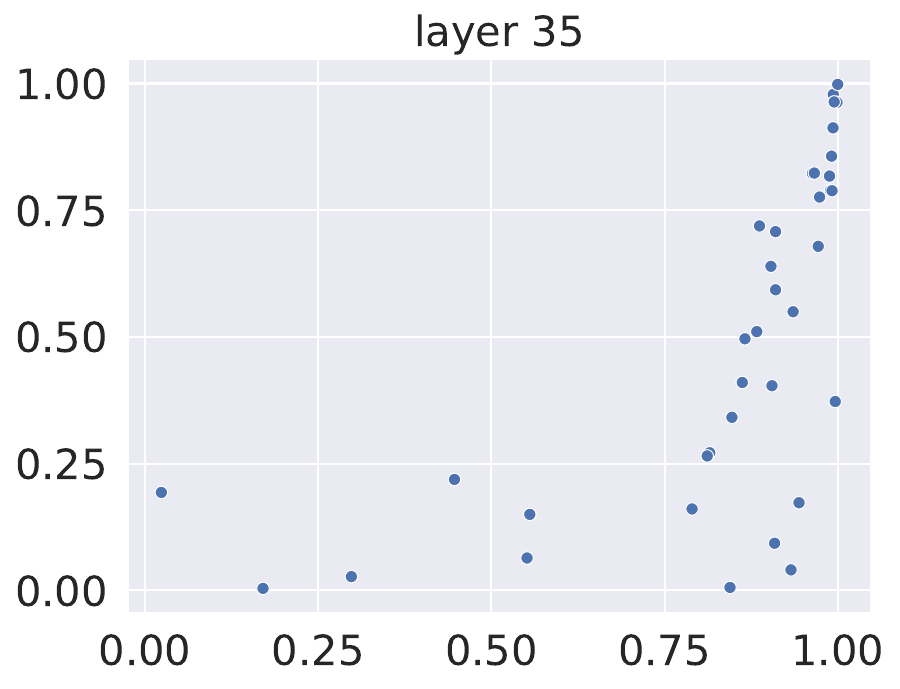}};
        \begin{scope}[x={(image.south east)},y={(image.north west)}]
            \node [anchor=north] at (0.5,0.0) {Rewrite score (ROME)};
        \end{scope}
    \end{tikzpicture}
    \end{minipage}
    \caption{Rewrite score comparison between ROME (x-axis) and our approximation to it (y-axis) via a subspace intervention for layers 20, 25, 30, 35.}
    \label{fig:rank1-to-subspace-rewrite-scores}
\end{figure}

\begin{figure}[ht]
    \centering
    \begin{minipage}[b]{0.42\textwidth}
    \begin{tikzpicture}
        \node[anchor=south west,inner sep=0] (image) at (0,0) {\includegraphics[width=1.0\textwidth]{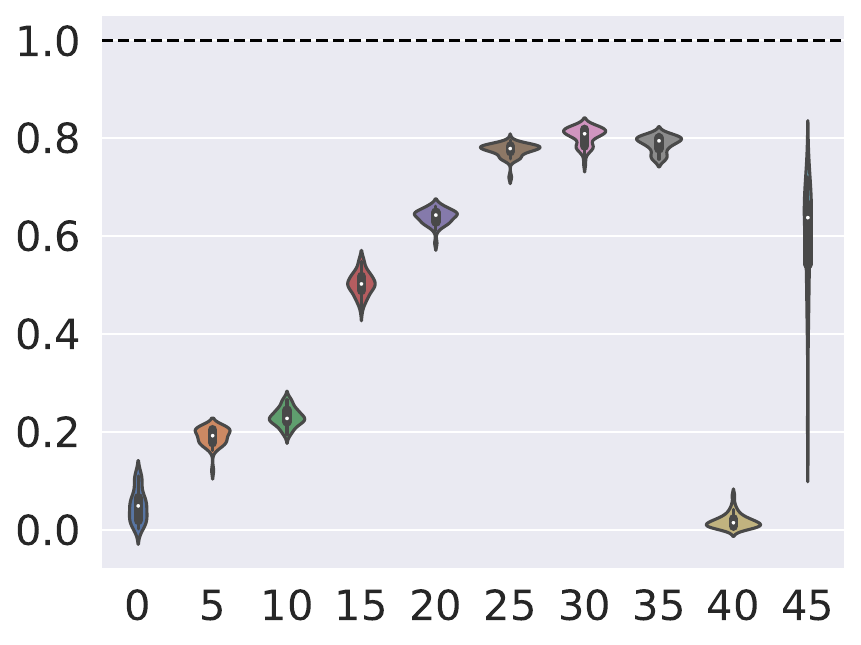}};
        \begin{scope}[x={(image.south east)},y={(image.north west)}]
            \node [anchor=east, rotate=90] at (-0.05,0.8) {Cosine similarity};
            \node [anchor=north] at (0.5,0.0) {Intervention layer};
        \end{scope}
    \end{tikzpicture}
    \end{minipage}
    \hfill
    \begin{minipage}[b]{0.42\textwidth}
    \begin{tikzpicture}
        \node[anchor=south west,inner sep=0] (image) at (0,0)
        {\includegraphics[width=1.0\textwidth]{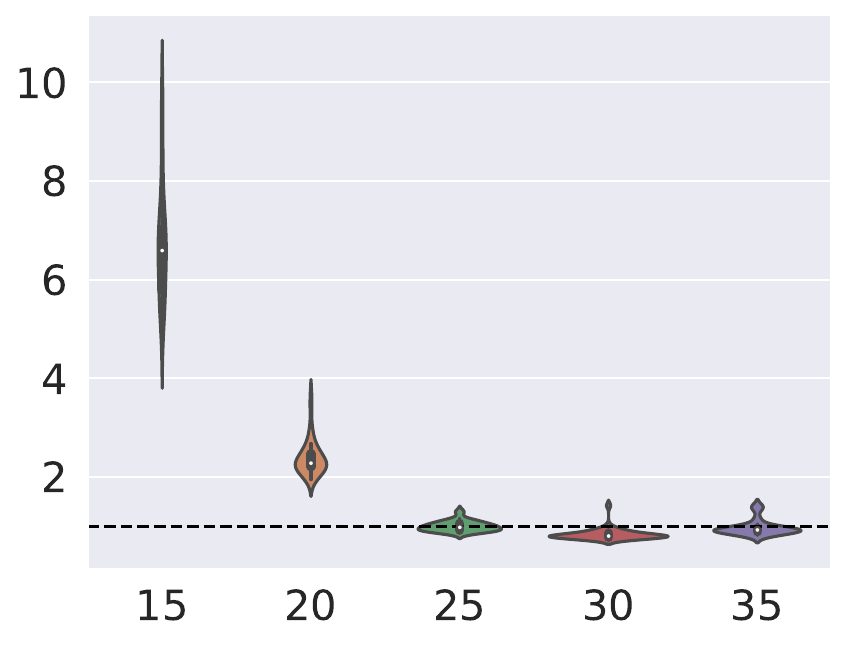}};
        \begin{scope}[x={(image.south east)},y={(image.north west)}]
            \node [anchor=east, rotate=90] at (-0.05,1.0) {Variance ratio (ours/ROME)};
            \node [anchor=north] at (0.5,0.0) {Intervention Layer};
        \end{scope}
    \end{tikzpicture}
    \end{minipage}
    
    \caption{Comparisons between ROME rank-1 edits and our approximation via a
    subspace intervention. Left: cosine similarity between the vector
    $\mathbf{v}$ defining the subspace we intervene on and the vector
    $\mathbf{b}$ from the ROME edit (dashed horizontal line is at $y=1$). Right: ratio of the total variance
    introduced by the subspace intervention to the total variance of the ROME
    intervention (x-axis scale is restricted to make the plot readable; dashed
    horizontal line is at $y=1$).}
    \label{fig:rank1-to-subspace-metrics}
\end{figure}

\subsection{1-Dimensional Fact Patches Imply Equivalent Rank-1 Fact Edits}
\label{sub:patch-implies-rome}
Finally, we show that the existence of an activation patch as in Subsection
\ref{sub:fact-patching} implies the existence of a rank-1 weight edit which has
the same contribution to the MLP's output at the token being patched, and
otherwise results in very similar model outputs as the activation patch. 

Intuitively, a `fact patch' as in Subsection \ref{sub:fact-patching} should have
a corresponding rank-1 edit with the same effect: the last subject token MLP
activation $\mathbf{u}_A$ for prompt A takes the role of $\mathbf{k}$, and the
patch modifies the MLP's output (making it $\mathbf{v}$) to change the model's
output to $o'$. We make this intuition formal in Appendix \ref{app:fact-lemma},
where we show that for each 1-dimensional activation patch between a pair of
examples in the post-$\operatorname{GELU}$ activations of an MLP layer, there is
a rank-1 model edit to $W_{out}$ that results in the same change to the MLP
layer's output at the token where we do the patching, and minimizes the variance
of the extra contribution in the sense of Equation \ref{eq:rome-opt-objective}.

While this shows that the patch implies a rank-1 edit with the same behavior
\emph{at the token where we perform the patch}, the rank-1 edit is applied \emph{permanently}
to the model, which means that it (unlike the activation patch) applies to
\emph{every} token. Thus, it is not a priori obvious whether the rank-1 edit
will still succeed in making the model predict $o'$ instead of $o$.
To this end, in Appendix \ref{app:rome-extra-exps}, we evaluate empirically how
using the rank-1 edit derived in Appendix \ref{app:fact-lemma} instead of the
activation patch changes model predictions, and we find negligible differences.

\section{Reasons to Expect the MLP-in-the-middle Illusion to be Prevalent}
\label{sec:prevalent}
We only exhibit our illusion empirically in two settings: the IOI task and
factual recall. However, we believe it is likely prevalent in practice. In this
section, we provide several theoretical, heuristic and empirical arguments in
support of this.

Specifically, we expect the illusion to be likely occur whenever we have an MLP $M$ which
is not used in the model's computation on a given task, but is between two
components $A$ and $B$ which \emph{are} used, and communicate by writing to /
reading from the direction $\mathbf{v}$ via the skip connections of the model.
This structure has been frequently observed in the mechanistic interpretability
literature \citep{lieberum2023does, wang2022interpretability, olsson2022context,
geva_key_value}: circuits contain components composing with each other separated
by multiple layers, and circuits have often been observed to be sparse, with
most components (including most MLP layers) not playing a significant role.

\subsection{Assumptions: a Simple Model of Linear Features in the Residual
Stream}
\label{sub:prevalent-assumptions}

The linear representation hypothesis suggests a natural way to formalize this
intuition. Namely, let's assume for simplicity that there is a binary feature
$C$ in the data, and the value of $C$ influences the model's behavior on a
task, by e.g. making some next-token predictions more likely than others. Concretely, there is
a residual stream direction $\mathbf{v}\in \mathbb{R}^{d_{resid}}$ that mediates
this effect: projections on $\mathbf{v}$ (at an appropriate token position)
linearly separate examples according to the value of $C$, and intervening on
this projection by setting it to e.g. the mode of the opposite class has the
same effect on model outputs as changing the value of $C$ in the input itself.
Furthermore, we assume that this direction has this property in all residual
stream layers between some two layers $a < b$. 

We note that these assumptions can be realized `in the wild': the highly similar directions $\mathbf{v}_{\text{grad}}$
and $\mathbf{v}_{\text{resid}}$ discussed in Section \ref{sec:ground-truth} are
both examples of such directions $\mathbf{v}$ for the binary concept of whether the
\textbf{IO} name comes first or second in the sentence, as we argued empirically.

\subsection{Overview of Argument}
\label{subsection:}

The key hypothesis is that, given the setup from the previous Subsection
\ref{sub:prevalent-assumptions}, the post-nonlinearity activations of every MLP
layer between layers $a$ and $b$ are likely to contain a 1-dimensional subspace
whose patching will have the same effect (possibly with a smaller magnitude) on
model outputs as changing the value of $C$ in the input. For this, it is
sufficient to have two kinds of directions in the MLP's activation space:
\begin{itemize}
\item a `causal' direction, such that changing the projection of an activation
along this direction results in the
expected change of model outputs. Such a direction will exist simply because
$W_{out}$ is a full-rank matrix, so we can simply pick $W_{out}^+ \mathbf{v}$.
We give an empirically-supported theoretical argument for this in Appendix \ref{subsection:dormant-evidence}.
\item a `correlated' direction that linearly discriminates between the values of
$C$: such a direction will exist because the pre-nonlinearity activations (which
are an approximately linear image of the residual stream) will linearly
discriminate $C$, and the transformation $\mathbf{x}\mapsto
\operatorname{gelu}(\mathbf{x}) \mapsto \operatorname{proj}_{\ker W_{out}}
\operatorname{gelu}\l(\mathbf{x}\r)$ approximately preserves linear
separability. We give an empirically-supported theoretical argument for this in
Appendix \ref{subsection:disconnected-evidence}.
\end{itemize}

\section{Discussion, Limitations, and Recommendations}
\label{sec:discussion}

Throughout this paper, we have seen that interventions on arbitrary linear
subspaces of model activations, such as subspace activation patching, can have
counterintuitive properties. In this section, we take a step back and provide a
more conceptual point of view on these findings, as well as concrete advice for
interpretability practitioners.

\paragraph{Why should this phenomenon be considered an illusion?}
One argument for the illusory nature of the subspaces we find is the reliance on
a large causally disconnected component (in all our examples, this component is
in the kernel of the down-projection $W_{out}$ of an MLP layer). In particular,
patching along only the causally-relevant component of the subspace (the one in
$\l(\ker W_{out}\r)^\perp$) destroys the effect of the subspace patch;
we find this a convincing reason to be suspicious of the explanatory
faithfulness of these subspaces. 

Beyond this argument, there are several more subtle considerations. For an
explanation to be `illusory', there has to be some notion of what the `true'
explanation is.  We admit that a definition of a `ground truth' mechanistic
explanation is conceptually challenging. In the absence of such a definition,
our claims rest on various observations about model's inner workings that we now
collect in one place and make more explicit. We believe these findings
collectively point to meaningful constraints on mechanistic explanations.

For example, the IOI circuit work of \citet{wang2022interpretability} finds
through various component-level interventions that the layer 8 MLP 
\emph{as a whole} does not contribute significantly to the model's ability to do
the IOI task. However, does this imply that there aren't individual subspaces of
the MLP layer that mediate the model's behavior on the IOI task? Not
necessarily: there could be, for example, two subspaces mediating the position
signal, but which have opposite contributions to the MLP's output that cancel
out. This is compatible with our model of the illusion from Section
\ref{sec:methods}: for example, we can form two `cancelling' 1-dimensional
subspaces by taking the sum and difference of the causally disconnected and
dormant directions in our model. From this point of view, our subspace
intervention decouples these (ordinarily coupled) subspaces by changing the
activation only along one of them. This is impossible for an intervention that
operates on entire model components.

Should we prefer the view under which the MLP layer simply does not participate
in any meaningful way in the IOI task, or the view under which it contains
subspaces that mediate information about the IOI task, but whose contributions
cancel out? Note that meaningful cancellation behavior between entire model
components has been observed to some extent in the mechanistic interpretability
literature, such as with negative heads \citep{wang2022interpretability},
anti-induction heads \citep{olsson2022context} and copy suppression heads
\citep{mcdougall2023copy}. Furthermore, it is not clear that, in general, a
component-level explanation should take precendence over subspace-level
explanations. So, a priori, we have a conundrum: two different kinds of
interventions arrive at conflicting interpretations.

Nevertheless, based on our experiments, we suggest that the view under which the
MLP layer contains meaningful subspaces that cancel out is the less likely
mechanistic explanation for several reasons. A first argument is that, as we
argue in Section \ref{sec:prevalent}, the existence of the illusory subspace
only relies on the existence of certain directions in the residual stream; the
MLP weights themselves don't play a role. In some sense, the illusory subspace
is a `necessity of linear algebra'. This is further reinforced by the fact that
we find illusory directions even when the MLP weights are replaced by random
matrices (see Appendix \ref{app:random-mlp}). A second argument is that features
that are individually strong, but whose contributions almost exactly cancel out,
seem unlikely to be prevalent. 

Finally, we again remark that circuits for specific tasks have been observed to
be sparse (recall Section \ref{sec:prevalent}). Our model of the illusion from
Section \ref{sec:methods} and the evidence from Section \ref{sec:prevalent}
suggest that any MLP layer between two circuit components using the residual
stream as a communication bottleneck for some feature will contain a subspace
that appears to mediate this feature. Thus, even if we cannot conclusively rule
out any given MLP layer on the path as not being meaningfully involved in the
computation, it would be quite surprising if always all of them are involved. So
we expect that at least some of these subspaces will be illusory.

\paragraph{The importance of correct model units.} A further implicit assumption
in our work is that model components are meaningful boundaries for mechanistic
explanation. As we illustrate in Appendix \ref{app:rotated-basis}, our toy
example of the illusion can be considered in a rotated basis, in which the
`illusory' direction appears `meaningful'. In a similar way, if we allow
ourselves to arbitrarily reparametrize spaces of activations by crossing the
boundaries between e.g. attention heads and MLP layers, calling the MLP subspace
`illusory' is much more tenuous. 

To respond to this criticism, we point to the many observations in the
mechanistic literature that different components (like heads and MLP layers)
perform qualitatively different functions in a model. For example, tasks
involving algorithmic manipulations of in-context text, such as the IOI task,
often rely predominantly on attention heads
\citep{wang2022interpretability,docstring}. On the other hand, MLP layers have
so far been implicated in tasks having to do with recalling bigrams and facts \citep{meng2022locating,gurnee2023finding}. On these grounds, mixing
activations between them is likely to lead to less parsimonious and less principled mechanistic explanations. 

\paragraph{Takeaways and recommendations.} As we have seen, optimization-based
methods using subspace activation patching, such as DAS
\citep{geiger2023finding} can find both faithful (Section
\ref{sec:ground-truth}) and illusory (Section \ref{sec:ioi}) features with
respect to the model's computation. We recommend running such methods in
activation bottlenecks, especially the residual stream, as well as using
validations beyond end-to-end evaluations to ascertain the precise role of such
features. 

\section{Acknowledgments}
\label{section:}

We are deeply indebted to Atticus Geiger for many useful discussions and helpful feedback, as well as help writing parts of the paper. We particularly appreciate his thoughtful pushback on framing and narrative, and commitment to rigour, without which this manuscript would be far poorer. We would also like to thank Christopher
Potts, Curt Tigges, Oskar Hollingsworth,  Tom Lieberum, Senthooran Rajamanoharan and Peli Grietzer for valuable
feedback and discussions. The authors used the open-source library
\verb|transformerlens| \citep{nanda2022transformerlens} to carry out the
experiments related to the IOI task. AM and GL did this work as part of the SERI
MATS independent research program, with support from AI Safety Support.

\section{Author Contributions}
\label{section:}

NN was the main supervisor for this project, and provided high level feedback on
experiments, prioritisation, and writing throughout. NN came up with the
original idea of the illusion and the conceptual example. AM came up with the
correspondence with factual recall, developed the factual recall results, and
ran the experiments for Sections \ref{sec:facts}, \ref{sec:prevalent} and part
of \ref{sec:ground-truth} (with the exception of experiments from Appendix
\ref{app:linearly-separable} ran by GL), and wrote the majority of the paper and
appendices, as well as the public version of the code for the paper. GL also ran the experiments for Sections \ref{sec:ioi} and
\ref{sec:ground-truth} with some guidance from AM, wrote and ran the experiments
for Appendices
\ref{app:random-mlp}, \ref{app:generalization-high-dim}, \ref{app:overfitting},
\ref{app:ground-truth} and \ref{app:linearly-separable}, and contributed to
Section \ref{sec:ground-truth} and various other sections. 
\bibliography{refs}
\bibliographystyle{style}

\appendix

\section{Additional Notes on Section \ref{sec:methods}}
\label{section:}

\begin{figure}
\begin{center}
\begin{minipage}[b]{.49\textwidth}
\centering
\resizebox{\textwidth}{!}{
\begin{tikzpicture}
    \draw[line width=2pt,-{Stealth[length=8pt, scale=2]}] (-8,0)--(8,0)
    node[below left,font=\LARGE, xshift=0.7cm]{causally disconnected direction};
    \draw[line width=2pt,-{Stealth[length=8pt, scale=2]}] (0,-6)--(0,6) node[above left,rotate=90,font=\LARGE]{dormant direction};
\end{tikzpicture}
}
\captionof{figure}{Consider a 2-dimensional subspace of model activations, with
an orthogonal basis where the $x$-axis is \emph{causally disconnected} 
(changing the activation along it makes no difference to model outputs) and
values on the $y$ axis are always zero for examples in the data distribution (a
specical case of a \emph{dormant} direction).}
\end{minipage}%
\hfill
\begin{minipage}[b]{.49\textwidth}
\centering
\resizebox{\textwidth}{!}{
\begin{tikzpicture}
    \draw[line width=2pt, myred, -{Stealth[length=8pt, scale=2]}] (-135:7)--(45:7)
    node[above left, font=\LARGE, rotate=45, xshift=1.0cm]{direction to patch along ($\mathbf{v}$)};
    
    \draw[line width=2pt,-{Stealth[length=8pt, scale=2]}] (-8,0)--(8,0)
    node[below left,font=\LARGE, xshift=0.7cm]{causally disconnected direction};
    \draw[line width=2pt,-{Stealth[length=8pt, scale=2]}] (0,-6)--(0,6)
    node[above left,rotate=90,font=\LARGE]{dormant direction};
    \draw[-{Stealth[length=8pt, scale=2]}, mygreen, line width=2pt] (0,0)--(4.5,0)
    node[above,font=\LARGE]{example to patch into};
    \draw[-{Stealth[length=8pt, scale=2]}, mygreen, line width=2pt] (0,0)--(-4.5,0)
    node[above,font=\LARGE]{example to patch from};

\end{tikzpicture}
}
\captionof{figure}{Suppose we have two examples (green) which differ in their projection on the causally disconnected direction (and have zero projection on the dormant direction, by definition). Let's consider what happens when we patch from the example on the left into the example on the right along the 1-dimensional subspace $\mathbf{v}$ spanned by the vector $(1, 1)$ (red)}
\end{minipage}

\medskip

\begin{minipage}[b]{.49\textwidth}
\centering

\resizebox{\textwidth}{!}{
\begin{tikzpicture}
    \draw[line width=2pt, myred, -{Stealth[length=8pt, scale=2]}] (-135:7)--(45:7)
    node[above left, font=\LARGE, rotate=45, xshift=1.0cm]{direction to patch along ($\mathbf{v}$)};
    \draw[line width=2pt, myred, -{Stealth[length=8pt, scale=2]}] (135:7)--(-45:7)
    node[pos=0, below right,font=\LARGE, rotate=-45, xshift=-1.0cm]{orthogonal complement
    $\mathbf{v}^\perp$};
    
    \draw[dashed, line width=2pt, myblue] (-4.5,0)--(0,-4.5);
    \fill[myblue] (-1.8,-1.2) circle (0.3);
    \node at (-1.8,-1.2) [font={\large\bfseries}, text=white] {1};

    \fill[myblue] (1.8,-1.2) circle (0.3);
    \node at (1.8,-1.2) [font={\large\bfseries}, text=white] {2};

    \draw[dashed, -{Stealth[length=8pt, scale=2]}, line width=2pt, myblue]
    (0,0)--(-2.25,-2.25);
    \draw[dashed, line width=2pt, myblue] (4.5,0)--(0,-4.5);
    \draw[dashed, -{Stealth[length=8pt, scale=2]}, line width=2pt, myblue]
    (0,0)--(2.25,-2.25);

    \draw[line width=2pt,-{Stealth[length=8pt, scale=2]}] (-8,0)--(8,0)
    node[below left,font=\LARGE, xshift=0.7cm]{causally disconnected direction};
    \draw[line width=2pt,-{Stealth[length=8pt, scale=2]}] (0,-6)--(0,6) node[above left,rotate=90,font=\LARGE]{dormant direction};
    \draw[-{Stealth[length=8pt, scale=2]}, mygreen, line width=2pt] (0,0)--(4.5,0)
    node[above,font=\LARGE]{example to patch into};

    \draw[-{Stealth[length=8pt, scale=2]}, mygreen, line width=2pt] (0,0)--(-4.5,0)
    node[above,font=\LARGE]{example to patch from};

\end{tikzpicture}
}
\captionof{figure}{To patch along $\mathbf{v}$ from the left into the right example, we
match the projection on $\mathbf{v}$ from the left one, and leave the projection
on $\mathbf{v}^\perp$
unchanged. In other words, we take the component of the left example along $\mathbf{v}$
(\raisebox{.5pt}{\textcircled{\raisebox{-.9pt} {1}}}) and sum it with the
$\mathbf{v}^\perp$ component (\raisebox{.5pt}{\textcircled{\raisebox{-.9pt} {2}}}) of the original activation.}
\end{minipage}%
\hfill
\begin{minipage}[b]{.49\textwidth}
\centering

\resizebox{\textwidth}{!}{
\begin{tikzpicture}
    \draw[line width=2pt, myred, -{Stealth[length=8pt, scale=2]}] (-135:7)--(45:7)
    node[above left, font=\LARGE, rotate=45, xshift=1.0cm]{direction to patch along ($\mathbf{v}$)};
    \draw[line width=2pt, myred, -{Stealth[length=8pt, scale=2]}] (135:7)--(-45:7)
    node[pos=0, below right,font=\LARGE, rotate=-45, xshift=-1.0cm]{orthogonal complement
    $\mathbf{v}^\perp$};
    
    \draw[dashed, line width=2pt, myblue] (-4.5,0)--(0,-4.5);
    \fill[myblue] (-1.8,-1.2) circle (0.3);
    \node at (-1.8,-1.2) [font={\large\bfseries}, text=white] {1};

    \fill[myblue] (1.8,-1.2) circle (0.3);
    \node at (1.8,-1.2) [font={\large\bfseries}, text=white] {2};

    \draw[dashed, -{Stealth[length=8pt, scale=2]}, line width=2pt, myblue]
    (0,0)--(-2.25,-2.25);
    \draw[dashed, line width=2pt, myblue] (4.5,0)--(0,-4.5);
    \draw[dashed, -{Stealth[length=8pt, scale=2]}, line width=2pt, myblue]
    (0,0)--(2.25,-2.25);

    \draw[line width=2pt,-{Stealth[length=8pt, scale=2]}] (-8,0)--(8,0)
    node[below left,font=\LARGE, xshift=0.7cm]{causally disconnected direction};
    \draw[line width=2pt,-{Stealth[length=8pt, scale=2]}] (0,-6)--(0,6) node[above left,rotate=90,font=\LARGE]{dormant direction};
    \draw[-{Stealth[length=8pt, scale=2]}, mygreen, line width=2pt] (0,0)--(4.5,0)
    node[above,font=\LARGE]{example to patch into};
    \draw[-{Stealth[length=8pt, scale=2]}, mygreen, line width=2pt] (0,0)--(-4.5,0)
    node[above,font=\LARGE]{example to patch from};
    \draw[-{Stealth[length=8pt, scale=2]}, myblue, line width=2pt] (0,0)--(0,-4.5)
    node[above,font=\LARGE, rotate=-90]{result of the patch};

    \fill[myblue] (-1.8,-4.5) circle (0.3);
    \node at (-1.8,-4.5) [font={\large\bfseries}, text=white] {1};
    \node at (-1.2,-4.5) [font={\large\bfseries}, text=black] {+};
    \fill[myblue] (-0.6,-4.5) circle (0.3);
    \node at (-0.6,-4.5) [font={\large\bfseries}, text=white] {2};

\end{tikzpicture}
}
\captionof{figure}{This results in the patched activation 
  \raisebox{.5pt}{\textcircled{\raisebox{-.9pt} {1}}}
+\raisebox{.5pt}{\textcircled{\raisebox{-.9pt} {2}}}, which points completely
along the dormant direction. In this way, activation patching makes the
variation of activations along the causally disconnected $x$-axis result in
activations along the previously dormant $y$-axis.}
\end{minipage}
\end{center}
\caption{A step-by-step illustration of the phenomenon shown in Figure
\ref{fig:illusion}.}
\label{fig:illusion-step-by-step}
\end{figure}

\subsection{The Illusion for Higher-Dimensional Subspaces}
\label{app:higher-dim}

In the main text, we mostly discuss the illusion for activation patching of 1-dimensional subspaces for ease of exposition. Here, we develop a more complete picture of the mechanics of the illusion for higher-dimensional subspaces.

Let $\mathcal{C}$ be a model component taking values in $\mathbb{R}^d$, and let $U\subset \mathbb{R}^d$ be a linear subspace.
Let $V$ be a matrix whose columns form an orthonormal basis for $U$.
If the $\mathcal{C}$ activations for examples $A$ and $B$ are $\mathbf{act}_A, \mathbf{act}_B\in\mathbb{R}^d$ respectively, patching $U$ from $A$ into $B$ gives the patched activation
\begin{align*}
    \mathbf{act}_B^{patched} = \mathbf{act}_B + VV^\top (\mathbf{act}_A - \mathbf{act}_B) = (I - VV^\top)\mathbf{act}_B + VV^\top \mathbf{act}_A
\end{align*}
For intuition,  note that $VV^\top$ is the orthogonal projection on $U$, so this formula says to replace the orthogonal projection of $\mathbf{act}_B$ on $U$ with that of $\mathbf{act}_A$, and keep the rest of $\mathbf{act}_B$ the same.

Generalizing the discussion from Section \ref{sec:methods}, for the illusion to
occur for subspace $S$, we need $S$ to be sufficiently aligned with a causally
disconnected subspace $V_{\text{disconnected}}$ that is correlated with the
feature being patched, and a dormant but causal subspace $V_{\text{dormant}}$
which, when set to out of distribution values, can achieve the wanted causal
effect.

For example, a particularly simple way in which this could happen is if we let
$V_{\text{disconnected}}, V_{\text{dormant}}$ be 1-dimensional subspaces (like
in the setup for the 1-dimensional illusion), and we form $S$ by combining
$V_{\text{disconnected}} \oplus V_{\text{dormant}}$ with a number of orthogonal
directions that are approximately constant on the data with respect to the
feature we are patching.  These extra directions effectively don't matter for
the patch (because they are cancelled by the $\mathbf{act}_A - \mathbf{act}_B$
term).  Given a specific feature, it is likely that such weakly-activating
directions will exist in a high-dimensional activation space.  Thus, if the
1-dimensional illusion exist, so will higher-dimensional ones.

\subsection{Optimal Illusory Patches are Equal Parts Causally Disconnected and Dormant}
\label{app:equal-norms}
In this subsection, we prove a quantitative corollary of the model of our
illusion that suggests that we should expect optimal illusory patching
directions to be of the form
\begin{align*}
    \mathbf{v}_{\text{illusory}}=\frac{1}{\sqrt{2}}\left(\mathbf{v}_{\text{disconnected}}
+ \mathbf{v}_{\text{dormant}}\right)
\end{align*}
for unit vectors $\mathbf{v}_{\text{disconnected}}\perp
\mathbf{v}_{\text{dormant}}$. In other words, we expect the strongest illusory
patches to be formed by combining a disconnected and illusory direction with
\emph{equal} coefficients, like depicted in Figure \ref{fig:illusion}:

\begin{lemma}
   Suppose we have two distributions of input prompts $\mathcal{D}_{A},
   \mathcal{D}_{B}$. In the terminology of Section \ref{sec:methods}, let
   $\mathbf{v}_{\text{disconnected}}\perp \mathbf{v}_{\text{dormant}}$ be unit
   vectors such that the subspace spanned by $\mathbf{v}_{\text{disconnected}}$
   is a causally disconnected subspace, and the subspace spanned by
   $\mathbf{v}_{\text{dormant}}$ is \textbf{strongly} dormant, in the sense that
   the projections of the activations of all examples $\mathcal{D}_{source}\cup
   \mathcal{D}_{base}$ onto $\mathbf{v}_{\text{dormant}}$ are equal to some
   constant $c$.
   
   Let $\mathbf{v} = \mathbf{v}_{\text{disconnected}}\cos\alpha  +
   \mathbf{v}_{\text{dormant}}\sin\alpha$ be a unit-norm linear combination of
   the two directions parametrized by an angle $\alpha$. Then the magnitude of
   the expected change in projection along $\mathbf{v}_{\text{dormant}}$ when
   patching from $x_A\sim\mathcal{D}_{A}$ into $x_B\sim\mathcal{D}_{B}$ is
   maximized when $\alpha=\frac{\pi}{4}$, i.e.
   $\cos\alpha=\sin\alpha=\frac{1}{\sqrt{2}}$.
\end{lemma}
\begin{proof}
Recall that the patched activation from $x_A$ into $x_B$ along $\mathbf{v}$ is
\begin{align*}
    \mathbf{act}_B^{\text{patched}} = \mathbf{act}_B + (p_A-p_B)\mathbf{v}
\end{align*}
where $p_A = \mathbf{v}^\top \mathbf{act}_A, p_B = \mathbf{v}^\top \mathbf{act}_B$ are the projections of the two examples' activations on $v$. The change along $\mathbf{v}_{\text{dormant}}$ is thus
\begin{align*}
    &\mathbf{v}_{\text{dormant}}^\top\left(\mathbf{act}_B^{\text{patched}} - \mathbf{act}_B\right) = (p_A-p_B)\sin\alpha = (\mathbf{v}^\top \mathbf{act}_A - \mathbf{v}^\top \mathbf{act}_B)\sin\alpha
    \\
    &= \mathbf{v}_{\text{disconnected}}^\top (\mathbf{act}_A - \mathbf{act}_B)\cos\alpha \sin\alpha
\end{align*}
where we used the assumption that $\mathbf{v}_{\text{dormant}}^\top \mathbf{act}_A = \mathbf{v}_{\text{disconnected}}^\top\mathbf{act}_B$. Hence, the expected change is
\begin{align*}
    \cos\alpha\sin\alpha\, \mathbf{v}_{\text{disconnected}}^\top\mathbb{E}_{A\sim\mathcal{D}_A, B\sim\mathcal{D}_B}\left[\mathbf{act}_A - \mathbf{act}_B\right].
\end{align*}
The function $f(\alpha) = \cos\alpha \sin\alpha$ for $\alpha\in[0, \pi/2]$ is maximized for $\alpha=\pi/4$, concluding the proof.
\end{proof}

\subsection{The Toy Illusion in a Rotated Basis}
\label{app:rotated-basis}

There is a subtlety in the toy example of the illusion from
\ref{sub:toy-illusion}. Suppose we reparametrized the hidden layer of the
network so that, instead of the standard basis $(\mathbf{e}_1, \mathbf{e}_2,
\mathbf{e}_3)$, we use a rotated basis where one of the directions is
$\mathbf{e}_1+\mathbf{e}_2$, the other direction is orthogonal to it and to
$\mathbf{w}_2$ (hence will be causally disconnected), and the last direction is
the unique direction orthogonal to the first two. 

The unit basis vectors for this new basis are given by 
\begin{align*}
    \mathbf{d}_1 &= \l(\mathbf{e}_1 + \mathbf{e}_2\r) / \sqrt{2}, \\
    \mathbf{d}_2 &=\l(-\mathbf{e}_1 + \mathbf{e}_2 - 2 \mathbf{e}_3\r) / \sqrt{6}, \\
    \mathbf{d}_3 &= \l(-\mathbf{e}_1 + \mathbf{e}_2 + \mathbf{e}_3\r) / \sqrt{3}.
\end{align*}
If we assemble these into the rows of a rotation matrix 
\begin{align*}
R = \begin{bmatrix}
\frac{1}{\sqrt{2}} & \frac{1}{\sqrt{2}} & 0 \\
\frac{-1}{\sqrt{6}} & \frac{1}{\sqrt{6}} & \frac{-2}{\sqrt{6}} \\
\frac{-1}{\sqrt{3}} & \frac{1}{\sqrt{3}} & \frac{1}{\sqrt{3}}
\end{bmatrix}
\end{align*}
the re-parametrized network is then given by 
\begin{align*}
    x \quad\mapsto\quad \mathbf{h}' = R \mathbf{w}_1 x \quad\mapsto\quad y = \l(R \mathbf{w}_2\r)^\top \mathbf{h}'
\end{align*}
and a diagram of this new network is shown in Figure
\ref{fig:mathematical-example-rotated-basis}. From this point of view,
$\mathbf{d}_1$ takes the role that $\mathbf{e}_3$ had before: the input is
essentially copied to it (modulo scalar multiplication), and then read from it
at the output. By contrast, $\mathbf{d}_2$ is now a causally disconnected
direction, and $\mathbf{d}_3$ is a dormant direction.

\begin{figure}
  \begin{center}
\resizebox{0.5\textwidth}{!}{
\begin{tikzpicture}[state/.style={circle, draw, 
  minimum size=2.5cm, font=\LARGE}]
\node[state] (X) {$x$};
\node[state] (H1) at (7,5) {$h_1' \gets x/\sqrt{2}$};
\node[state] (H2) at (7,0) {$h_2'\gets -\sqrt{3/2}x$};
\node[state] (H3) at (7,-5) {$h_3' \gets 0$};
\node[state, right of= H2] at (14,0) (Y) {$y \gets \sqrt{2}h_1' + \sqrt{3}h_3'$};

\draw[-{Stealth[scale=2]}] (X) -> (H1) node[midway, above left, font=\LARGE]
{$\times 1/\sqrt{2}$};
\draw[-{Stealth[scale=2]}] (X) -> (H2) node[midway, above, font=\LARGE]
{$\times \l(-\sqrt{3/2}\r)$};
\draw[-{Stealth[scale=2]}] (H1) -> (Y) node[midway, above, font=\LARGE] {$\times
\sqrt{2}$};
\draw[-{Stealth[scale=2]}] (H3) -> (Y) node[midway, below, font=\LARGE] {$\times
\sqrt{3}$};
\end{tikzpicture}
}
  \end{center}
\caption{}
\label{fig:mathematical-example-rotated-basis}
\end{figure}

\section{Additional details for Section \ref{sec:ioi}}
\label{section:}

\subsection{Dataset, Model and Evaluation Details for the IOI Task}
\label{app:ioi-dataset-details}
We use GPT2-Small for the IOI task, with a dataset that spans 216 single-token names, 144 single-token objects and 75 single-token places, which are split $1:1$ across a training and test set. 
Every example in the data distribution includes (i) an initial clause introducing the indirect object (\textbf{IO}, here `Mary') and the subject (\textbf{S}, here `John'),
and (ii) a main clause that refers to the subject a second time.
Beyond that, the dataset varies in the two names, the initial clause content, and the main clause content.
Specifically, use three templates as shown below:
\begin{center}
    \text{Then, [ ] and [ ] had a long and really crazy argument. Afterwards, [ ] said to}
    \\
    \text{Then, [ ] and [ ] had lots of fun at the [place]. Afterwards, [ ] gave a [object] to}
    \\
    \text{Then, [ ] and [ ] were working at the [place]. [ ] decided to give a [object] to}
\end{center}
and we use the first two in training and the last in the test set. Thus, the test set relies on unseen templates, names, objects and places. We used fewer templates than the IOI paper \cite{wang2019structured} in order to simplify tokenization (so that the token positions of our names always align), but our results also hold with shifted templates like in the IOI paper.

On the test partition of this dataset, GPT2-Small achieves an accuracy of
$\approx 91\%$. The average difference of logits between the correct and
incorrect name is $\approx 3.3$, and the logit of the correct name is greater
than that of the incorrect name in $\approx 99\%$ of examples. Note that, while
the logit difference is closely related to the model's correctness, it being
$>0$ does not imply that the model makes the correct prediction, because there
could be a third token with a greater logit than both names.

\subsection{Details for Computing the Gradient Direction $\mathbf{v}_{\text{grad}}$}
\label{app:gradient-details}
For a given example from the test distribution and a given name mover head, we compute the gradient of the difference of attention scores from the final token position to the 3rd and 5th token in the sentence (where the two name tokens always are in our data). We then average these gradients over a large sample of the full test distribution and over the three name mover heads, and finally normalize the resulting vector to have unit $\ell_2$ norm. 

We note that there is a `closed form' way to compute approximately the same quantity that requires no optimization. Namely, for a single example we can collect the keys $k_{S}, k_{IO}$ to the name mover heads at the first two names in the sentence (the \textbf{S} and \textbf{IO} name). Then, for a single name mover head with query matrix $W_Q$, a maximally causal direction $v$ in the residual stream at the last token position after layer 8 will be one such that $W_Qv$ is in the direction of $k_S - k_{IO}$, because the attention score is simply the dot product between the keys and queries. We can use this to `backpropagate' to $v$ by multiplying with the pseudoinverse $W_Q^+$. This is slightly complicated by the fact that we have been ignoring layer normalization, which can be approximately accounted for by estimating the scaling parameters (which tend to concentrate well) from the IOI data distribution. We note that this approach leads to broadly similar results.

\subsection{Training Details for DAS}
\label{app:das-training-details}

To train DAS, we always sample examples from the training IOI distribution. We sample equal amounts of pairs of
base (which will be patched into) and source (where we take the activation to
patch in from) prompts where the two names are the same between the prompts, and
pairs of prompts where all four names are distinct. We optimize DAS to maximize
the logit difference between the name that should be predicted if the position
information from the source example is correct and the other name. 

For training, we use a learned rotation matrix as in the original DAS paper \citep{geiger2023finding}, parametrized with \texttt{torch.nn.utils.parametrizations.orthogonal}. We use the Adam optimizer and minibatch training over a training set of several hundred patching pairs. We note that results remain essentially the same when using a higher number of training examples.

\subsection{Discussion of the Magnitude of the Illusion}
\label{app:illusion-magnitude}
While the contribution of the $\mathbf{v}_{\text{MLP}}$ patch to logit difference may appear relatively small, we note that this is the result of patching a direction in a single model component at a single token position. Typical circuits found in real models (including the IOI circuit from \citet{wang2022interpretability}) are often composed of multiple model components, each of which contribute. In particular, the position signal itself is written to by 4 heads, and chiefly read by 3 other heads. As computation tends to be distributed, when patching an individual component accuracy may be a misleading metric (eg patching 1 out of 3 heads is likely insufficient to change the output), and a fractional logit diff indicates a significant contribution.
By contrast, patching in the residual stream is a more potent intervention,
because it can affect \textit{all} information accumulated in the model that is
communicated to downstream components.

\subsection{Random ablation of MLP weights}
\label{app:random-mlp}
How certain are we that MLP8 doesn’t actually matter for the IOI task? While we find the IOI paper analysis convincing, to make our results more robust to the possibility that it does matter, we also design a further experiment. 

Given our conceptual picture of the illusion, the computation performed by the MLP layer where we find the illusory subspace does not matter as long as it propagates the correlational information about the position feature from the residual stream to the hidden activations, and as long as the output matrix $W_{out}$ is full rank (also, see the discussion in \ref{sec:discussion}). Thus, we expect that if we replace the MLP weights by randomly chosen ones with the same statistics, we should still be able to exhibit the illusion.

Specifically, we randomly sampled MLP weights and biases such that the norm of the output activations matches those of MLP8. As random MLPs might lead to nonsensical text generation, we don’t replace the layer with the random weights, but rather train a subspace using DAS on the MLP activations, and add the difference between the patched and unpatched output of the random MLP to the real output of MLP8. This setup finds a subspace that reduces logit difference even more than the $\mathbf{v}_{\text{MLP}}$ direction.

This suggests that the existence of the $\mathbf{v}_{\text{MLP}}$ subspace is less about \emph{what} information MLP8 contains, and more about \emph{where} MLP8 is in the network.

\subsection{Generalization to high-dimensional Subspaces}
\label{app:generalization-high-dim}

In the main text, we focus on activation patching in one-dimensional subspaces for clarity. Here, we extend the discussion to higher-dimensional subspaces and show that the interpretability illusion generalizes to high-dimensional linear subspaces.

We investigate two different $100$-dimensional subspaces $U_{\text{MLP8}}$ in MLP8 and $U_{\text{resid8}}$ in the output of layer 8. Specifically, we used DAS to find orthonormal bases $V_{\text{MLP}}$ and $V_{\text{resid}}$ that align the position information in these two locations, as explained in \ref{app:higher-dim}.
We found that these subspaces performed slightly better compared to their 1-dimensional counterparts (for $V_{\text{resid}}$: $190 \%$ FLDD and $89 \%$ interchange accuracy; for $V_{\text{MLP}}$: $62 \%$ FLDD and $13 \%$ interchange accuracy). 

We hypothesize that the subspace trained on MLP8 is pathological while the subspace in the residual stream is not. 
To test this, we decompose every basis vector $\textbf{v}^{(d)}_{\{\text{MLP8, resid8}\}}$ into its projection $\textbf{v}^{\text{nullspace}}_{\{\text{MLP8, resid8}\}}$ on the nullspace $\text{ker} \{W_{\text{out}}, W_Q\}$ and its orthogonal complement $\textbf{v}^{\text{rowspace}}_{\{\text{MLP8, resid8}\}}$ such that 

\[
\textbf{v}^{(d)}_{\{\text{MLP8, resid8}\}} = \textbf{v}^{\text{nullspace}}_{\{\text{MLP8, resid8}\}} + \textbf{v}^{\text{rowspace}}_{\{\text{MLP8, resid8}\}}. 
\]

Note that $W_Q$ denotes the query weight of name mover head 9.9. We then patched the 200-dimensional subspace spanned by $\hat{V}$ with 
\[
\hat{V} = \mathcal{QR}([\textbf{v}^{\text{nullspace}}_1, ..., \textbf{v}^{\text{nullspace}}_d, \textbf{v}^{\text{rowspace}}_1, ..., \textbf{v}^{\text{rowspace}}_d])
\]
composed out of the decomposed subspace vectors and orthonormalized using QR-decomposition (see Figure \ref{fig:pos_100d_fldd}). For patching the output of layer 8, FLDD and interchange accuracy remained similarly high. (FLDD: $200 \%$, interchange accuracy: $86 \%$). However, patching on MLP8 mostly removes the effect of the patch (FLDD: $17 \%$, interchange accuracy: $2 \%$). Thus, the causally disconnected subspace is required for patching MLP8 which suggests that the interpretability illusion generalizes to higher-dimensional subspaces.

\begin{figure}[ht]
    \centering
    \begin{tikzpicture}
        \node[anchor=south west,inner sep=0] (image) at (0,0) {\includegraphics[width=.6\textwidth]{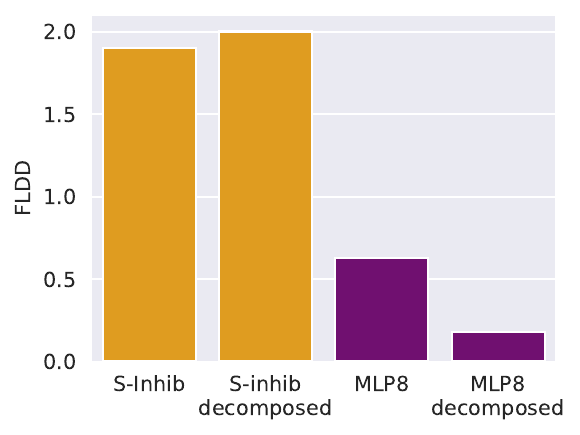}};
    \end{tikzpicture}
    \caption{Fractional logit difference decrease (FLDD) for patching a 100-dimensional subspace on the S-inhibition heads or on MLP8; "decomposed" patches the 200-dimensional subspace made out of the nullspace projection vectors into $W_Q$ of name mover or $W_{out}$ of MLP8, respectively, and their orthogonal complements}
    \label{fig:pos_100d_fldd}
\end{figure}

\subsection{Overfitting on Small Datasets}
\label{app:overfitting}

\begin{figure}[ht]
    \centering
    \begin{tikzpicture}
        \node[anchor=south west,inner sep=0] (image) at (0,0)
        {\includegraphics[width=1.0\textwidth]{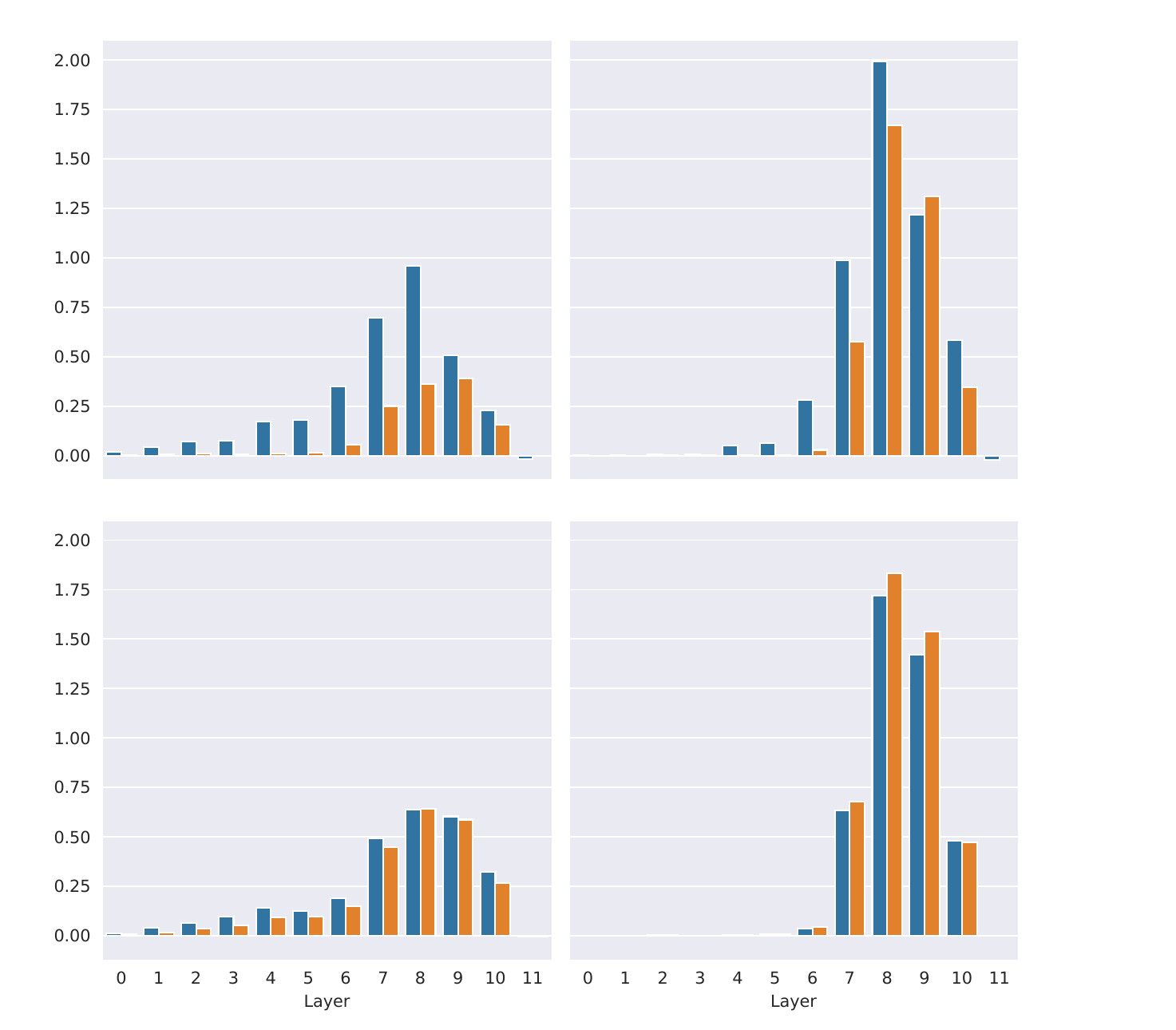}};
        \begin{scope}[x={(image.south east)},y={(image.north west)}]
            \node [rotate=90] at (-0,0.75) {FLDD (small training set)};
            \node [rotate=90] at (-0,0.3) {FLDD (full training set)};
            \node [anchor=east] at (0.35,1.0) {MLP patch};
            \node [anchor=east] at (0.8,1.0) {Residual stream patch};
        \end{scope}
    \end{tikzpicture}
    \caption{FLDD for different IOI-position subspaces: Subspaces were fitted to either a small version of the IOI dataset that only contained 2 names (first row) or on the full dataset (second row) using activations from the MLP (first column) or the residual stream (second column). Subspace performance on the IOI task was evaluated on the training (blue) distribution and the full test dataset containing all names (orange)}
    \label{fig:overfitting}
\end{figure}

How important is a large and diverse dataset for training DAS? We initially hypothesized
that for very small datasets, it is possible to find working subspaces in all
layers as there are only a few fixed activation vectors in each layer and we
might be able to find subspaces that utilize this noise to overfit.

To test this, we created a small IOI dataset containing only two names from a
fixed template. We fitted a one-dimensional subspace using DAS for every layer
on that dataset and the full dataset as a control (Figure
\ref{fig:overfitting}). We repeated the experiment for subspaces in the MLP and
residual stream and evaluated the subspaces on their train distribution and a
test distribution containing all names and templates. FLDD was highest in layer
8, the component between S-inhibition heads and name movers, and also high in
neighboring layers that still contain IOI information (e.g. some of the
S-inhibition heads are in layer 7 and some of the name movers are in layer 10).
Moreover, train FLDD was significantly higher than test FLDD when trained on
only 2 names. 

Importantly, we also observe that subspaces optimized on the small dataset
reached a FLDD bigger than zero in some of the other layers but contrary to our
expectation, this was neither high in absolute terms nor compared to subspaces
trained on the full distribution (see Figure \ref{fig:overfitting}).

\section{Additional Details for Section \ref{sec:ground-truth}}
\label{app:ground-truth}

\textbf{Which model components write to the $\mathbf{v}_{\text{resid}}$ direction?}
To test how every attention head and MLP contributes to the value of projections on $\mathbf{v}_{\text{MLP}}$, we sampled activations from head and MLP outputs at the last token position of IOI prompts, and calculated their dot product with $\mathbf{v}_{resid}$ (Figure \ref{fig:which-heads-to-resid}).
We found that the dot products of most heads and MLPs was low, and that the S-inhibition heads were the only heads whose dot product differed between different patterns ABB and BAB.
This shows that only the S-inhibition heads write to the $\mathbf{v}_{\text{resid}}$ direction (as one would hope).
Importantly, this test separates $\mathbf{v}_{resid}$ from the interpretability illusion $\mathbf{v}_{MLP}$.
While patching $\mathbf{v}_{\text{MLP8}}$ also writes to $\mathbf{v}_{\text{resid8}}$ (i.e. $\mathbf{v}_{\text{MLP8}} W_{out} \approx \mathbf{v}_{\text{resid8}}$), the MLP layer does not write this subspace on the IOI task (see Figure \ref{fig:mlp8-output-projections}). This further supports the observation that the $\mathbf{v}_{\text{MLP}}$ patch activates a dormant pathway in the model.

\textbf{Generalization beyond the IOI distribution}. 
We also investigate how the subspace generalizes. We sample prompts from OpenWebText-10k and look at those with particularly high and low activations in $\mathbf{v}_{\text{sinhib}}$. Representative examples are shown in Figure \ref{fig:openwebtext2} together with the name movers attention at the position of interest, how the probability changes after subspace ablation, and how the name movers attention changes.

\textbf{Stability of found solution}.
Finally, we note that solutions found by DAS in the residual stream are stable, including when trained on a subset of S-inhibition heads (see Figure \ref{fig:pos-das-is-robust}).

\begin{figure}[ht]
    \centering
    \begin{tikzpicture}
        \node[anchor=south west,inner sep=0] (image) at (0,0) {\includegraphics[width=.6\textwidth]{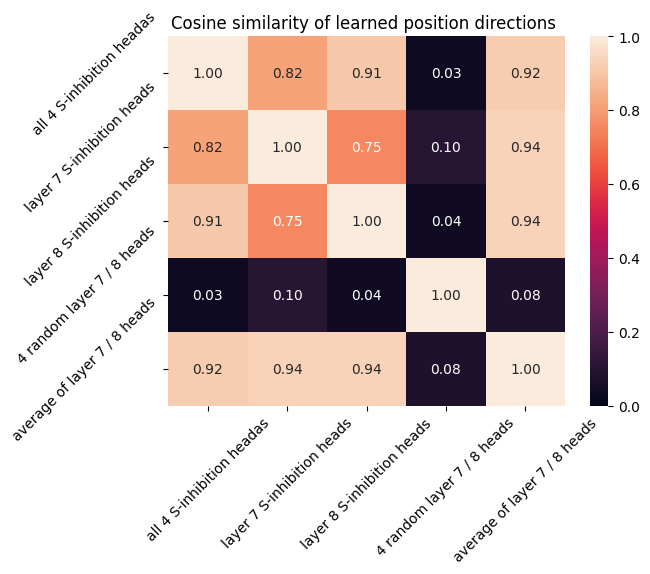}};
    \end{tikzpicture}
    \caption{Cosine Similarity between learned position subspaces in the S-inhibition heads is high even when using only a subset of S-inhibition heads for training}
    \label{fig:pos-das-is-robust}
\end{figure}

\begin{figure}[ht]
    \centering
    \begin{tikzpicture}
        \node[anchor=south west,inner sep=0] (image) at (0,0) {\includegraphics[width=1
\textwidth]{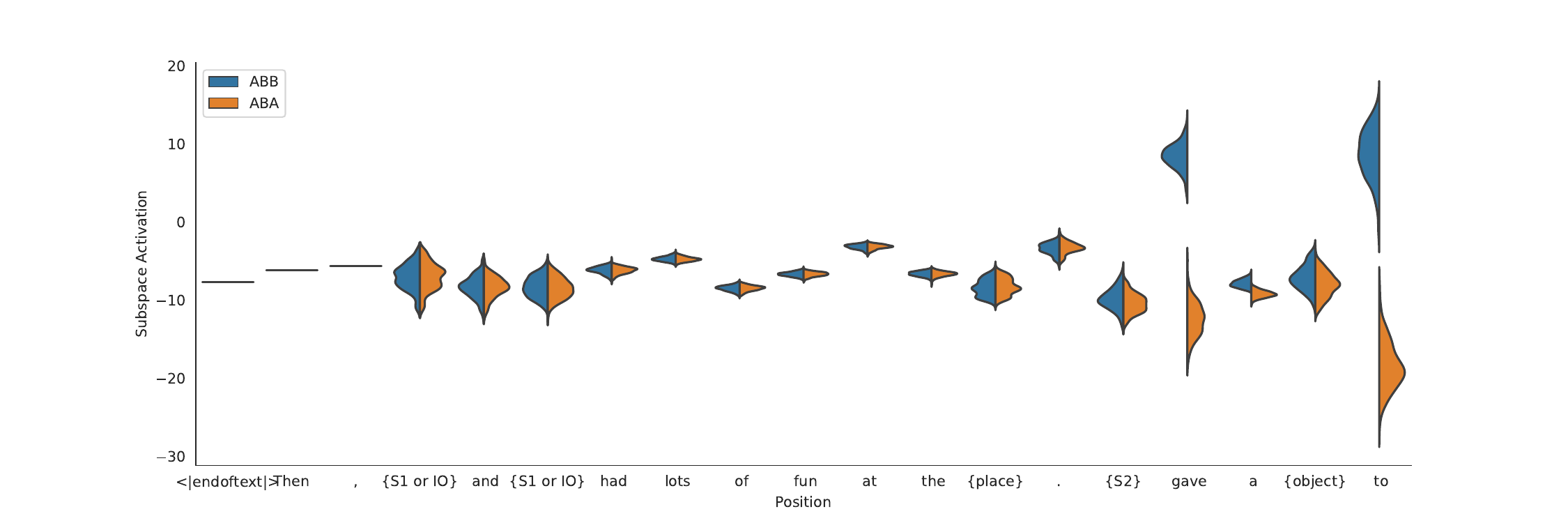}};
    \end{tikzpicture}
    \caption{The IOI position subspace activates at words that predict a repeated name. S-inhibition subspace activations for different IOI prompts per position}
    \label{fig:IOI-violin}
\end{figure}

\begin{figure}[ht]
    \centering
    \begin{tikzpicture}
        \node[anchor=south west,inner sep=0] (image) at (0,0) {\includegraphics[width=1\textwidth]{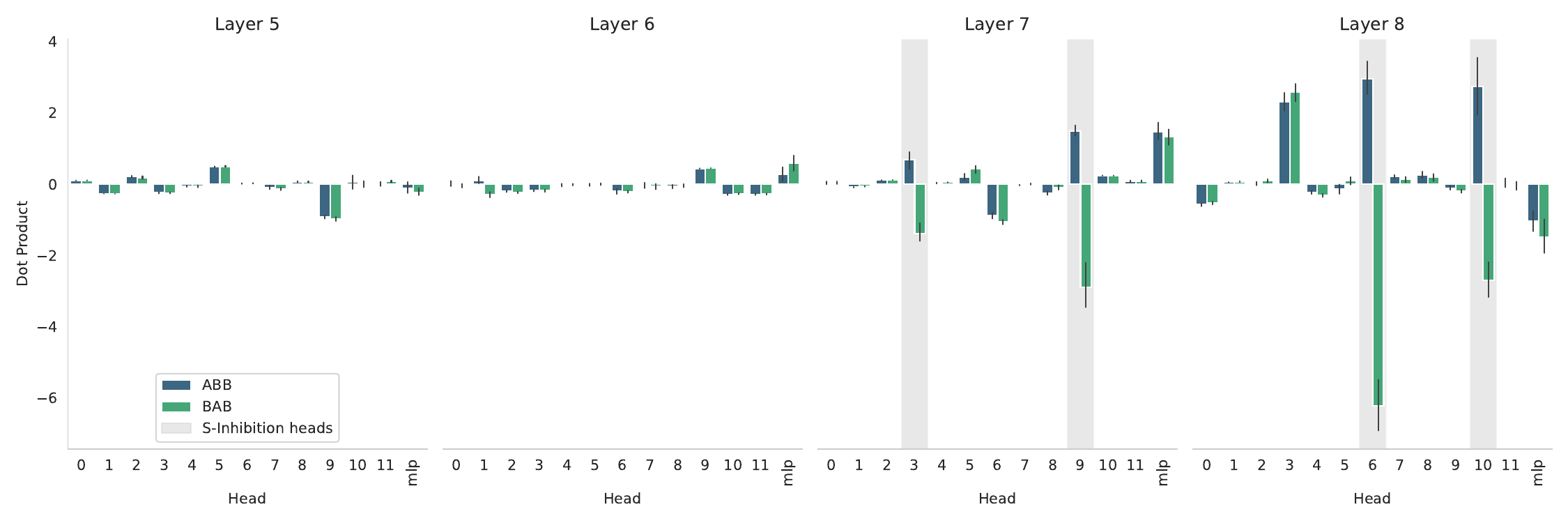}};
    \end{tikzpicture}
    \caption{S-Inhibition heads but not MLP8 write to the position subspace in the residual stream that is causally connected to the name movers on the IOI task}
    \label{fig:which-heads-to-resid}
\end{figure}

\begin{figure}[ht]
    \centering
    \begin{tikzpicture}
        \node[anchor=south west,inner sep=0] (image) at (0,0) {\includegraphics[width=.7\textwidth]{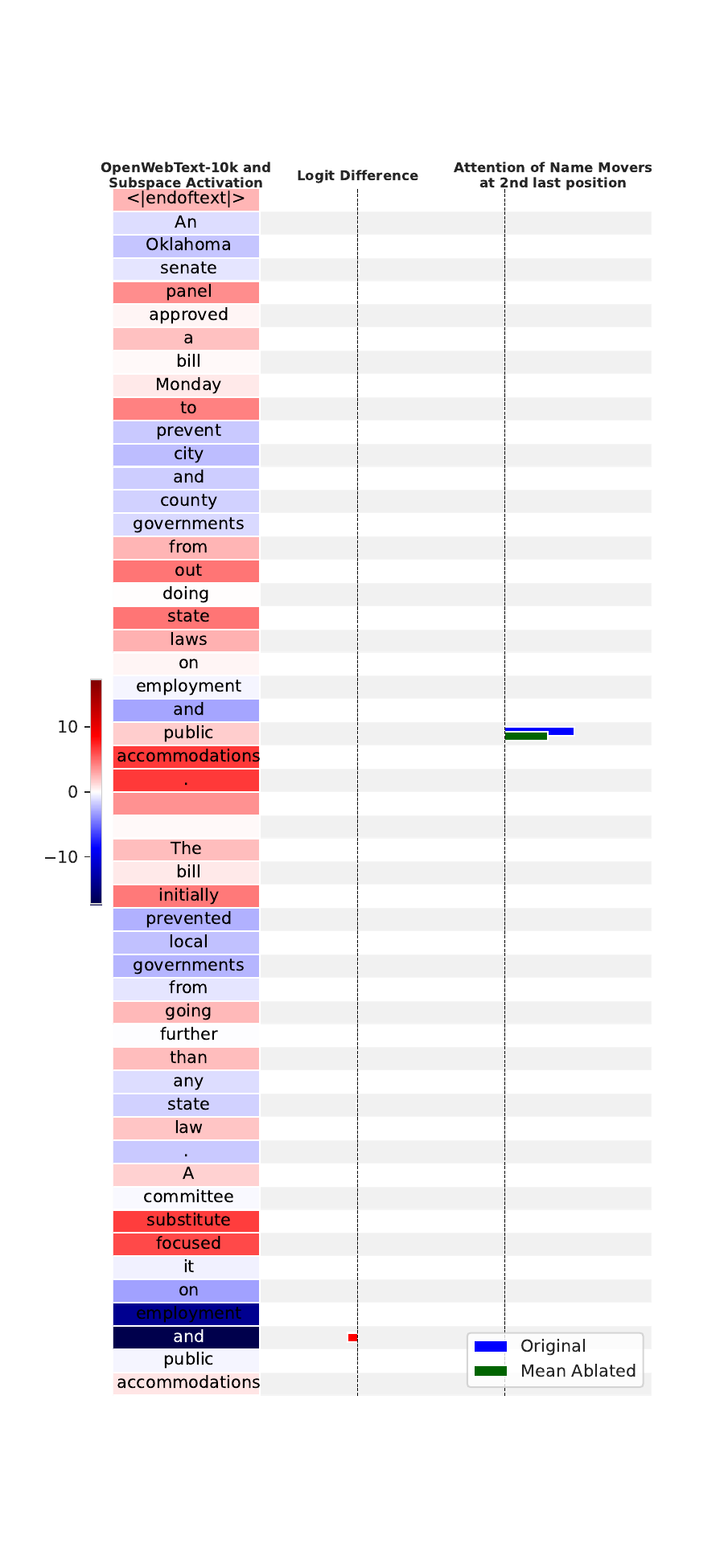}};
    \end{tikzpicture}
    \caption{The IOI position subspace generalizes to arbitrary OpenWebText prompts}
    \label{fig:openwebtext2}
\end{figure}

\begin{figure}[ht]
    \centering
    \begin{tikzpicture}
        \node[anchor=south west,inner sep=0] (image) at (0,0) {\includegraphics[width=.7\textwidth]{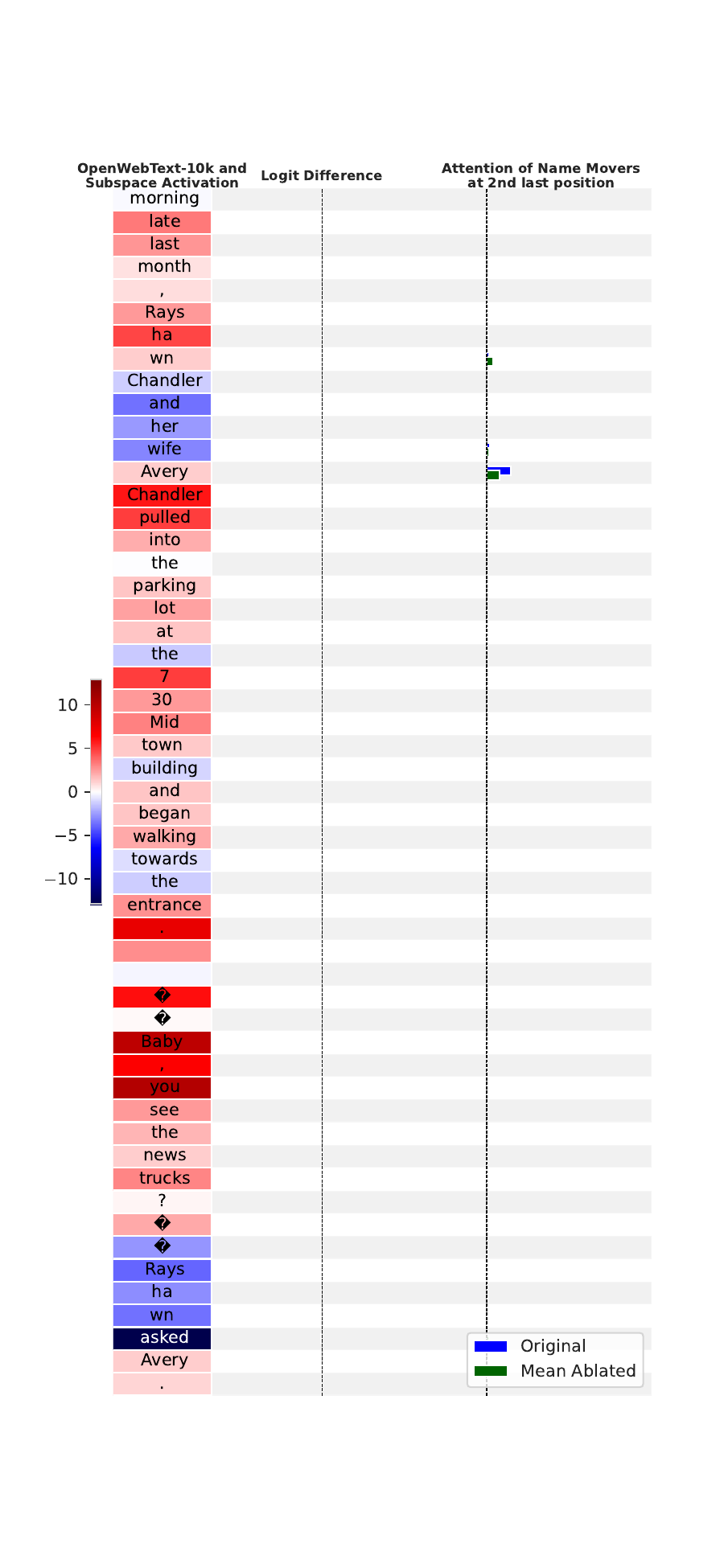}};
    \end{tikzpicture}
    \caption{}
    \label{fig:openwebtext8}
\end{figure}

\begin{figure}[ht]
    \centering
    \begin{tikzpicture}
        \node[anchor=south west,inner sep=0] (image) at (0,0) {\includegraphics[width=.7\textwidth]{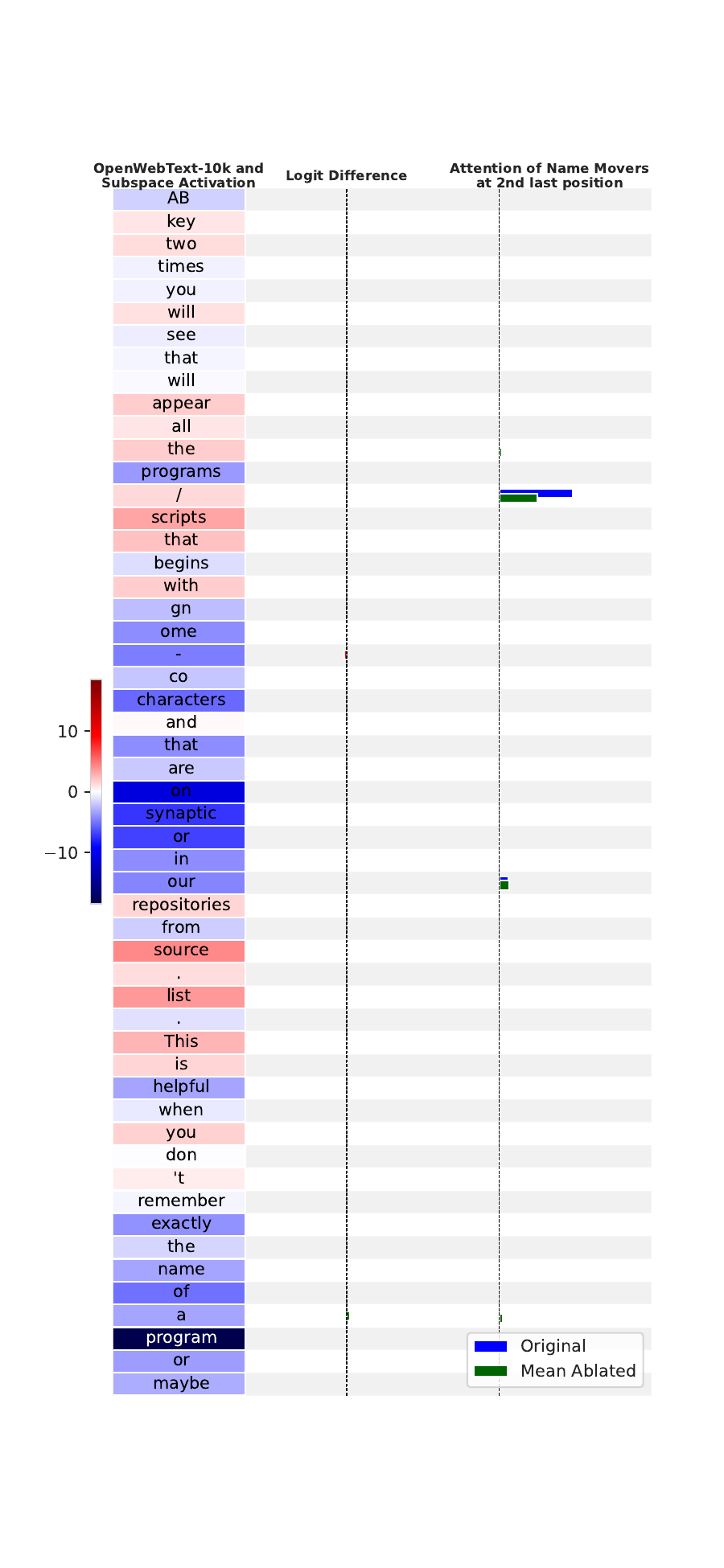}};
    \end{tikzpicture}
    \caption{}
    \label{fig:openwebtext10}
\end{figure}

\section{Additional details for Section \ref{sec:facts}}    
\label{app:rome}

\subsection{Dataset construction and training details}
\label{app:fact-patching-details}

We use the first 1000 examples from the \textsc{CounterFact} dataset
\citep{meng2022locating}. We filter the facts which GPT2-XL correctly recalls.
Out of the remaining facts, for each relation we form all pairs of distinct
facts, and we sample 5 such pairs from each relation with at least 5 facts. This
results in a collection of 40 fact pairs spanning 8 different relations. We then
use these facts as follows:
\begin{itemize}
\item for the ROME experiments in Subsection \ref{sub:from-rome-to-patch}, we
define edits by requesting one of the facts in each pair to be rewritten with
the object of the other fact;
\item for the activation patching experiments in Subsection
\ref{sub:fact-patching}, we patch from the last token of $s'$ in $B$ to the last
token of $s$ in $A$ (prior work has shown that the fact is retrieved on $s$
\citep{geva2023dissecting}), and we again use DAS \citet{geiger2023finding} to
optimize for a direction that maximizes the logit difference between $o'$ and
$o$.
\end{itemize}

\subsection{Additional fact patching experiments}
\label{app:additional-fact-patching}
In figure \ref{fig:fact-patching-violins}, we show the distribution of the fractional logit difference metric (see Subsection \ref{sub:ioi-methodology} for a definition) when patching between facts as described in Subsection \ref{sub:fact-patching}. Like in the related Figure \ref{fig:fact-patching}, we observe that, while patching along the directions found by DAS achieves strongly negative values (indicating that the facts are very often successfully changed by the patch), the interventions that replace the entire MLP layer or only the causally relevant component of the DAS directions have no such effect.

\begin{figure}
    \centering
    \begin{tikzpicture}
        \node [anchor=east, rotate=90] at (-0.15,4.2) {Fractional logit difference};
        \node [anchor=north] at (3.5,0.0) {Intervention layer};
        \node[anchor=south west,inner sep=0] (image) at (0,0) {\includegraphics[width=0.5\textwidth]{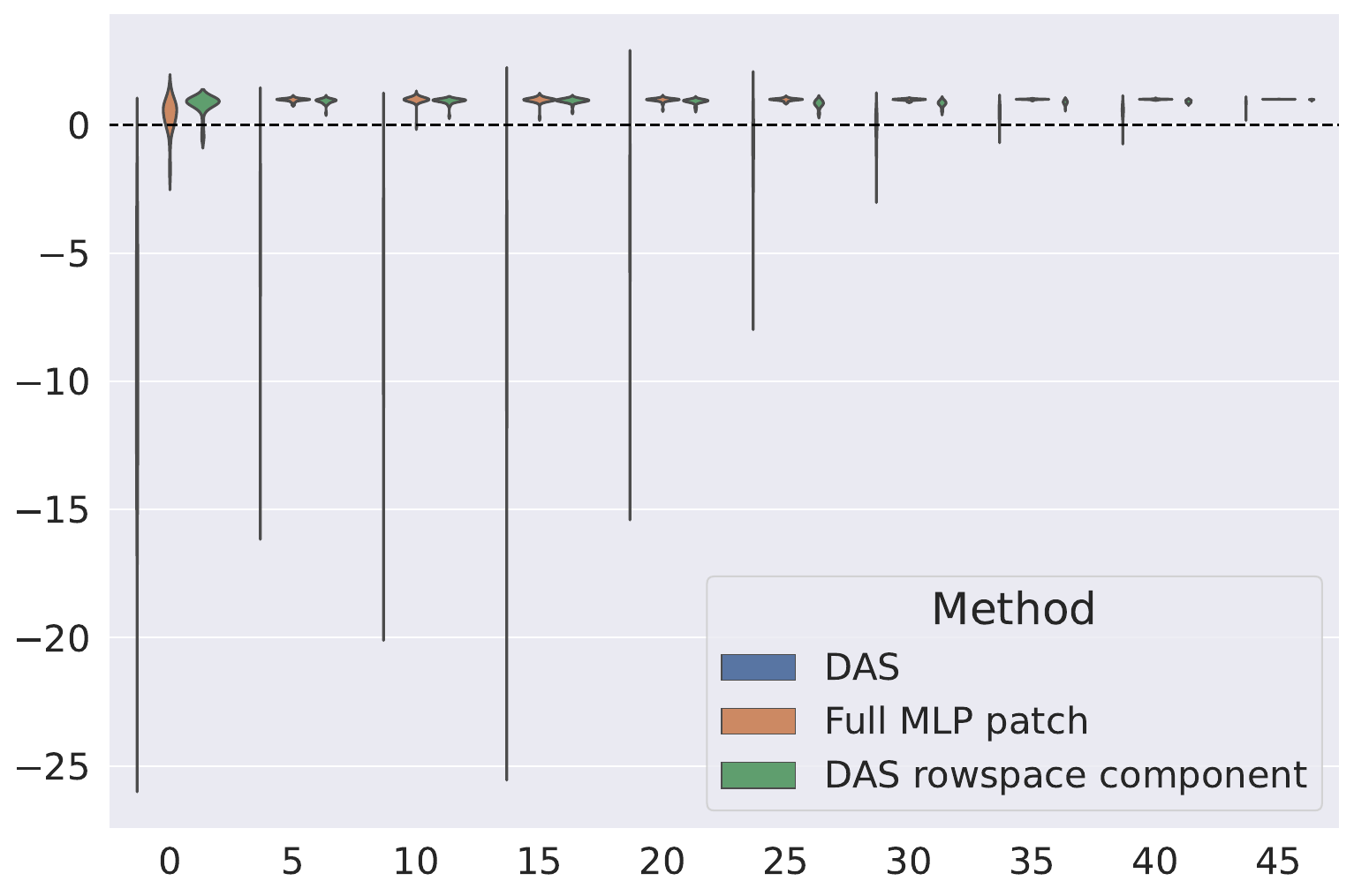}};
    \end{tikzpicture}
    \caption{Fractional logit difference distributions under three interventions: patching along the direction found by DAS (blue), patching the component of the DAS direction in the rowspace of $W_{out}$ (green), and patching the entire hidden MLP activation (orange).}
    \label{fig:fact-patching-violins}
\end{figure}

Next, we observe that the nullspace component of the patching direction is the one similar to the variation in the inputs (difference of last-token activations at the two subjects). Specifically, in Figure \ref{fig:fact-patching-correlation-violins}, we plot the (absolute value of the) cosine similarity between the difference in activations for the two last subject tokens, and the nullspace component of the DAS direction. We note that this similarity is consistently significantly high (note that it can be at most $1$, which would indicate perfect alignment).

\begin{figure}
    \centering
    \begin{tikzpicture}
        \node [anchor=east, rotate=90] at (-0.15,4.4) {Absolute cosine similarity};
        \node [anchor=north] at (3.5,0.0) {Intervention layer};
        \node[anchor=south west,inner sep=0] (image) at (0,0) {\includegraphics[width=0.5\textwidth]{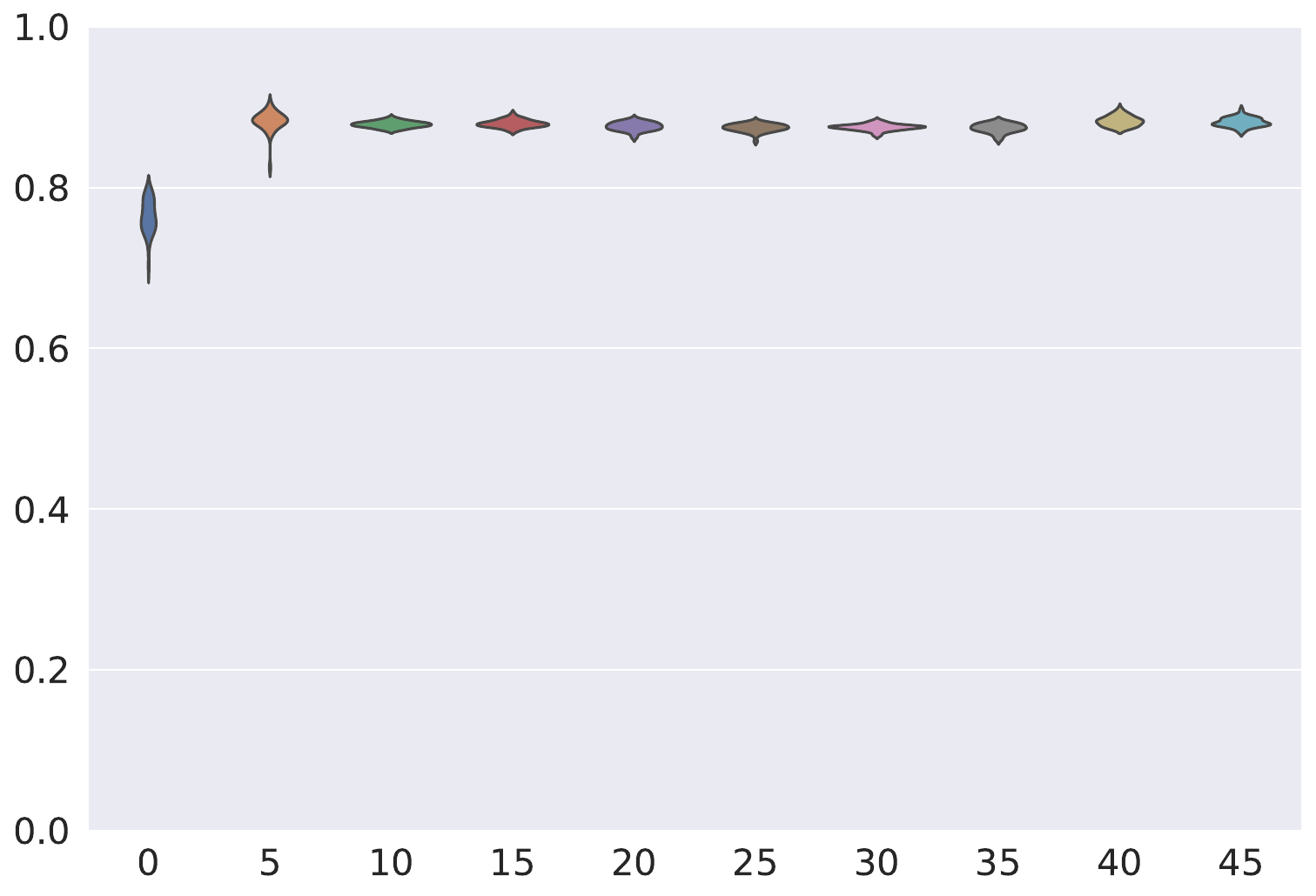}};
    \end{tikzpicture}
    \caption{Distribution of the absolute value of the cosine similarity between the nullspace component of the DAS fact patching directions and the difference in activations of the last tokens of the two subjects.}
    \label{fig:fact-patching-correlation-violins}
\end{figure}

Finally, we observe that the nullspace component of the patching direction is a non-trivial part of the direction in Figure \ref{fig:fact-patching-norm-violins}, where we plot the distribution of the $\ell_2$ norm of this component.
\begin{figure}
    \centering
    \begin{tikzpicture}
        \node [anchor=east, rotate=90] at (-0.15,4.6) {Norm of nullspace component};
        \node [anchor=north] at (3.5,0.0) {Intervention layer};
        \node[anchor=south west,inner sep=0] (image) at (0,0) {\includegraphics[width=0.5\textwidth]{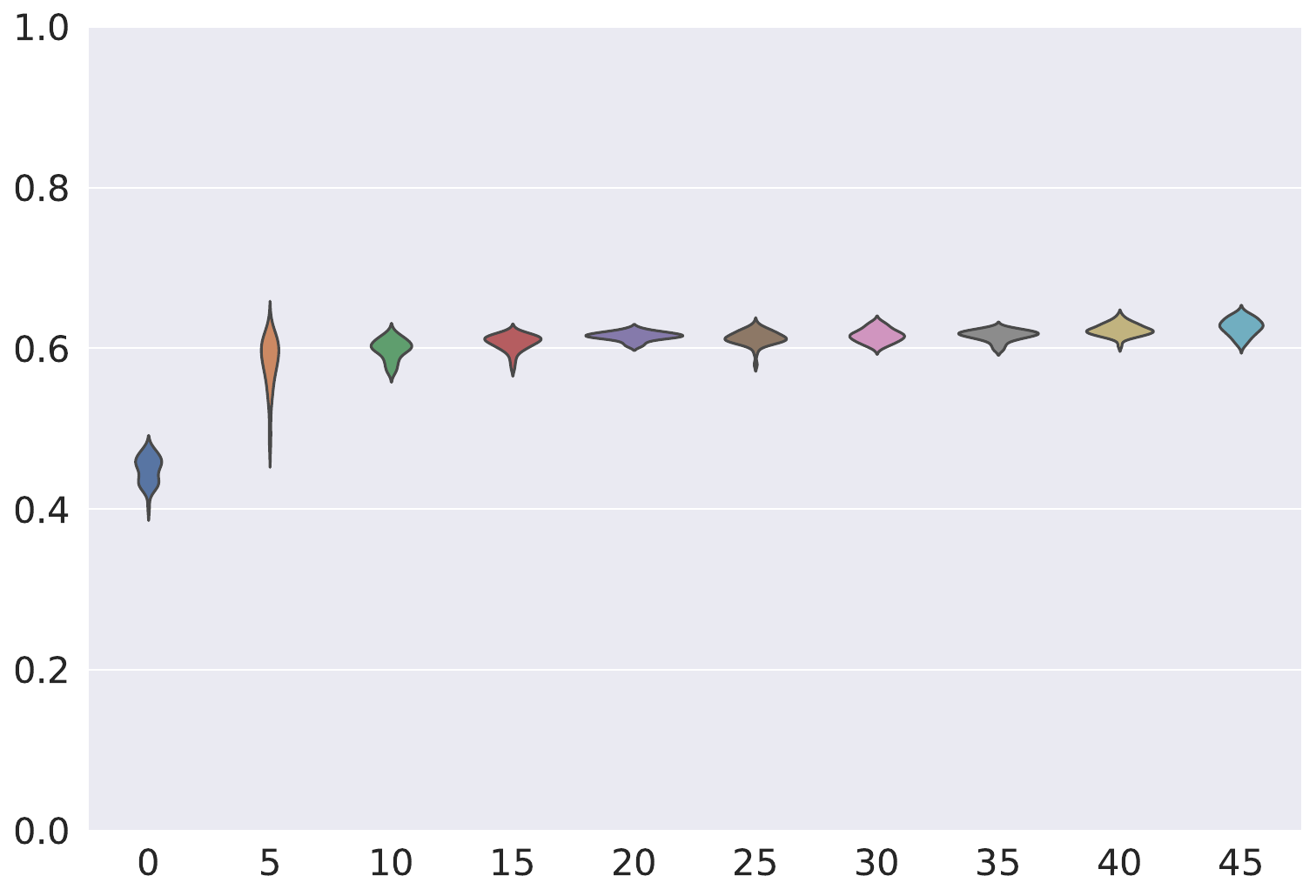}};
    \end{tikzpicture}
    \caption{Distribution of the norm of the nullspace component of the DAS direction across intervention layers.}
    \label{fig:fact-patching-norm-violins}
\end{figure}

\subsection{ROME implementation details}
\label{app:rome-implementation}
ROME takes as input a vector $\mathbf{k}\in\mathbb{R}^{d_{\text{MLP}}}$ representing the subject (e.g. an average of last-token representations of the subject) and a vector $\mathbf{v}\in\mathbb{R}^{d_{\text{resid}}}$ which, when output by the MLP layer, will cause the model to predict a new object for the factual prompt, but at the same time won't change other facts about the subject.
ROME modifies the MLP weight by setting $W_{out}' = W_{out} +
\mathbf{a}\mathbf{b}^\top$, where $\mathbf{a}\in\mathbb{R}^{d_{\text{resid}}},
\mathbf{b}\in\mathbb{R}^{d_{\text{MLP}}}$ are chosen so that $W_{out}'\mathbf{k}
= \mathbf{v}$, and the MLP's output is otherwise minimally changed. Without loss
of generality, the first condition implies that $\mathbf{a}= \mathbf{v} -
W_{out}\mathbf{k}$ and $\mathbf{b}^\top \mathbf{k} = 1$; the second condition is
then modeled by minimizing the variance of $\mathbf{b}^\top \mathbf{x}$ when
$\mathbf{x}\sim\mathcal{N}(0, \Sigma)$ for an empirical estimate
$\Sigma\in\mathbb{R}^{d_{\text{MLP}}\times d_{\text{MLP}}}$ of the covariance of
MLP activations (see Lemma \ref{lem:rome-as-opt} in Appendix \ref{app:rome} for
details and a proof). In all our experiments involving ROME, we use GPT2-XL
\citep{radford2019language}, and we use the precomputed values of $\Sigma$ from
\cite{meng2022locating} accessible online
\href{https://rome.baulab.info/data/stats/gpt2-xl/wikipedia_stats/}{here}.

\subsection{ROME as an Optimization Problem}
\label{app:rome-as-opt}

We now review the ROME method from \citet{meng2022locating} and show how it can be characterized as the solution of a simple optimization problem. Following the terminology of \ref{sub:patch-implies-rome}, let us have an MLP layer with an output projection \(W_{out}\), a key vector \(\mathbf{k}\in\mathbb{R}^{d_{\text{MLP}}}\) and a value vector \(\mathbf{v}\in \mathbb{R}^{d_{\text{resid}}}\).

In \citet{meng2022locating}, equation 2, the formula for the rank-1 update to \(W_{out}\) is given by
\begin{align}
\label{eq:rome}
    W_{out}' = W_{out} + (\mathbf{v}-W_{out}\mathbf{k})\frac{\mathbf{k}^\top\Sigma^{-1}}{\mathbf{k}^\top\Sigma^{-1}\mathbf{k}}
\end{align}
where \(\Sigma\) is an empirical estimate of the uncentered covariance of the pre-\(W_{out}\) activations. We derive the following equivalent characterization of this solution (which may be of independent interest):

\begin{lemma}
\label{lem:rome-as-opt}
    Given a matrix \(W_{out}\in\mathbb{R}^{d_{\text{resid}}\times
    d_{\text{MLP}}}\), a key vector \(\mathbf{k}\in\mathbb{R}^{d_{\text{MLP}}}\)
    and a value vector \(\mathbf{v}\in \mathbb{R}^{d_{\text{resid}}}\), let
    \(\Sigma\succ 0, \Sigma\in\mathbb{R}^{d_{\text{MLP}}\times d_{\text{MLP}}}\)
    be a positive definite matrix (specifically, the uncentered empirical
    covariance), and let \(\mathbf{x}\sim\mathcal{N}(0, \Sigma)\) be a
    normally distributed random vector with mean $0$ and covariance \(\Sigma\).
    Then, the ROME weight update is \(W_{out}' = W_{out} + \mathbf{a}\mathbf{b}^\top\) where \(\mathbf{a}\in \mathbb{R}^{d_{\text{resid}}}, \mathbf{b}\in \mathbb{R}^{d_{\text{MLP}}}\) solve the optimization problem
\begin{align*}
    \min_{\mathbf{a}, \mathbf{b}} \operatorname{trace}(\operatorname{Cov}_{\mathbf{x}}\left[W_{out}'\mathbf{x} - W_{out}\mathbf{x}\right])\quad \text{subject to}\quad W_{out}'\mathbf{k} = \mathbf{v}.
\end{align*}
In other words, the ROME update is the update that causes \( W_{out} \) to output \( \mathbf{v} \) on input \( \mathbf{k} \), and minimizes the total variance of the extra contribution of the update in the output of the MLP layer under the assumption that the pre-\( W_{out} \) activations are normally distributed with covariance \( \Sigma \).
\end{lemma}
\begin{proof}
We have $W_{out}' \mathbf{x} - W_{out} \mathbf{x} = \mathbf{a}\mathbf{b}^\top
\mathbf{x}$. Next, Using \(
\mathbb{E}_{\mathbf{x}}[\mathbf{x}\mathbf{x}^\top]=\Sigma \) and the cyclic
property of the trace, we see that
\begin{align*}
    \operatorname{trace}(\operatorname{Cov}_{\mathbf{x}}\left[W_{out}'\mathbf{x} - W_{out}\mathbf{x}\right]) = \|\mathbf{a}\|_2^2 \mathbf{b}^\top \Sigma \mathbf{b}
\end{align*}
We must have \( \mathbf{a}\mathbf{b}^\top \mathbf{k} = \mathbf{v} - W_{out}\mathbf{k} \), so without loss of generality we can rescale \( \mathbf{a}, \mathbf{b} \) so that \( \mathbf{a}=\mathbf{v}-W_{out}\mathbf{k} \). Then, we want to solve the problem
\begin{align*}
    \min_{\mathbf{b}} \mathbf{b}^\top\Sigma \mathbf{b} \quad \text{subject to}\quad \mathbf{b}^\top \mathbf{k} = 1
\end{align*}
which we can solve using Lagrange multipliers. The Lagrangian is
\begin{align*}
    \mathcal{L}(\mathbf{b}, \lambda) = \frac{1}{2}\mathbf{b}^\top\Sigma \mathbf{b} - \lambda \mathbf{b}^\top \mathbf{k}
\end{align*}
and the derivative w.r.t. \( \mathbf{b} \) is \( \Sigma \mathbf{b} - \lambda \mathbf{k} = 0 \), which tells us that \( \mathbf{b} \) is in the direction of \( \Sigma^{-1}\mathbf{k} \). Then the constraint \( \mathbf{b}^\top \mathbf{k} = 1 \) forces the constant of proportionality, and we arrive at \( \mathbf{b} = \frac{\mathbf{k}^\top\Sigma^{-1}}{\mathbf{k}^\top\Sigma^{-1}\mathbf{k}} \)
\end{proof}

\subsection{Connection between 1-dimensional activation patching and model editing}
\label{app:fact-lemma}
\begin{lemma}
\label{lem:patching-implies-rome}
    Given prompts A and B, two token positions $t_A$, $t_B$, and an MLP layer with output projection weight $W_{out}\in\mathbb{R}^{d_{\text{resid}}\times d_{\text{MLP}}}$, let $u_A, u_B\in\mathbb{R}^{d_{\text{MLP}}}$ be the respective (post-nonlinearity) activations at these token positions in this layer.
    If $v$ is a direction in the activation space of the MLP layer, then there exists a ROME edit $W_{out}' = W_{out} + ab^\top$ such that the activation patch from $u_B$ into $u_A$ along $v$ and the edit result in equal outputs of the MLP layer at token $t_A$ when run on prompt A. Moreover, the ROME edit is given by
\begin{align*}
a = \left((u_B - u_A)^\top v\right)W_{out}v \quad \text{ and any $b$ that satisfies }\quad b^\top u_A = 1.
\end{align*}
Choosing $b=\frac{\Sigma^{-1}u_A}{u_A^T\Sigma^{-1}u_A}$ minimizes the change to the model (in the sense of \citet{meng2022locating}) over all such rank-1 edits.
\end{lemma}
\begin{proof}
    The activation after patching from B into A along $v$ is $u_A' = u_A + ((u_B - u_A)^\top v)v$, which means that the change in the output of the MLP layer at this token will be
\begin{align*}
    W_{out}u_A' - W_{out}u_A = ((u_B - u_A)^\top v)W_{out}v
\end{align*}
The change introduced by a fact edit at this token is 
\begin{align*}
    W_{out}'u_A - W_{out}u_A = ab^\top u_A = \left(b^\top u_A\right)\left((u_B - u_A)^\top v\right)W_{out}v
\end{align*}
and the two are equal because $b^\top u_A = 1$. 

To find the $b$ that minimizes the change to the model, we minimize the variance of $b^\top x$ when $x\sim\mathcal{N}(0, \Sigma)$ subject to $b^\top u_A = 1$. The variance is equal to $b^\top\Sigma b$, so we have a constrained (convex) minimization problem
\begin{align*}
    \min \frac{1}{2}b^\top\Sigma b \quad \text{subject to}\quad b^\top u_A = 1
\end{align*}
The rest of the proof is the same as in Lemma \ref{lem:rome-as-opt}. Namely, we can solve this optimization problem using Lagrange multiplies. The Lagrangian is
\begin{align*}
    \mathcal{L}(b, \lambda) = \frac{1}{2}b^\top\Sigma b - \lambda b^\top u_A
\end{align*}
and the derivative w.r.t. $b$ is $\Sigma b - \lambda u_A = 0$, which tells us that $b$ is in the direction of $\Sigma^{-1}u_A$. Then the constraint $b^\top u_A = 1$ forces the constant of proportionality.
\end{proof}

\subsection{Additional experiments comparing fact patching and rank-1 editing}
\label{app:rome-extra-exps}
In Figure \ref{fig:fact-patching-rome-logitdiff}, we plot the distributions of the logit difference between the correct object for a fact and the object we are trying to substitute when patching the 1-dimensional subspaces found by DAS, and performing the equivalent rank-1 weight edit according to Lemma \ref{lem:patching-implies-rome}. We observe that the two metrics quite closely track each other, indicating that the additional effects of using a weight edit (as opposed to only intervening at a single token) are negligible.

Similarly, in Figure \ref{fig:fact-patching-rome-accuracy}, we show the success rate of the the two methods in terms of making the model output the object of the fact we are patching from. Again, we observe that they quite closely track each other. 

\begin{figure}
    \centering
    \begin{tikzpicture}
        \node [anchor=east, rotate=90] at (-0.15,3.9) {Logit difference};
        \node [anchor=north] at (3.5,0.0) {Intervention layer};
        \node[anchor=south west,inner sep=0] (image) at (0,0) {\includegraphics[width=0.5\textwidth]{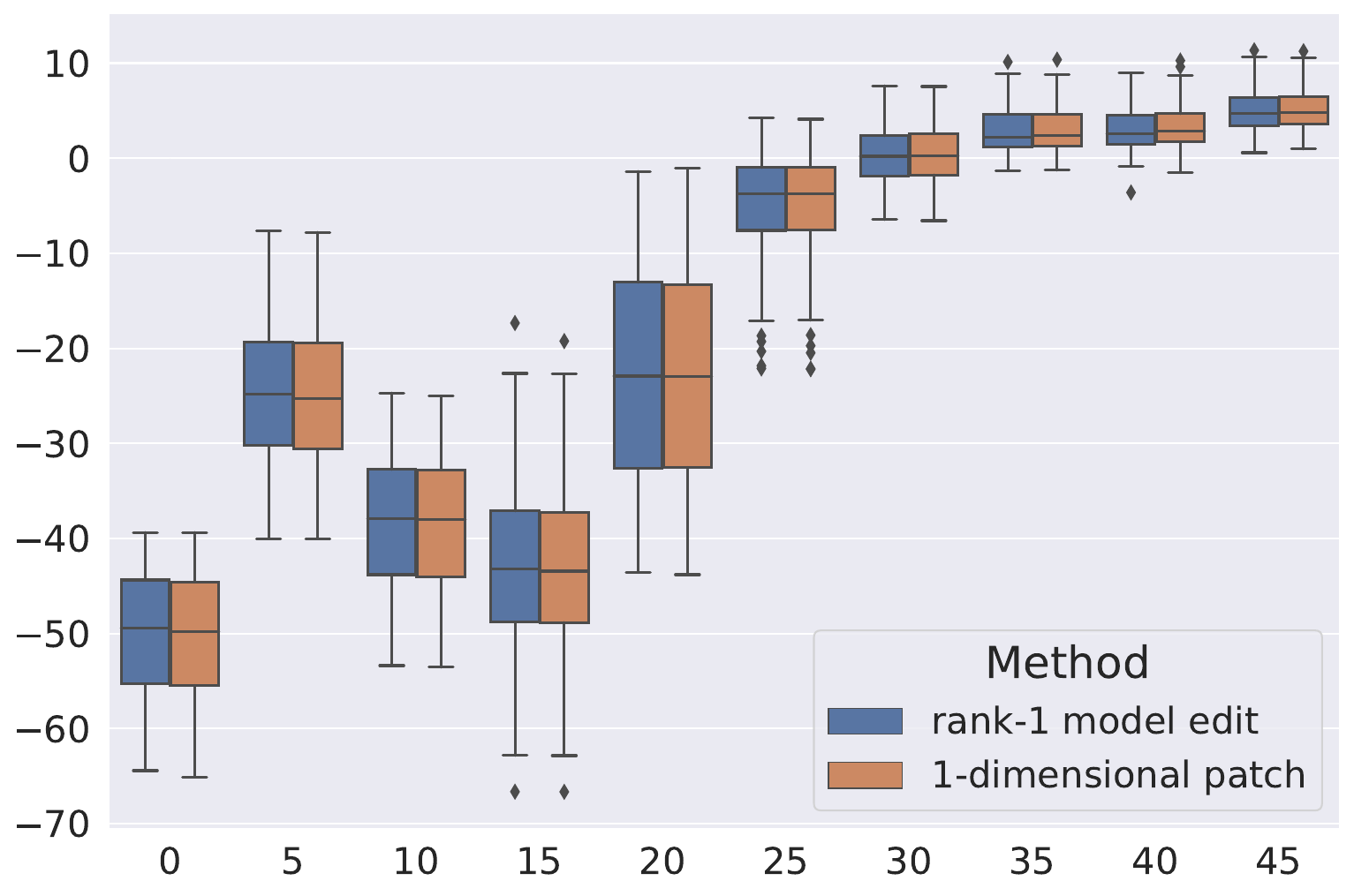}};
    \end{tikzpicture}
    \caption{Comparison of logit difference between 1-dimensional fact patches and their derived rank-1 model edits}
    \label{fig:fact-patching-rome-logitdiff}
\end{figure}

\begin{figure}
    \centering
    \begin{tikzpicture}
        \node [anchor=east, rotate=90] at (-0.15,4.3) {Fact patch success rate};
        \node [anchor=north] at (3.5,0.0) {Intervention layer};
        \node[anchor=south west,inner sep=0] (image) at (0,0) {\includegraphics[width=0.5\textwidth]{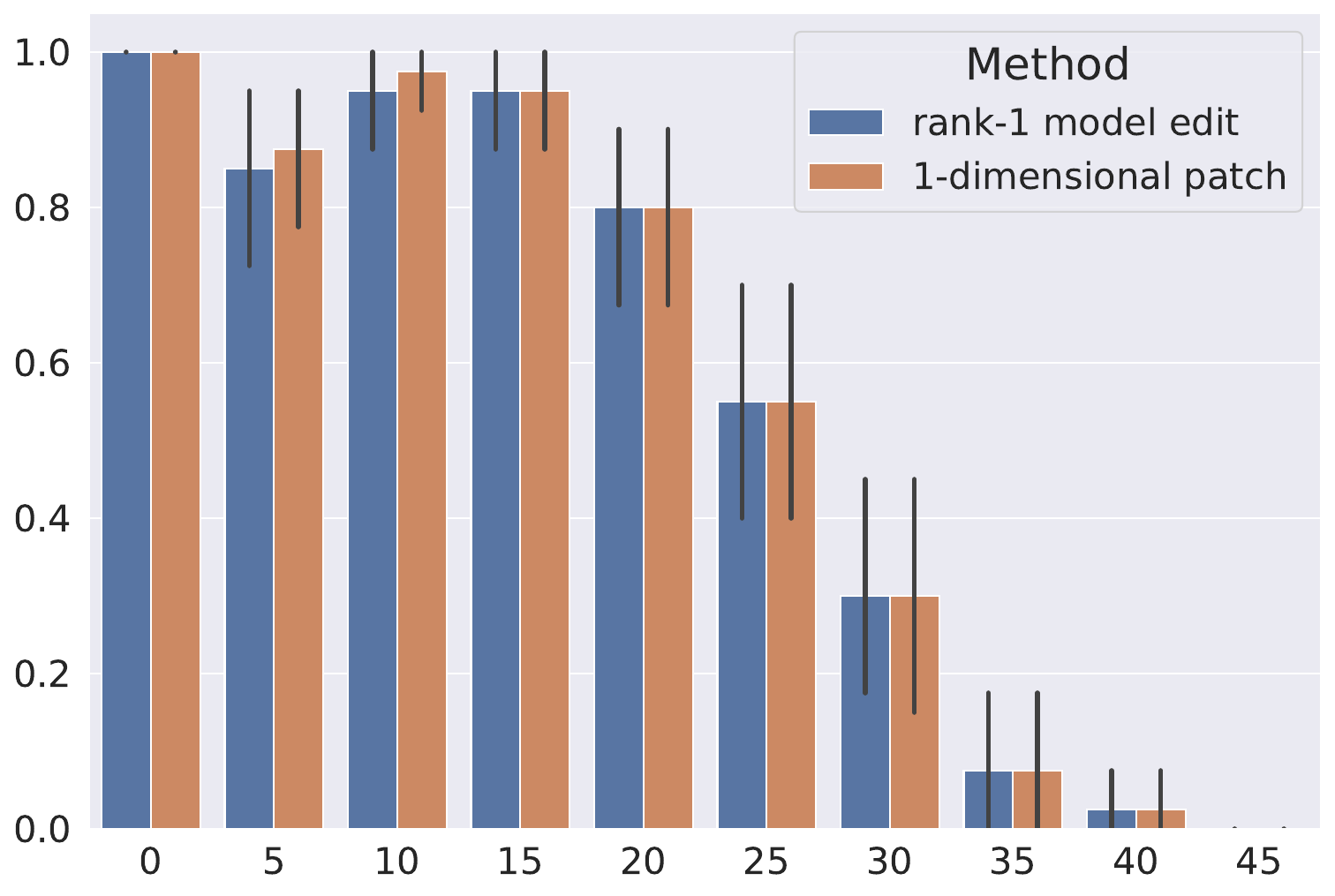}};
    \end{tikzpicture}
    \caption{Comparison of fact editing success rate between 1-dimensional fact patches and their derived rank-1 model edits}
    \label{fig:fact-patching-rome-accuracy}
\end{figure}

\subsection{From Rank-1 Model Edits to Subspace Interventions}
\label{app:from-rank1-to-subspace}

In this section, we describe how, given a rank-1 edit $W_{out}'=W_{out} + ab^T$, to
obtain a direction $v\in \mathbb{R}^{d_{MLP}}$ such that intervening on the
model by setting the projection on $v$ to some constant $c\in \mathbb{R}$ (at
each token) is approximately equivalent to intervening via the rank-1 edit. 

Specifically, given an activation $x\in \mathbb{R}^{d_{MLP}}$, the patched
activation is $x' = x + \l(c - v^Tx\r)v$ and the extra contribution of the
subspace intervention to the output of the MLP layer will be
\begin{align*}
    \operatorname{contrib}_{\text{subspace}}(x) = W_{out}x' - W_{out}x = \l(c-v^Tx\r)Wv.
\end{align*}
Similarly, the extra contribution of the rank-1 edit to the output of the MLP
layer is
\begin{align*}
    \operatorname{contrib}_{\text{rank-1}}(x) = W_{out}'x - W_{out}x = \l(b^Tx\r)a.
\end{align*}
Recall (see Appendix \ref{app:rome-as-opt}) that the ROME method 
\citep{meng2022locating} implicitly treats the activation $x$ as a random vector
sampled from $\mathcal{N}\l(0,\Sigma\r)$ where $\Sigma$ is an empirical estimate
of the covariance. In particular, this distribution is used to quantify the
amount to which a rank-1 edit changes the model.

Motivated by this, we formalize approximating the rank-1 edit by the subspace
intervention using the following criteria analogous to the ROME method:
\begin{itemize}
\item $\mathbb{E}_{x\sim \mathcal{N}(0, \Sigma)} \left[
        \operatorname{contrib}_{\text{subspace}}(x) \right] = \mathbb{E}_{x\sim \mathcal{N}(0, \Sigma)} \left[
        \operatorname{contrib}_{\text{rank-1}}(x) \right] $, i.e. the
        interventions have the same expectation;
\item $W_{out}v \parallel a$, i.e. the interventions point in the same direciton;
\item $\operatorname{trace} \l(\operatorname{Cov}_{x}\left[
    \operatorname{contrib}_{\text{subspace}}(x) - 
    \operatorname{contrib}_{\text{rank-1}}(x)
    \right]\r)$ is minimized, i.e. the two interventions are maximally similar
    with respect to the activation distribution modeled as $x\sim \mathcal{N}
    \l(0, \Sigma\r)$ (this is the criterion used by ROME; recall
    \ref{app:rome-as-opt}).
\end{itemize}

The expectation of $\operatorname{contrib}_{\text{rank-1}}(x)$ is zero, while
the expectation of $ \operatorname{contrib}_{\text{subspace}}(x)$ is
$cW_{out}v$, and since $W_{out}v=0$ would lead to a trivial intervention, we must have
\begin{align*}
    c=0.
\end{align*}
Next, to ensure $W_{out}v\parallel a$, we have to pick $v =\alpha W_{out}^+a +
u$ where $u\in \ker W_{out}$. With this, the covariance minimization can then be
written as
\begin{align*}
    \min_{\alpha, v} \l\|a\r\|_2^2 \l(b+\alpha v\r)^T \Sigma \l(b+\alpha v\r)
\end{align*}
(this is a similar derivation to the one in Appendix \ref{app:rome-as-opt}).
After removing constant terms and setting $w=u/\alpha$, we are left with
\begin{align*}
\min_{\alpha, w}\left[\alpha^4 \l(W_{out}^+a+w\r)^T\Sigma \l(W_{out}^+a+w\r) + 2\alpha^2 b^T\Sigma \l(W_{out}^+a+w\r)\right].
\end{align*}
subject to $W_{out}w=0$. The Lagrangian is
\begin{align*}
    \mathcal{L}\l(\alpha, w,\lambda\r) = \alpha^4 \l(W_{out}^+a+w\r)^T\Sigma
    \l(W_{out}^+a+w\r) + 2\alpha^2 b^T\Sigma \l(W_{out}^+a+w\r) + \lambda^T
    W_{out}w
\end{align*}
with the first-order conditions
\begin{align*}
    \frac{\partial \mathcal{L}}{\partial w} = 2\alpha^4 \Sigma
    \l(W_{out}^+a+w\r)+2\alpha^2 \Sigma b + W_{out}^T\lambda = 0
\end{align*}
and $\partial \mathcal{L} / \partial \lambda = W_{out}w = 0$. Multiplying the
$w$ derivative with $W_{out}\Sigma^{-1}$ on the left gives us a linear system
for $\lambda$:
\begin{align*}
    W_{out}\Sigma^{-1}W_{out}^T\lambda = -2\alpha^2W_{out}b -2\alpha^4a,
\end{align*}
which can be solved assuming we know $\alpha$, and then substituting $\lambda$
in $\frac{\partial \mathcal{L}}{\partial w} = 0$ gives us $w$. In practice, we
guess several values for $\alpha$ (typically, $\alpha^2=0.05$ performs best) and
pick the one resulting in the best value for the objective.

\section{Additional Details for Section \ref{sec:prevalent}}
\label{app:prevalent}

\subsection{Prevalence of Causal Directions in MLP Layers}
\label{subsection:dormant-evidence}

Given an MLP activation $\mathbf{x}$ and a vector $\mathbf{u}\in
\mathbb{R}^{d_{MLP}}$, changing the projection of $\mathbf{x}$ on $\mathbf{u}$
means replacing $\mathbf{x}$ with the new activation $\mathbf{x}' = \mathbf{x} +
\alpha \mathbf{u}$ for some $\alpha\in \mathbb{R}$. This translates to the new
output of the MLP layer being
\begin{align*}
    W_{out} \l(\mathbf{x} + \alpha \mathbf{u}\r) = W_{out} \mathbf{x} + \alpha W_{out} \mathbf{u}.
\end{align*}
Under our assumptions, the direction $\mathbf{u}$ will be causally relevant if
the extra contribution to the residual stream $\alpha W_{out} \mathbf{u}$ points
along $\mathbf{v}$; thus it suffices to find a $\mathbf{u}$ such that $W_{out}
\mathbf{u} \parallel \mathbf{v}$.

As it turns out, we can simply choose $\mathbf{u} = W_{out}^+ \mathbf{v}$.
Indeed, we empirically observe that $W_{out}\in \mathbb{R}^{d_{resid}\times
d_{MLP}}$ is a full-rank matrix\footnote{This is also heuristically plausible:
models want to maximize their expressive capacity, and pre-training datasets are
very complex, so making $W_{out}$ low-rank would not be preferred by
optimization.}, with almost all singular values bounded well
away from $0$ (see Appendix \ref{app:singular}). Since $d_{MLP} > d_{resid}$, it
follows that $W_{out} W_{out}^+ \mathbf{v} = \mathbf{v}$. This establishes that
$\mathbf{u}$ is a causal direction.

\subsection{Prevalence of Directions Discriminating for $C$ in MLP layers}
\label{subsection:disconnected-evidence}

For a feature $\mathbf{u}\in \mathbb{R}^{d_{MLP}}$ to 
discriminate between values of $C$, we need projections of the post-nonlinearity
activations on $\mathbf{u}$ to linearly separate examples according to the
values of $C$. 
By assumption, $\mathbf{v}$ is a good linear separator for the values of $C$ in
the residual stream. We can thus frame our goal as a more general question: 
\begin{center}
  \emph{
  If two sets of activations are linearly separable in the residual stream, are
  their images after the non-linearity also (approximately) linearly separable?
  }
\end{center}
The transformation from residual vectors $\mathbf{x}\in \mathbb{R}^{d_{resid}}$
to post-nonlinearity activations is given by the steps
\begin{align*}
    \mathbf{x}\mapsto \operatorname{LayerNorm}\l(W_{in}\mathbf{x}\r)\mapsto
    \operatorname{gelu}\l(\operatorname{LayerNorm}\l(W_{in}\mathbf{x}\r)\r)
\end{align*}
The composition of $\operatorname{LayerNorm}$ and $W_{in}$ is approximately a
linear operation \citep{elhage2021mathematical}, so the values of the concept
$C$ are also linearly separated in the pre-$\operatorname{gelu}$ activations.
However, it is not a priori clear if the $\operatorname{gelu}$ operation
(approximately) preserves linear separability. 

We show ample empirical evidence in Appendix \ref{app:linear-separability-proof}
that this transformation approximately preserves the Euclidean geometry of
activations in a certain restricted sense; then, we prove that this preservation
implies that points remain approximately linearly separable after this
transformation. We further argue this empirically in Appendix
\ref{app:linearly-separable}, where we show that linear separability is
approximately preserved in MLP activations for random directions $\mathbf{v}$ in
the residual stream.

\subsection{Empirical Analysis of Distortion Introduced by the Non-linearity}
\label{app:distortion-experiments}

\paragraph{Methodology.} We use the first 10K texts of OpenWebText dataset
\citep{gokaslan2019open}. Each of these texts contains 1,024 tokens; we pass
each text through GPT-2 Small, and for each layer collect the
pre-$\operatorname{gelu}$ activations $\mathbf{x}_i$ of the MLP layer, as well
as the values $\mathbf{z}_i = \operatorname{proj}_{\ker W_{out}}
\l(\operatorname{gelu}\l(\mathbf{x}_i\r)\r)$. 
We sample 250 quadruples of distinct $i, j, k, l$ per text, and compute the
values 
\begin{align*}
    a_{ijkl} = \l(\mathbf{x}_i-\mathbf{x}_j\r)^\top \l(\mathbf{x}_k-\mathbf{x}_l\r)
    \\
    b_{ijkl} = \l(\mathbf{z}_i-\mathbf{z}_j\r)^\top \l(\mathbf{z}_k-\mathbf{z}_l\r)
\end{align*}
We collect these numbers across the first 1000 texts out of the first 10K,
resulting in 250K datapoints per layer, and perform linear regression of $b_i$
against $a_i$. 

We note that there is some inherent linearity in the quantities
$a_{ijkl},b_{ijkl}$ that could in principle skew the results of the linear
regression towards a higher $r^2$ statistic in the presence of enough samples.
In particular, there are linear dependencies of the form 
\begin{align*}
    a_{ijkl} = a_{ijkp} + a_{ijpl}
\end{align*}
for any $p$, and similarly for the $b_{ijkl}$ quantities. This makes these
quantities potentially misleading targets for linear regression. However, in our
regime, we sample 250 4-element subsets from the set $\{1,\ldots, 1024\}$, and
the probability of sampling quadruples that are linearly related is quite small.

\paragraph{Results.} We find that the coefficients of determination $r^2$ for
the linear regression are consistently high ($\approx 0.8$ or higher) for all
layers except for layer 0, indicating a high degree of fit. The $r^2$ values are
given in Table TODO, and regression plots are shown in Figure TODO. We also
remark that the intercept coefficients are $\approx0$ relative to the standard
deviation in the dependent variable. 

\subsection{Details for Subsection \ref{subsection:disconnected-evidence}}
\label{app:linear-separability-proof}

We will show the stronger property that activations remain approximately
linearly separable even after projecting on the kernel of $W_{out}$. Define the
function
\begin{align*}
    f:\mathbf{x}' \mapsto \operatorname{gelu}\l(\mathbf{x}'\r)
\mapsto \operatorname{proj}_{\ker W_{out}}
\l(\operatorname{gelu}\l(\mathbf{x}'\r)\r)
\end{align*}
where $\operatorname{proj}_{\ker W_{out}}$ is orthogonal projection on the
kernel of $W_{out}$. 

To overcome the non-linearity of $f$, we establish an empirical property
of $f$ on activations from the model's pre-training distribution. Specifically,
we show that $f$ approximately preserves the Euclidean geometry of activations
in a certain restricted sense:
\begin{align}
\label{eq:distortion-approximation}
    \l(f(\mathbf{x}_i) - f(\mathbf{x}_j)\r)^\top 
    \l(f(\mathbf{x}_k)-f(\mathbf{x}_l)\r) \approx 
    \lambda
    \l(\mathbf{x}_i-\mathbf{x}_j\r)^\top \l(\mathbf{x}_k-\mathbf{x}_l\r) + \eta
\end{align}
for $\lambda > 0$ and $\eta\approx0$ (relative to the standard deviation of the
expression on the left-hand side of the approximation). Specifically, we perform
linear regression of $\l(f(\mathbf{x}_i) - f(\mathbf{x}_j)\r)^\top 
\l(f(\mathbf{x}_k)-f(\mathbf{x}_l)\r)$ using
$\l(\mathbf{x}_i-\mathbf{x}_j\r)^\top \l(\mathbf{x}_k-\mathbf{x}_l\r)$ as the
predictor variable in Appendix \ref{app:distortion-experiments}, and find very
high coefficients of determination ($r^2\approx0.8$) in all layers except for
layer $0$. To avoid relying solely on the coefficient of determination, we also
generate regression plots for the data; see Figure
\ref{fig:distortion-regression} for regression lines over samples of $10^4$
points from each layer of GPT2-Small.

\begin{figure}
    \includegraphics[width=0.9\textwidth]{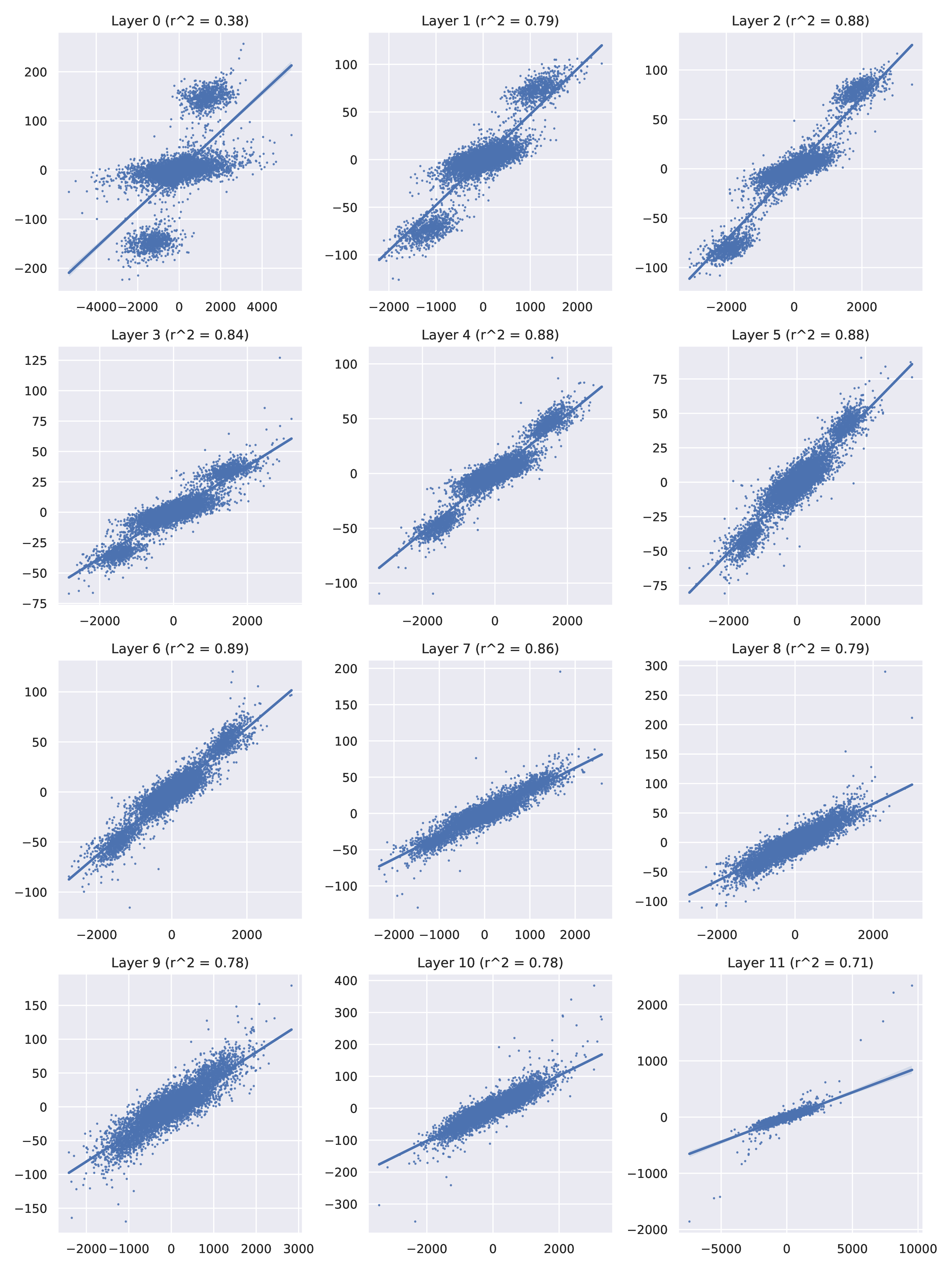}
    \caption{Regression plots accompanying the experiments in Appendix \ref{app:linear-separability-proof}.}
    \label{fig:distortion-regression}
\end{figure}

Finally, we prove that $f$ maintains linear separability if we assume that
Equation \ref{eq:distortion-approximation} holds exactly with $\eta=0$:
\begin{lemma}
\label{lemma:linear-separability}
Let $x_i\in \mathbb{R}^d, 1\leq i \leq n$ be linearly separable with respect to
binary labels $y_i\in\{-1,1\}$. Let $f:\mathbb{R}^d\to \mathbb{R}^d$ be a
transformation with the property that
\begin{align*}
  \l(f(x_i)-f(x_j)\r)^\top \l(f(x_k)-f(x_l)\r)=
    \lambda\l(x_i-x_j\r)^\top \l(x_k-x_l\r) 
\end{align*}
for all distinct $i,j,k,l$ and some $\lambda > 0$. Then, $f(x_i)$ are linearly
separable with respect to the labels $y_i$ as well.
\end{lemma}
\begin{proof}
Consider the hard SVM objective for the points $(x_i, y_i)$, 
\begin{align*}
    \min_{w, b} \frac{1}{2}\l\|w\r\|_2^2 \quad \text{subject to}\quad y_i
    \l(w^\top x_i+b\r)\geq 1.
\end{align*}
Since the examples are linearly separable, we know that the minimizer
$(w^*,b^*)$ exists and satisfies all constraints. Furthermore, from the
optimality conditions of the dual formulation of the objective we know that we
can write $w^* = \sum_{i}^{}\alpha_i x_i$ where $\sum_{i}^{}\alpha_i=0$ (see for
example \citet{awad2015support}). Let $S=\{s_1, \ldots, s_t\}\subset \{1,\ldots,
n\}$ be the support of $\alpha$, i.e. the indices $s_j$ such that
$\alpha_{s_j}\neq 0$.  Since $\sum_{i}^{}\alpha_i=0$, we can rewrite $w^*$ as
\begin{align*}
    w^* = \sum_{j}^{}\beta_j \l(x_{s_j}-x_{s_{j+1}}\r)
\end{align*}
with indices modulo $\left|S\right|$. Since $(w^*,b^*)$ is a separating
hyperplane for $(x_i, y_i)$, we have 
\begin{align*}
    \l(w^*\r)^\top x_i &\geq 1 - b \quad \text{when}\quad y_i=1
    \\
    \l(w^*\r)^\top x_i &\leq -1 - b \quad \text{when}\quad y_i=-1
\end{align*}
and thus 
\begin{align*}
    \l(w^*\r)^\top \l(x_i - x_j\r) \geq 2 \quad \text{when}\quad y_i=1, y_j=-1.
\end{align*}
Using the expansion of $w^*$ as a linear combination of differences between
examples, this says
\begin{align*}
    \sum_{j}^{}\beta_j \l(x_{s_j}-x_{s_{j+1}}\r)^\top \l(x_i-x_j\r) \geq 2 
\quad \text{when}\quad y_i=1, y_j=-1.
\end{align*}
and thus 
\begin{align*}
    \sum_{j}^{}\beta_j \l(f(x_{s_j}) - f(x_{s_{j+1}})\r)^\top
    \l(f(x_i)-f(x_j)\r)\geq 2\lambda > 0 \quad \text{when}\quad y_i=1, y_j=-1.
\end{align*}
Now we claim that 
\begin{align*}
    \widehat{w} = \sum_{j}^{}\beta_j \l(f(x_{s_j}) - f(x_{s_{j+1}})\r)
\end{align*}
is a linear separator for $(f(x_i), y_i)$ for some bias to be determined later. Indeed, let $M = \min_{y_i=1}
\widehat{w}^\top f(x_i)$ and $m = \max_{y_i=-1} \widehat{w}^\top f(x_i)$. Then
we have $M - m \geq 2\lambda > 0$. Choosing any $\widehat{b}\in (m, M)$, we have 
\begin{align*}
    \widehat{w}^\top f \l(x_i\r) - \widehat{b} \geq M - \widehat{b} > 0\quad \text{when}\quad y_i=1
    \\
    \widehat{w}^\top f \l(x_i\r) - \widehat{b} \leq m - \widehat{b} < 0\quad \text{when}\quad y_i=-1
\end{align*}
which shows that $(\widehat{w}, \widehat{b})$ linearly separates the points
$(f(x_i), y_i)$.
\end{proof}

\subsection{MLP weights are full-rank matrices}
\label{app:singular}

In figure \ref{fig:singular}, we plot the 100 smallest singular values of the MLP weights in GPT2-Small for all 12 layers. We observe that they the vast majority are bounded well away from $0$. This confirms that both MLP weights are full-rank transformations. 

\begin{figure}[ht]
    \centering
    \begin{minipage}[b]{0.45\textwidth}
    \begin{tikzpicture}
        \node[anchor=south west,inner sep=0] (image) at (0,0) {\includegraphics[width=1.0\textwidth]{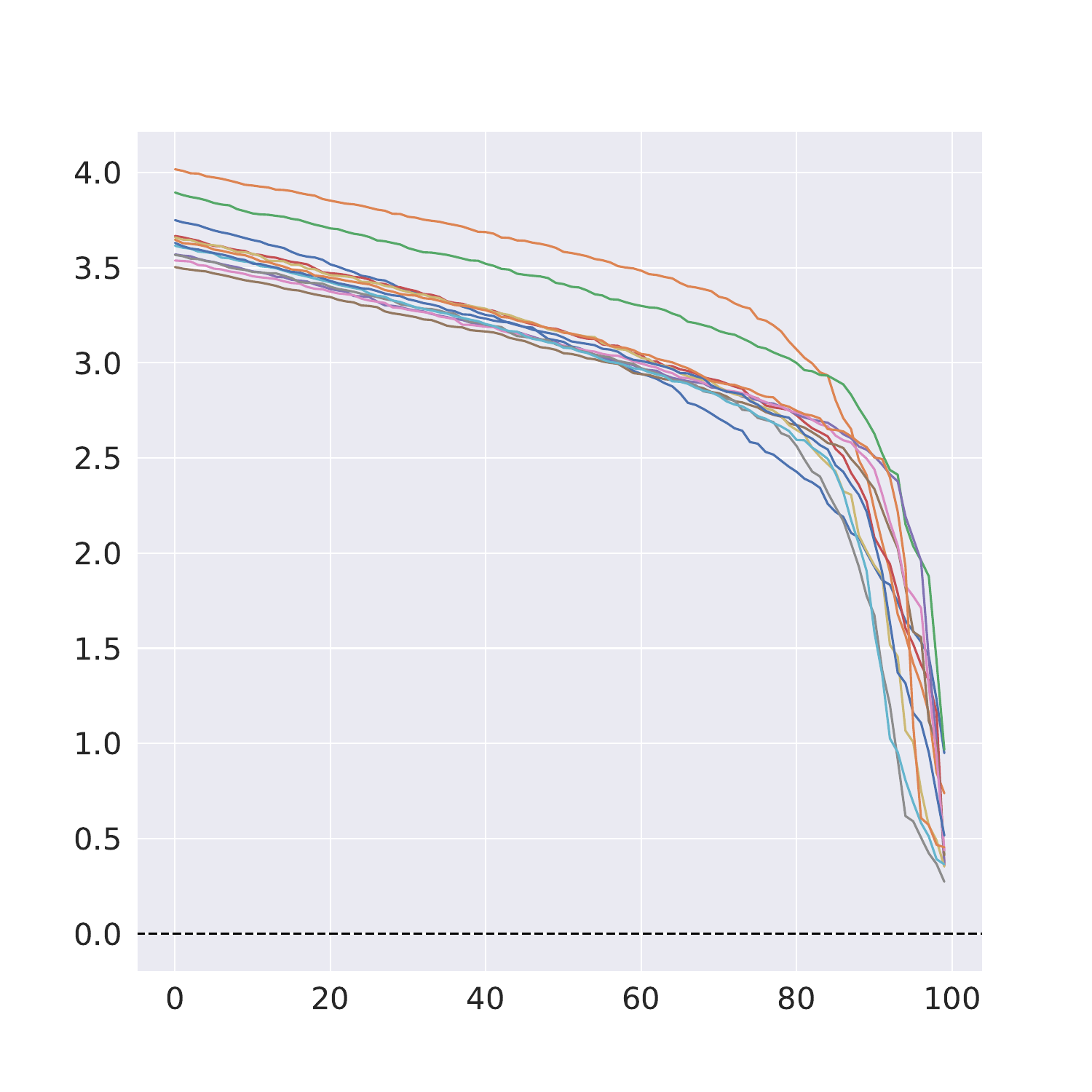}};
        \begin{scope}[x={(image.south east)},y={(image.north west)}]
            \node [anchor=east, rotate=90] at (-0.05,0.7) {Singular value};
            \node [anchor=north] at (0.5,0.0) {Index (decreasing)};
        \end{scope}
    \end{tikzpicture}
    \end{minipage}
    \begin{minipage}[b]{0.45\textwidth}
    \begin{tikzpicture}
        \node[anchor=south west,inner sep=0] (image) at (0,0) {\includegraphics[width=1.0\textwidth]{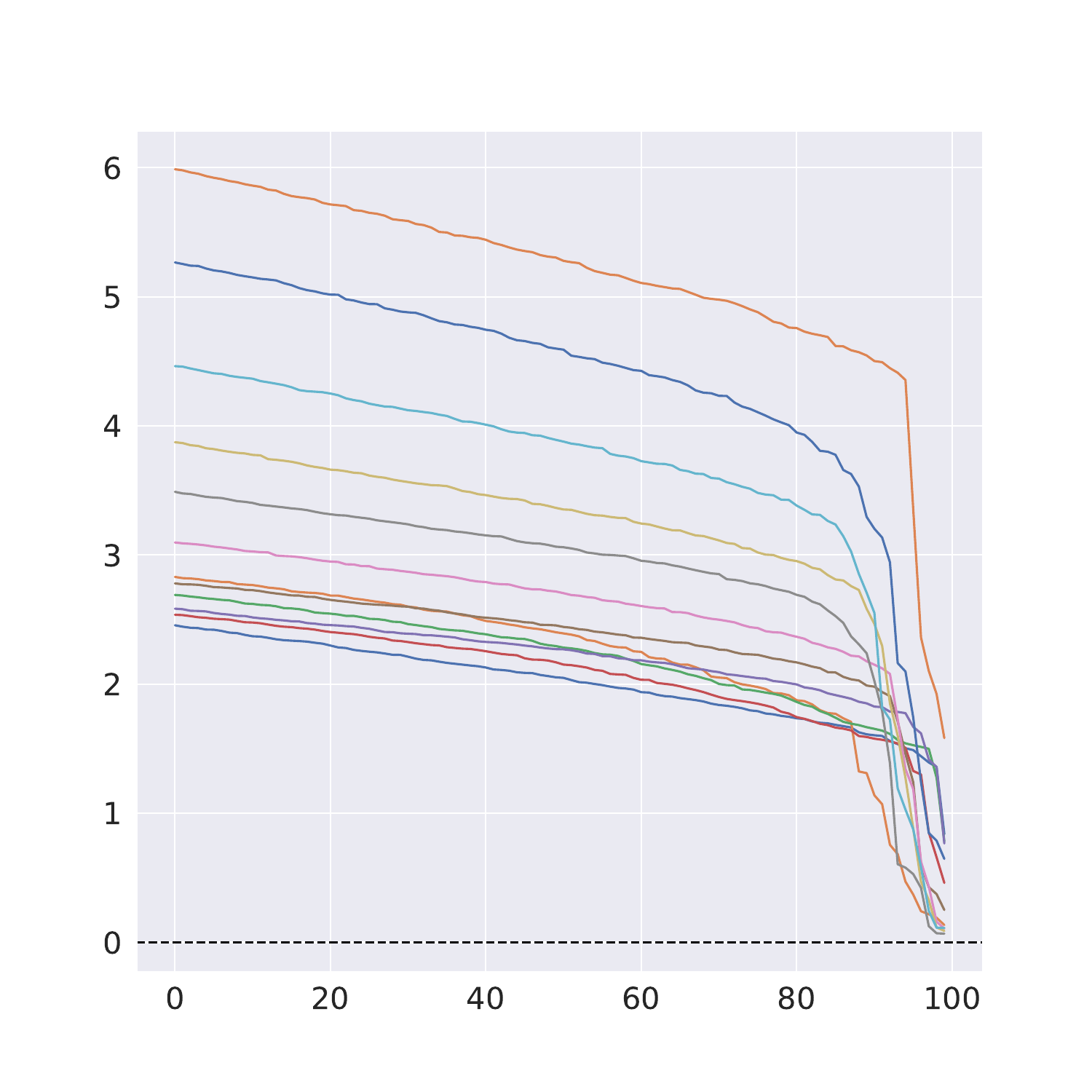}};
        \begin{scope}[x={(image.south east)},y={(image.north west)}]
            \node [anchor=north] at (0.5,0.0) {Index (decreasing)};
        \end{scope}
    \end{tikzpicture}
    \end{minipage}
    
    \caption{Smallest 100 singular values of the $W_{in}$ (left) and $W_{out}$ (right) MLP weights by layer in in GPT2-Small}
    \label{fig:singular}
\end{figure}

\subsection{Features in the residual stream propagate to hidden MLP activations}
\label{app:linearly-separable}

\textbf{Intuition}. Suppose we have two classes of examples that are linearly separable in the residual stream. The transformation from the residual stream to the hidden MLP activations is a linear map followed by a nonlinearity, specifically $x\mapsto\operatorname{gelu}(W_{in}x)$. 
As we observed in \ref{app:singular}, the $W_{in}$ matrix is full-rank, meaning that all the information linearly present in $x$ will also be so in $W_{in}x$.
Even better, since $W_{in}$ maps $x$ from a $d_{\text{resid}}$-dimensional space to a $d_{\text{MLP}}=4d_{\text{resid}}$-dimensional space, this should intuitively make it much easier to linearly separate the points, because in a higher-dimensional space there are many more linear separators.
On the other hand, the non-linearity has an opposite effect: by compressing the space of activations, it makes it harder for points to be separable. So it is a priori unclear which intuition is decisive.

\textbf{Empirical validation}.
However, it turns out that empirically this is not such a problem. 
To test this, we run the model GPT2-Small on random samples from its data distribution (we used OpenWebText-10k), and extract 2000 activations of an MLP-layer after the non-linearity. We train a linear regression with  $\ell_2$-regularization to recover the dot product of the residual stream immediately before the MLP-layer of interest and a randomly chosen direction. 
We repeat this experiment with different random vectors and for each layer. We observe that all regressions are better than chance and explain a significant amount of variance on the held-out test set ($R^2=0.71\pm 0.17, \text{MSE}=0.31\pm0.18, p<0.005$). 
Results are shown in Figure \ref{fig:regression-classification} (right) (every marker corresponds to one regression model using a different random direction).

The position information in the IOI task is really a binary feature, so we are also interested in whether \emph{binary} information in general is linearly recoverable from the MLP activations. To test this, we sample activations from the model run on randomly-sampled prompts. This time however, we add or subtract a multiple of a random direction $v$ to the residual stream activation $u$, and calculate the MLP activations using this new residual stream vector $u'$:
\begin{align*}
    u' = u + y\times z\times \|u\|_2 \times v
\end{align*}
where $y\in\{-1, 1\}$ is uniformly random, $z$ is a scaling factor we manipulate, and $v$ is a randomly chosen direction of unit norm.
For each classifier, we randomly sample a direction $v$ that we either add or subtract (using $y$) from the residual stream. The classifier is trained to predict $y$. We rescale v to match the average norm of a residual vector and then scale it with a small scalar $z$.

Then, a logistic classifier is trained on 1600 samples. Again, we repeat this experiment for different $v$ and $z$, and for each layer. We observe that the classifier works quite well across layers even with very small values of $z$ (still, accuracy drops for $z = 0.0001$). Results are shown in Figure \ref{fig:regression-classification} (right), and Table \ref{tab:linear-classifier-accuracy}.

\begin{table}[h]
\centering
\caption{Mean Accuracy for Different Values of \( z \)}
\label{tab:linear-classifier-accuracy}
\begin{tabular}{cc}
\hline
\( z \) & Mean Accuracy \\
\hline
0.0001 & 0.69 \\
0.001  & 0.83 \\
0.01   & 0.87 \\
0.1    & 0.996 \\
\hline
\end{tabular}
\end{table}

\begin{figure}[ht]
    \centering
    \begin{minipage}[b]{0.45\textwidth}
    \begin{tikzpicture}
        \node[anchor=south west,inner sep=0] (image) at (0,0) {\includegraphics[width=1.0\textwidth]{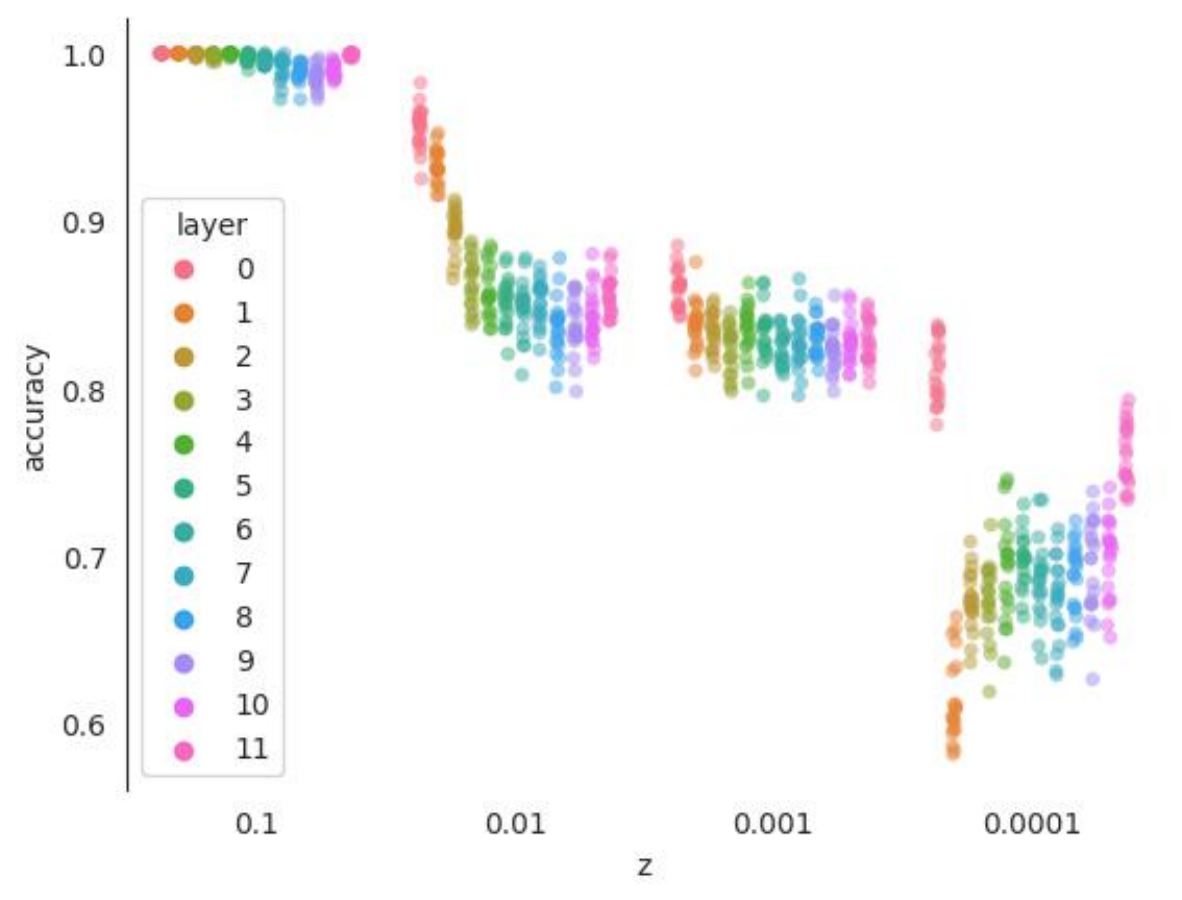}};
        \begin{scope}[x={(image.south east)},y={(image.north west)}]
        \end{scope}
    \end{tikzpicture}
    \end{minipage}
    \begin{minipage}[b]{0.45\textwidth}
    \begin{tikzpicture}
        \node[anchor=south west,inner sep=0] (image) at (0,0) {\includegraphics[width=1.0\textwidth]{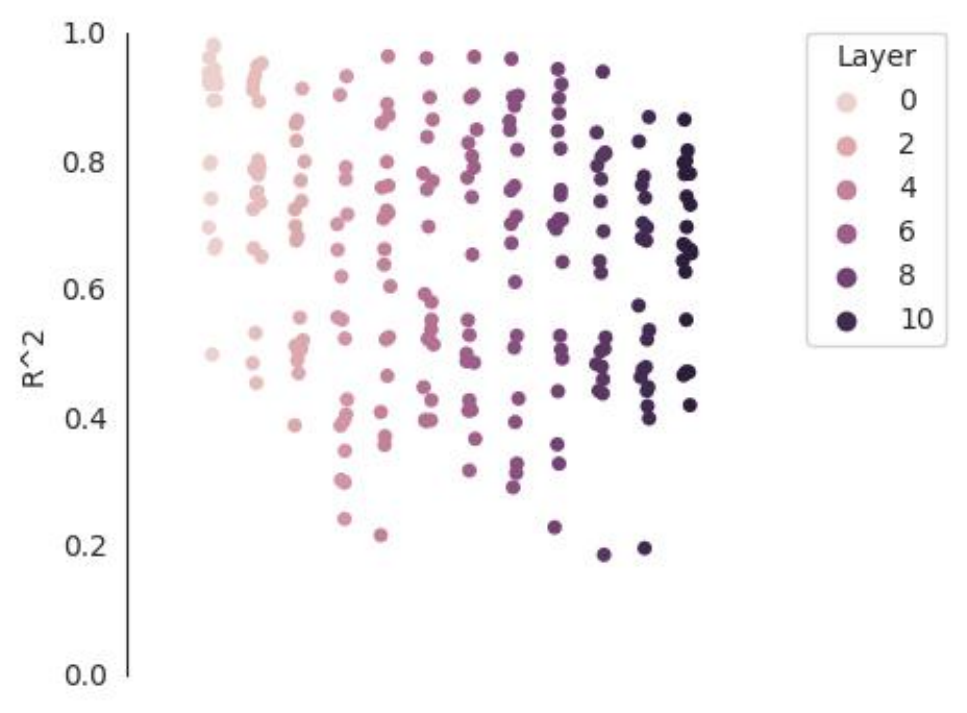}};
        \begin{scope}[x={(image.south east)},y={(image.north west)}]
        \end{scope}
    \end{tikzpicture}
    \end{minipage}
    
    \caption{Recovering residual stream features linearly from hidden MLP activations: classification (left) and regression (right).}
    \label{fig:regression-classification}
\end{figure}

\end{document}